\documentclass{article}

\usepackage{microtype}
\usepackage{graphicx}
\usepackage{booktabs} 

\usepackage{amsmath,amsfonts,amsthm,amssymb,bbm}
\usepackage{algorithm}
\usepackage[noend]{algpseudocode}
\usepackage{mathtools}
\usepackage{color}
\usepackage{tikz, subcaption}
\usetikzlibrary{shapes, arrows, positioning}

\newcommand{\Pa}[1]{\textbf{pa}(#1)}
\newcommand{\PaC}[1]{\textbf{pa}^\leftrightarrow(#1)}
\newcommand{\BiD}[1]{\textbf{biD}(#1)}

\usepackage{thmtools,thm-restate}
\usepackage{enumitem}
\usepackage{natbib}

\newtheorem{assumption}{Assumption}

\newtheorem{claim}{Claim}

\newcommand*{\Break}{\textbf{break}}

\DeclareMathAlphabet\mathbfcal{OMS}{cmsy}{b}{n}

\DeclareMathOperator*{\argmin}{arg\,min}

\usepackage{jmlr2_mod}
\begin{document}


\jmlrheading{23}{2022}{??}{??; Revised ??}{??}{??}{Sina Akbari, Jalal Etesami and Negar Kiyavash}
\ShortHeadings{Min-Cost Intervention}{Akbari, Etesami and Kiyavash}
\firstpageno{1}



\title{Experimental Design for Causal Effect Identification}
\author{\name Sina Akbari \email sina.akbari@epfl.ch \\
       \addr EPFL,\\
       Lausanne, Switzerland
       \AND
       \name Jalal Etestami \email seyed.etesami@epfl.ch \\
       \addr EPFL,\\
       Lausanne, Switzerland
       \AND
       \name Negar Kiyavash \email negar.kiyavash@epfl.ch \\
       \addr EPFL,\\
       Lausanne, Switzerland}

\editor{}

\maketitle
\begin{abstract}
Pearl's \emph{do calculus} is a complete axiomatic approach to learn the \emph{identifiable} causal effects from observational data.
When such an effect is not identifiable, it is necessary to perform a collection of often costly interventions in the system to learn the causal effect.
In this work, we consider the problem of designing the collection of interventions with the minimum cost to identify the desired effect.
First, we  prove that this problem is NP-complete and subsequently propose an algorithm that can either find the optimal solution or a logarithmic-factor approximation of it. This is done by establishing a connection 
between our problem and the minimum hitting set problem. Additionally, we propose several polynomial time heuristic algorithms to tackle the computational complexity of the problem. 
Although these algorithms could potentially stumble on sub-optimal solutions, our simulations show that they achieve small regrets on random graphs.
\end{abstract}

\section{Introduction }
Causal inference plays a key role in many applications such as psychology \citep{foster2010causal}, econometrics \citep{hoover1990logic}, education, social sciences \citep{murnane2010methods,gangl2010causal}, etc.
Causal effect identification, one of the most fundamental topics in causal inference,
 is concerned with estimating the effect of intervening on a set of variables, say $X$ on another set of variables, say $Y$ denoted by $P(Y|do(X))$. The estimation is performed having access to a set of observational and/or interventional distributions under causal assumptions that are usually encoded in the form of a causal graph. 
The causal graph of a system of variables captures the interconnection among the variables and  can be inferred from a combination of observations, experiments, and expert knowledge about the phenomenon under investigation \citep{spirtes2000causation}.
Throughout this work, we assume that the causal graph is given as side information.

Given a causal graph, it is known that in the absence of unobserved (latent) variables, every causal effect is identifiable from mere observational data  \citep{robins1987graphical, spirtes2000causation}.
On the other hand, inferring causal effects from data becomes challenging in the presence of latent variables.
In the setting where only observational data is available, the \emph{do calculus}, introduced by \citet{pearl1995causal}, has been shown to be complete.
That is, it provides a complete set of rules to compute a causal effect (if identifiable) given a causal graph and observational data \citep{huang2012pearl}.
Moreover, polynomial time algorithms exist that can determine the identifiability of a causal effect using do-calculus \citep{shpitser2006identification}.

In recent years, there has been an increase in the effort  to generalize Pearl's do-calculus to the setting in which data from both observational and interventional data are available for identifying a causal effect.
For instance,
\citet{bareinboim2012causal} studied the problem of estimating the causal effect of intervening on a set of variables $X$ on the outcome $Y$ when we experiment on a different set $Z$. 
This problem is known as $z$-identifiability, and \citet{bareinboim2012causal} provide a complete algorithm for computing $P(Y|do(X))$ using information provided by experiments on all subsets of $Z$. 
A slightly more general version of $z$-identifiability is called $g$-identifiability, which considers the problem of identifying $P(Y|do(X))$ from an arbitrary collection of 
distributions. 
\citet{lee2020general} and \citet{kivva2022revisiting} studied the $g$-identifiability problem and showed that Pearl's do-calculus is also complete in this setting.
All three of the aforementioned works study the identifiability of $P(Y|do(X))$.
There are also various works that consider the more general problem of identifying a conditional causal effect of the form $P(Y|do(X),W)$ from a combination of distributions. 
However, most of these works do not manage to provide complete results a la Pearl's do-calculus. See \citep{tikka2019causal}, for a complete review on causal effect identification.

When a causal effect is not identifiable from observations, it is necessary to perform a collection of interventions to infer the effect of interest.
However, such interventions could be costly, impossible, or unethical to perform.
Therefore, naturally we are interested in the problem of designing a collection of low cost permitted interventions to identify a causal effect. 
This is the focus of our paper.
A closely related work to ours is \citep{kandasamy2019minimum}, in which, the authors considered the problem of finding the minimum number of interventions to identify every possible causal query.
Their approach is based on two limiting assumptions, namely that all interventions have the same cost and that we are allowed to intervene on any variable. 
More importantly, the result in \citep{kandasamy2019minimum} guarantees to render every causal effect identifiable, which makes the solution sub-optimal for a specific query.
In other words, a set of interventions that makes all causal effects identifiable might have higher aggregate cost than a set of interventions designed for identifying a specific causal effect.

Designing minimum-cost interventions has also received attention in the causal discovery literature, under the term \emph{experimental design}. 
In causal discovery, the goal is to infer the causal graph from a dataset \citep{spirtes2000causation,colombo2012learning,akbari2021recursive}. 
It is known that mere observational data cannot fully recover the causal graph, and thus additional interventional data is required to precisely learn the graph. 
\citet{lindgren2018experimental} considered the problem of designing a set with minimum number of interventions to learn a causal graph given the essential graph (assuming no latent variable), and showed that this problem is NP-hard.
\citet{addanki2020efficient} studied a similar problem in the presence of latent variables.
The problem of orienting the maximum number of edges using a fixed number of interventions was studied in \citep{hauser2014two, ghassami2018budgeted, agrawal2019abcd}.
\citet{addanki2021intervention} studied designing interventions for causal discovery when the goal is to learn a portion of the edges in the causal graph instead of all of them. 

In this work, we 
study the problem of designing the set of minimum cost interventions for identifying a specific causal effect, where intervening on each variable may induce a different cost, and we are not necessarily allowed to intervene on every variable.
Extending the findings of our ICML paper, \cite{akbari2022minimum}, we outline our contributions as follows.
\begin{itemize}[leftmargin=*]
    \item We prove that finding a minimum-cost intervention set for identifying a specific causal effect is NP-complete.
    Further, we show that approximating the solution to this problem within a sub-logarithmic factor of the optimal solution is NP-hard.
    \item We formulate the minimum cost intervention problem in terms of a minimum hitting set problem, and propose an algorithm based on this formulation that can find the optimal solution to the minimum cost intervention problem.
    This algorithm can also be used to approximate the solution up to a logarithmic-factor\footnote{The implementations of all the algorithms proposed in this work can be found at 
    \url{https://github.com/SinaAkbarii/min_cost_intervention/tree/main}.}.
    \item We propose several heuristic algorithms to solve the minimum cost intervention problem in polynomial time, and through empirical evaluations show that they achieve low-regret solutions in randomly generated causal graphs.
    We also provide an upper bound on the regret of these algorithms in the average case when the graph is generated due to Eros-Renyi model.
    \item We analyze several special cases of the minimum cost intervention problem that can be solved in polynomial time, and provide efficient algorithms for these cases.
\end{itemize}

\section{Terminology \& Problem Description}\label{sec:termino}
We briefly introduce the notations used in this paper.
We begin with the definition of \emph{structural causal model} (SCM) \citep{pearl2000models}, which is the framework we use throughout this paper.
An SCM is a tuple $M=(U,V,F,P(U))$, where $U$ is the set of exogenous variables which are not observed but affect the relationship among the variables of the system, $V=\{v_1,...,v_n\}$ is the set of observed endogenous variables where each $v_i\in V$ is a function of a subset of $V\cup U$ denoted by $\Pa{v_i}\cup\mathbf{pa}^{U}(v_i)$, $F=\{f_1,...,f_n\}$ is a set of functions where each $f_i$ determines the value of $v_i=f_i(\Pa{v_i},\mathbf{pa}^{U}(v_i))$, and $P(U)$ is the joint probability distribution over the variables $U$.
An intervention is defined through a mathematical operator $do(X=\hat{X})$, which replaces the functions corresponding to variables $X$ in the model $M$ with a constant function $f=\hat{X}$.
Denoting this model by $M_{\hat{X}}$, the interventional distribution $P(Y\vert do(X=\hat{X}))$ is then given by $P_{M_{\hat{X}}}(Y)$, or $P_{\hat{X}}(Y)$ in short \citep{pearl2012calculus}.
For a subset $S$ of variables $V$, we denote by $Q[S]=P(S\vert do(V\setminus S))$, the interventional distribution of the variables $S$ after intervention on the rest of the variables \citep{tian2002testable}.

The causal graph corresponding to the SCM $M$ is a semi-Markovian graph $\mathcal{G}$ with one vertex for each $v_i\in V$, where there is a directed edge from $v_i$ to $v_j$ if the value of $v_j$ is a function of $v_i$ ($v_i\in\Pa{v_j}$), and there is a bidirected edge between $v_i$ and $v_j$ if the values of $v_i$ and $v_j$ are both functions of a common exogenous variable $u$.
See Figure \ref{fig:semi-markov} for an example.
We use $V$ to denote the set of vertices of $\mathcal{G}$ throughout the paper.
We use the words vertex and variable interchangeably throughout this work, as each vertex represents a variable.


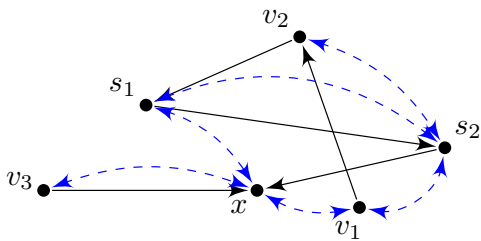
\begin{figure}[t]
    \centering
	\tikzstyle{block} = [circle, inner sep=1.5pt, fill=black]
	\tikzstyle{input} = [coordinate]
	\tikzstyle{output} = [coordinate]
    \resizebox{0.45\textwidth}{!}{
        \begin{tikzpicture}
            \tikzset{edge/.style = {->,> = latex'},line width=1.4pt}
            \node[block](s1) at (-.5,2.5) {};
            \node[block](s2) at (3,2) {};
            \node[block](x) at (0.8,1.5) {};
            \node[block](v3) at (-1.7,1.5) {};
            \node[block](v2) at (1.3,3.3) {};
            \node[block](v1) at (2,1.3) {};
            
            \node[] ()[above left=-0.1cm and -0.1cm of s1]{$s_1$};
            \node[] ()[above right=-0.1cm and -0.1cm of s2]{$s_2$};
            \node[] ()[below left=-0.1cm and -0.1cm of x]{$x$};
            \node[] ()[above left=-0.2cm and -0.1cm of v3]{$v_3$};
            \node[] ()[above left=-0.1cm and -0.1cm of v2]{$v_2$};
            \node[] ()[below left=-0.05cm and -0.25cm of v1]{$v_1$};
            
            \draw[edge] (s1) to (s2);
            \draw[edge] (s2) to (x);
            \draw[edge] (v3) to (x);
            \draw[edge] (v1) to (v2);
            \draw[edge] (v2) to (s1);
            
            \draw[edge, color=blue, dashed, style={<->}, bend left=30] (s1) to (s2);
            \draw[edge, color=blue, dashed, style={<->}, bend left=20] (v1) to (x);
            \draw[edge, color=blue, dashed, style={<->}, bend right=45] (v1) to (s2);
            \draw[edge, color=blue, dashed, style={<->}, bend left=20] (s1) to (x);
            \draw[edge, color=blue, dashed, style={<->}, bend left=20] (v2) to (s2);
            \draw[edge, color=blue, dashed, style={<->}, bend left=20] (v3) to (x);
        \end{tikzpicture}
    }
    \caption{An example of a semi-Markovian graph. In this example, $\Pa{x}=\{v_3,s_2\}, \BiD{x}=\{v_3,s_1,v_1\}$, and $\PaC{x}=\{v_3\}$.}\label{fig:semi-markov}
    \end{figure}
We use small letters for variables, capital letters for sets of variables, and bold letters for collections of subsets of variables (set families), respectively.
We utilise common graph-theoretic terms such as parents of a vertex $x$ (denoted by $\Pa{x}$), as well as children, ancestors, and descendants of a vertex.
We denote by $\BiD{x}$, the set of vertices that have a bidirected edge to $x$.
We also denote by $\PaC{x}=\Pa{x}\cap\BiD{x}$ the set of parents of $x$ that have a bidirected edge to $x$.
For a set $X$, $\Pa{X}$ is defined as $\Pa{x}=\cup_{x\in X}\Pa{x}\setminus X$.
The rest of the aforementioned sets are defined analogously for a set of variables $X$.
For a set of variables $X$, we denote by $\mathcal{G}_{[X]}$ the induced vertex subgraph of $\mathcal{G}$ over the vertices $X$.
The connected components of the edge-induced subgraph of $\mathcal{G}$ over its bidirected edges are called \emph{c-components} (aka \emph{districts}) of $\mathcal{G}$ \citep{tian2002testable}.
For example, the causal graph in Figure \ref{fig:semi-markov} consists of only one c-component. However, its induced subgraph over $\{s_1,x,v_2\}$ consists of two c-components $\{x,s_1\}$ and $\{v_2\}$.

\begin{definition}[Identifiability]\label{def: identification}
We say a causal effect $P(Y\vert do(X))$ is identifiable in $\mathcal{G}$, if it is uniquely computable from $P(V)$, the joint distribution of the observed variables.
More precisely, for any positive models $M_1$ and $M_2$ that are compatible with the causal graph $\mathcal{G}$ and $P_{M_1}(V)=P_{M_2}(V)$ (admit the same joint distribution), $P_{M_1}(Y|do(X)) =P_{M_2}(Y\vert do(X))$.
\end{definition}
Analogously, for a given set of interventional distributions $\mathbf{P}=\{P(Y_1\vert do(X_1)),...,P(Y_k\vert do(X_k))\}$, we say $P(S\vert do(T))$ is identifiable in the causal graph $\mathcal{G}$ from $\mathbf{P}$, if for any positive model $M$ that is compatible with $\mathcal{G}$, $P_{M}(S\vert do(T))$ is uniquely computable from $\mathbf{P}$.
Letting $\mathbf{P}$ be $\mathbf{P}=\{P(V\vert do(\emptyset))\}$, this generalization reduces to Definition \ref{def: identification}.

\subsection{Problem Description}
Let $\mathcal{G}$ be a semi-Markovian graph on the vertex set $V$ along with a cost function $\mathbf{C}:V\to \mathbbm{R}^{\geq0}$, where
$\mathbf{C}(x)$ for some $x\in V$ denotes the cost of intervening on variable $x$.
With slight abuse of notation, we denote the cost of intervening on a set of variables $X\subseteq V$ by $\mathbf{C}(X)$.
In this work, we assume that the intervention cost is additive, unless otherwise stated (we shall discuss non-additive cost models in Section \ref{sec:nonlin}.) 
More precisely, we make the following assumption.
\begin{assumption}\label{ass:lincost}
    For a set $X\subseteq V$, the cost of intervening on $X$ is $\mathbf{C}(X):=\sum_{x\in X}\mathbf{C}(x)$,
    and for a collection $\mathbf{X}$ of subsets of $V$, the cost of intervention on $\mathbf{X}$ is 
    $
    \mathbf{C}(\mathbf{X}):=\sum_{X\in\mathbf{X}}\mathbf{C}(X).
    $
\end{assumption}
Moreover, we assume that there is no cost for observing a variable, i.e.,  $C(\emptyset)=0$.
Therefore, when intervening on set $X$, we have access to $Q[V\setminus X]=P(V\setminus X\vert do(X))$ at the cost of $C(X)$.
\begin{remark}
    In this setting, we can model a non-intervenable variable $x$ by assigning the cost $\mathbf{C}(x) = \infty$.
\end{remark}
For a given causal graph $\mathcal{G}$ and disjoint subsets $S,T\subseteq V$, our goal is to find a collection $\textbf{A}=\{A_1,A_2,...,A_m\}$ of subsets of $V$ with minimum $C(\textbf{A})$ such that $P(S|do(T))$ is identifiable in $\mathcal{G}$ given $\{Q[V\setminus A_1],...,Q[V\setminus A_m]\}$. 
More precisely, let $\mathbf{ID}_\mathcal{G}(S,T)$ denote the set of all collections of subsets of $V$, e.g., $\mathbf{A} = \{A_1,A_2,...,A_m\}$, where $A_i\subseteq V, 1\leq i\leq m$, such that $P(S|do(T))$ is identifiable in $\mathcal{G}$ given $\{Q[V\setminus A_1],...,Q[V\setminus A_m]\}$. 
Note that $|\mathbf{ID}_\mathcal{G}(S,T)|\leq2^{2^{|V|}}$.
Thus, the minimum-cost intervention design problem to identify $P(S|do(T))$ can be cast as the following optimization problem,
\begin{equation}\label{eq:optimization}
    \mathbf{A}^*_{S,T}\in\argmin_{\mathbf{A}\in \mathbf{ID}_\mathcal{G}(S,T)}\sum\mathop{}_{\mkern-5mu A\in\mathbf{A}} \mathbf{C}(A).
\end{equation}
We say $\mathbf{A}^*_{S,T}$ is 
the minimum-cost intervention for identifying $P(S|do(T))$ in $\mathcal{G}$. 
Note that additional constraints or regularization terms can be added to target a specific minimum-cost intervention set within $\mathbf{ID}_\mathcal{G}(S,T)$. 

It has been shown that $P(S\vert do(T))$ is identifiable in $\mathcal{G}$ if and only if $Q[\textbf{Anc}_{\mathcal{G}\setminus T}(S)]$ is identifiable in $\mathcal{G}$, where $\textbf{Anc}_{\mathcal{G}\setminus T}(S)$ are ancestors of $S$ in $\mathcal{G}$ after deleting vertices $T$ \citep{kivva2022revisiting,lee2020general,jaber2019causal,shpitser2006identification}.
That is, $\mathbf{ID}_\mathcal{G}(S,T)=\mathbf{ID}_\mathcal{G}(\textbf{Anc}_{\mathcal{G}\setminus T}(S),V\setminus \textbf{Anc}_{\mathcal{G}\setminus T}(S))$.
In other words, any causal query of the form $P(S\vert do(T))$ can be transformed into a causal query that is in the form of $Q[\cdot]$.
Therefore, in what follows, we focus on the minimum-cost intervention problem for identifying causal queries of the form $Q[S]=P(S\vert do(V\setminus S))$.
Throughout the rest of this work, we will assume $T=V\setminus S$ in Equation \eqref{eq:optimization}.
In Section \ref{sec:singlec}, we study the above problem when $\mathcal{G}_{[S]}$ is a single c-component.
In Section \ref{sec:genid}, we generalize our results to an arbitrary subset $S$.
We discuss the problem under non-additve costs in Section \ref{sec:nonlin}.
We evaluate our proposed algorithms in terms of runtime and optimality in Section \ref{sec:experiment}.


\section{Single C-component Identification}\label{sec:singlec}
The main challenge in solving the optimization problem in Equation \eqref{eq:optimization} is that the number of elements in $\mathbf{ID}_\mathcal{G}(S,T)$ is possibly super-exponential.
Throughout this section, we assume that $S$ is a subset of variables in $\mathcal{G}$ such that $\mathcal{G}_{[S]}$ is a single c-component, unless stated otherwise. 
Under this assumption, we first show\footnote{All proofs are provided in Appendix \ref{apdx: proofs}.} in Theorem \ref{thm:singleton} that $\mathbf{ID}_\mathcal{G}(S,T)$ in Equation \eqref{eq:optimization} can be replaced with a substantially smaller subset without changing the solution to the problem in \eqref{eq:optimization}. 
Next, we prove in Theorems \ref{thm:NP-hard} and \ref{thm:reduction2} that even after this substitution, the minimum-cost intervention problem remains NP-complete. 
\begin{restatable}{lemma}{lemintersect} \label{lem:intersect}
    Suppose $S$ is a subset of variables such that $\mathcal{G}_{[S]}$ is a single c-component. 
    Let $\mathbf{A}=\{A_1,A_2,...,A_m\}$ be a collection of subsets of $V$ such that 
    $A_\cup\cap S\!=\!\emptyset$, where 
    $A_\cup:=\cup_{i=1}^mA_i$.
    If $\mathbf{A}\in$ $\mathbf{ID}_\mathcal{G}(S,V\setminus S)$, then the singleton collection $\mathbf{A}_\cup=\{A_\cup\}$ also belongs to $\mathbf{ID}_\mathcal{G}(S,V\setminus S)$.
\end{restatable}

\begin{remark}\label{rem: cost-singleton}
    The cost of $\mathbf{A}_\cup$ in Lemma \ref{lem:intersect} is at most $C(\mathbf{A})$,
    \[
    \mathbf{C}(\mathbf{A}) = 
    \sum_{\mkern-5mu A_i\in\mathbf{A}}\:\: \sum_{\mkern-5mu a\in A_i}\mathbf{C}(a)
    \geq 
    \sum_{\mkern-5mu a\in A_\cup}\mathbf{C}(a)
    = \mathbf{C}(\mathbf{A}_\cup),
    \]
    where the inequality holds because each $a\in A_\cup$ appears exactly once in the right-hand-side summation, whereas it appears at least once on the left hand side.
\end{remark}
Next, we prove that for a given subset $S$ where $\mathcal{G}_{[S]}$ is a   c-component, the collection $\mathbf{A}_{S,V\setminus S}^*$ is singleton, that is, it contains exactly one intervention set.
\begin{restatable}{theorem}{thmsingleton} \label{thm:singleton}
    Suppose $S$ is a subset of variables such that $\mathcal{G}_{[S]}$ is a   c-component. 
    Let $\mathbf{A}=\{A_1, A_2,...,A_m\}$ be a collection of subsets such that $\mathbf{A}\in$ $\mathbf{ID}_\mathcal{G}(S,V\setminus S)$ and $m>1$.
    Then, there exists a subset $\Tilde{A}\subseteq V$ such that $\mathbf{\Tilde{A}}=\{\Tilde{A}\}\in$ $\mathbf{ID}_\mathcal{G}(S,V\setminus S)$ and
    $\mathbf{C}(\mathbf{\Tilde{A}})\leq \mathbf{C}(\mathbf{A})$.
\end{restatable}

Theorem \ref{thm:singleton} indicates that when $\mathcal{G}_{[S]}$ is a c-component, the minimum-cost intervention problem in Equation \ref{eq:optimization} reduces to the problem of finding a single intervention set $A^*$ such that $Q[S]$ is identifiable from $Q[V\setminus A^*]$.
More formally, the optimization in \eqref{eq:optimization} reduces to the following problem,
\begin{equation}\label{eq:optimization scc}
    A^*_S\in\argmin_{A\in \mathbf{ID_1}(S)}\sum\mathop{}_{\mkern-5mu a\in A} \mathbf{C}(a),
\end{equation}
where $\mathbf{ID_1}(S)
$ is the set of all subsets $A$ of $V$ such that $Q[S]$ is identifiable from $Q[V\!\setminus\!A]$.
Note that $|\mathbf{ID_1}(S)|\leq\! 2^{|V|}$.

For the rest of this section, we discuss the solution to Equation \eqref{eq:optimization scc}.
The following lemma further constrains the worst-case cardinality of $\mathbf{ID_1(S)}$ to at most $2^{|V|-|S|}$ elements, noting that intervening on variables in $S$ does not help in identifying $Q[S]$.  
That is, we need only to consider all subsets $A$ of $V$, such that $Q[S]$ is identifiable from $Q[V\setminus A]$ and $A\cap S=\emptyset$. 

\begin{restatable}{lemma}{leminsideS}\label{lem: inside S}
    Suppose $S$ is a subset of variables such that $\mathcal{G}_{[S]}$ is a   c-component. 
    If $A\in\mathbf{ID_1}(S)$, then $A\cap S=\emptyset$.
\end{restatable}


\subsection{Hardness}\label{sec: NP-hard}
In this section, we study the complexity of the minimum-cost intervention design problem.
We show that despite the substantial decrease in the search space complexity of the optimization achieved through Theorem \ref{thm:singleton} (from Eq.~\ref{eq:optimization} to Eq.~\ref{eq:optimization scc}), the minimum-cost intervention problem in Equation \eqref{eq:optimization scc} remains NP-hard.
More precisely, we show that there exists a polynomial-time reduction from the Weighted Minimum Vertex Cover (WMVC) problem to the min-cost intervention problem. 
For the sake of completeness, we formally define the WMVC problem.
\begin{definition}[WMVC]
    Given an undirected graph $\mathcal{H}=(V_\mathcal{H},E_\mathcal{H})$ and a weight function $\omega:V_\mathcal{H}\to \mathbbm{R}^{\geq0}$, a vertex cover is a subset $A\subseteq V_\mathcal{H}$ such that $A$ covers all the edges of $\mathcal{H}$, i.e., for any edge $\{x,y\}\in E_\mathcal{H}$, at least one of $x$ or $y$ is a member of $A$.
    The weighted minimum vertex cover problem's objective is to find a set $A^*$ among all vertex covers that minimizes $\sum_{a\in A}\omega(a)$.
\end{definition}
WMVC is known to be NP-hard \citep{karp1972reducibility}.
Even finding an approximation within a factor of 1.36 to this problem is NP-hard \citep{dinur2005hardness}. 
In fact, there is no known polynomial-time algorithm to approximate WMVC problem within a constant factor less than two\footnote{Factor 2 approximation algorithms appear in \citep{garey1979computers,papadimitriou1998combinatorial}.}.
Indeed, WMVC remains NP-hard even for bounded-degree graphs \citep{garey1974some}.
The following theorem shows that all these statements also hold for the min-cost intervention problem.

\begin{restatable}{theorem}{thmNPcomplete} \label{thm:NP-hard}
    WMVC problem is reducible to a minimum-cost intervention problem in polynomial time.
\end{restatable}

Note that since the ID algorithm proposed by \citet{shpitser2006identification} can be employed to verify a solution to our problem in polynomial time, the minimum-cost intervention problem is in NP.
Therefore, we have the following corollary.
\begin{corollary}
    The minimum-cost intervention problem is NP-complete.
\end{corollary}

\begin{restatable}{remark}{remconstantcost}\label{rem: constant cost}
    The unweighted version of WMVC problem (i.e., when the weight function is given by $\omega(\cdot)=1$) can be reduced to a minimum-cost intervention problem with the constant cost function $\mathbf{C}(\cdot)=1$ in polynomial time. 
    Consequently, the NP-completeness does not stem from arbitrary choice of cost functions.
    This claim is formally proved in Appendix \ref{apdx: proofs}.
\end{restatable}

A result more general than Theorem \ref{thm:NP-hard} is provided below, which shows that the minimum-weight hitting set (MWHS) problem can also be reduced to the min-cost intervention problem.
\begin{definition}[MWHS]\label{def:mwhs}
    Let $V=\{v_1,...,v_n\}$ be a set of objects along with a weight function $\omega:V\to\mathbbm{R}^{\geq0}$.
    Given a collection of subsets of $V$ such as $\mathbf{F}=\{F_1,...,F_k\}$, $F_i\subseteq V,\  1\leq i\leq k$, a hitting set for $\mathbf{F}$ is a subset $A\subseteq V$ such that $A$ hits all the sets in $\mathbf{F}$, i.e., for any $1\leq i\leq k$, $A\cap F_i\neq\emptyset$.
    The weighted minimum hitting set problem's objective is to find a set $A^*$ among all hitting sets that minimizes $\sum_{a\in A}\omega(a)$.
\end{definition}
\begin{restatable}{theorem}{thmreductiontwo}\label{thm:reduction2}
 MWHS problem is reducible to minimum-cost intervention problem in polynomial time.   
\end{restatable}
The significance of Theorem \ref{thm:reduction2} compared to Theorem \ref{thm:NP-hard} is that MWHS is not only NP-hard to solve, but also NP-hard to approximate within a factor better than a logarithmic factor \citep{feige1998threshold}.
This results in the following corollary.
\begin{corollary}
    There is no $(1-o(1))\ln{\vert V\vert}$ approximation scheme for the minimum-cost intervention problem, unless NP has quasi-polynomial-time algorithms.
\end{corollary}

On the other hand, as with any other NP-complete problem, certain instances of the minimum-cost intervention problem can be solved in polynomial-time.
An interesting group of such instances are discussed in Appendix \ref{apdx: special case}.
Despite being restrictive, these special cases might provide useful insights for finding efficient algorithms in more general settings.
Naturally, Theorem \ref{thm:reduction2} implies that the algorithms proposed in this paper can aid to solve some other problems in the NP class.


\subsection{Minimum Hitting Set Formulation}\label{sec: hitsetform}
In this section, we propose a formulation of the minimum-cost intervention problem in terms of the minimum-weight hitting set (MWHS) problem.
This formulation will allow us to find algorithms to solve or approximate our problem in later sections.

It is known that special structures, called \emph{hedges} that are formed for $Q[S]$ in $\mathcal{G}$ prevent the identifiability of the causal effect $Q[S]$ \citep{shpitser2006identification}. 
On the other hand, intervening on a vertex of a hedge allows us to eliminate it from the graph. Hence, the problem of identifying $Q[S]$ is equivalent to finding a subset of vertices that hits all the hedges formed for $Q[S]$.
In other words, the minimum-cost intervention problem can be reformulated as a MWHS problem.
For simplicity, here, we use a slightly modified definition of a hedge.
In Appendix \ref{apdx: prelim}, we show that it is equivalent to the original definition in \citep{shpitser2006identification}. 

\begin{restatable}{definition}{defhedge}\text{\normalfont(Hedge)}\label{def: hedge}
    Let $\mathcal{G}$ be a semi-Markovian graph and  $S$ be a subset of its vertices such that $\mathcal{G}_{[S]}$ is a  c-component.
    A subset $F$ is a hedge formed for $Q[S]$ in $\mathcal{G}$ if $S\subsetneq F$, $F$ is the set of ancestors of $S$ in $\mathcal{G}_{[F]}$, and $\mathcal{G}_{[F]}$ is a c-component.
\end{restatable}

As an example, suppose $S=\{s_1,s_2\}$ in the causal graph of Figure \ref{fig:semi-markov}. In this case, $\mathcal{G}_{[S]}$ is a c-component and $\{s_1,s_2,v_1,v_2\}$ and $\{s_1,s_2,v_2\}$ are two hedges formed for $Q[S]$.

Using the result of \citep{shpitser2006identification}, the following Lemma connects the minimum-cost intervention problem to the minimum-weight hitting set problem.

\begin{restatable}{lemma}{lemhitsetform}\label{lem: hitsetform}
Let $\mathcal{G}$ be a semi-Markovian graph with vertex set $V$, along with a cost function $\mathbf{C}:V\to\mathbbm{R}^{\geq0}$.
Let $S$ be a subset of $V$ such that $\mathcal{G}_{[S]}$ is a   c-component.
Suppose the set of all hedges formed for $Q[S]$ in $\mathcal{G}$ is $\{F_1,...,F_m\}$.
Then $A_S^*$ is a solution to Equation \eqref{eq:optimization scc} if and only if it is a solution to the MWHS problem for the sets $\{F_1\setminus S,...,F_m\setminus S\}$, with the weight function $\omega(\cdot):=\mathbf{C}(\cdot)$.
\end{restatable}

Lemma \ref{lem: hitsetform} suggests that designing an intervention to identify $Q[S]$ can be cast as finding a set that intersects (hits) with all the hedges formed for $Q[S]$.
A brute-force algorithm to find the minimum-cost intervention (Equation \eqref{eq:optimization scc}) is then to first enumerate all hedges formed for $Q[S]$ in $\mathcal{G}$ and solve the corresponding minimum hitting set problem.

Solving MWHS, which itself is equivalent to the set cover problem, is known to be NP-hard \citep{karp1972reducibility,bernhard2008combinatorial}.
However, there exist greedy algorithms that can approximate the optimal solution up to a logarithmic factor \citep{johnson1974approximation,chvatal1979greedy}, which has been shown to be optimum in the sense that they achieve the best approximation ratio \citep{feige1998threshold}.
Another approach for tackling MWHS is via linear programming relaxation which achieves the similar approximation ratio as the greedy approach \citep{lovasz1975ratio}.
But even in the case that we use an approximation algorithm for the minimum hitting set formulation of the minimum-cost intervention problem, the task of enumerating all hedges formed for $Q[S]$ in $\mathcal{G}$ requires exponential number of computations in terms of number of the variables.





\subsection{Properties of $A^*_S$}\label{sec: discussion}
In Section \ref{sec: NP-hard}, we proved that the minimum-cost intervention design problem in \eqref{eq:optimization scc} is NP-complete.
Herein, we shall study certain properties of the solution $A^*_S$ that allow us to reduce the complexity of solving \eqref{eq:optimization scc}.
We begin with characterizing a set of variables that we \emph{must} intervene upon to identify $Q[S]$.

Recall that $\PaC{S}$ is the set of parents of $S$ that have a bidirected edge to a variable in $S$.
Note that for a given set $S$, we can construct $\PaC{S}$ in linear time.
The following Lemma indicates that $Q[S]$ is not identifiable unless all of the variables in $\PaC{S}$ are intervened upon.

\begin{restatable}{lemma}{lemdirP}\label{lem: dirP}
Let $\mathcal{G}$ be a semi-Markovian graph with the vertex set $V$, and for $S\subseteq V$, let $\mathcal{G}_{[S]}$ be a c-component. 
For any subset $A\subseteq V$, if $A\in\mathbf{ID_1}(S)$, then 
$
\PaC{S}\subseteq A.
$
\end{restatable}

As a counterpart to Lemma \ref{lem: dirP}, below, we characterize a subset of vertices that do not belong to $A^*_S$.
\begin{definition}[Hedge hull]\label{def: hedge hull}
    Let $\mathcal{G}$ be a semi-Markovian graph and  $S$ be a subset of its vertices such that $\mathcal{G}_{[S]}$ is a c-component.
    The union of all hedges formed for $Q[S]$ is called hedge hull of $S$ and denoted by $Hhull(S, \mathcal{G})$.
\end{definition}
If $\mathcal{G}_{[S]}$ is not a c-component, it can be uniquely partitioned into maximal c-components \citep{tian2002testable}.
Let $S_1,...,S_k$ be the partition of $S$ such that $\mathcal{G}_{[S_1]},...,\mathcal{G}_{[S_k]}$ are the maximal c-components of $\mathcal{G}_{[S]}$.
We define $Hhull(S,\mathcal{G})$ as $Hhull(S,\mathcal{G})=\bigcup_{i=1}^kHhull(S_i,\mathcal{G})$.

\begin{algorithm}[t]
\caption{Find $Hhull(S,\mathcal{G})$, where $\mathcal{G}_{[S]}$ is a c-component.}
\label{alg: hedge hull}
\begin{algorithmic}[1]
    \Function{Hhull}{$S,\mathcal{G}$}
        \State Initialize $F\gets$ set of vertices of $\mathcal{G}$
        \While{True}
            \State $F_1\gets$ connected component of $S$ via bidirected edges in $\mathcal{G}_{[F]}$
            \State $F_2\gets$ ancestors of $S$ in $\mathcal{G}_{[F_1]}$
            \If{$F_2\neq F$}
                \State $F\gets F_2$
            \Else
                \State\Return $F$
            \EndIf
        \EndWhile
    \EndFunction
\end{algorithmic}
\end{algorithm}

\begin{restatable}{lemma}{lemhedgehull}\label{lem: hedge hull}
    Consider $A^*_S$ in Equation \eqref{eq:optimization scc}, then  $A^*_S\subseteq Hhull(S,\mathcal{G})\setminus S$.
\end{restatable}
For a given subset $S$ and a semi-Markovian graph $\mathcal{G}$, Lemmas \ref{lem: dirP} and \ref{lem: hedge hull} bound the solution to the minimum-cost intervention problem as $\PaC{S}\subseteq A^*_S\subseteq Hhull(S,\mathcal{G})$.

In Algorithm \ref{alg: hedge hull}, we propose a method to construct the hedge hull of a given subset $S$.
Lines 4 and 5 of this algorithm can be performed via \emph{depth first search} (DFS) algorithm, which is quadratic in the number of vertices in the worst-case scenario\footnote{Lines 4 and 5 can also be swapped, as the order in which we execute them does not affect the output.}.
On the other hand, the while loop of line 3 can run at most $|V|$ times in the worst case (as long as $F_2\neq F$, at least one vertex will be eliminated from $F$.) Hence, the complexity of this algorithm is\footnote{To be more precise, DFS takes time $\mathcal{O}(\vert V\vert + \vert E\vert)$, where $\vert E\vert$ is the number of edges. Therefore, Alg.~\ref{alg: hedge hull} runs in time $\mathcal{O}(\vert V\vert^2+\vert V\vert \cdot\vert E\vert)$.} $\mathcal{O}(\vert V\vert^3)$.
\begin{restatable}{lemma}{lemhedgehullcorrect}
    Given a semi-Markovian graph $\mathcal{G}$ over $V$ and a subset $S\subseteq V$ such that $\mathcal{G}_{[S]}$ is a c-component, Algorithm \ref{alg: hedge hull} returns $Hhull(S,\mathcal{G})$ in $\mathcal{O}(\vert V\vert^3)$.
\end{restatable}
Next theorem summarizes the results of this Section.
\begin{restatable}{theorem}{lemhhullpac} \label{thm: Hhullpac}
   Let $S$ be a subset of variables such that $\mathcal{G}_{[S]}$ is a c-component. Then,  $A^*_S$ is a solution to \eqref{eq:optimization scc} if and only if both $\PaC{S}\subseteq A^*_S$ and $A^*_S\setminus\PaC{S}$ is a minimum-cost intervention to identify $Q[S]$ in $\mathcal{G}_{[H]}$, where
   \begin{align}\label{eq:h-def}
       H:=Hhull(S,\mathcal{G}_{[V\setminus\PaC{S}]}).
   \end{align}
\end{restatable}

This result suggests that solving \eqref{eq:optimization scc} can be done by first identifying $\PaC{S}$, and then solving a reduced size minimum-cost intervention problem to identify $Q[S]$ in $\mathcal{G}_{[H]}$, where $H$ is given in Equation \eqref{eq:h-def}.
Note that all the minimal hedges of $S$ in $\mathcal{G}_{[H]}$ can be enumerated in $\mathcal{O}(2^{(\vert H\vert-\vert S\vert)})$.
Therefore, if $|H|$ is small, the hedge enumeration task of the brute-force approach in Section \ref{sec: hitsetform} can be done efficiently.
However, the performance of this method deteriorates as the size of $H$ increases.
Next, we propose an algorithm that circumvents the hedge enumeration task to solve the minimum-cost intervention problem more efficiently.

\subsection{Exact Algorithmic Solution to Minimum-cost Intervention Problem}
In this section, we propose an algorithm that can be used both to exactly solve the minimum-cost intervention problem and to approximate it within a logarithmic factor.

As we discussed earlier, the minimum-cost intervention problem can be formulated as a combination of two tasks: enumerating the hedge structures and solving a minimum hitting set for the hedges.
Although the minimum hitting set problem can be solved with polynomial-time approximation algorithms, enumerating all hedges requires  exponential computational complexity. 
To reduce this complexity, we propose Algorithm \ref{alg: ultimate}, that avoids enumerating all hedges formed for $Q[S]$ by utilizing the notion of minimality defined below, and Theorem \ref{thm: Hhullpac}.
We will next explain this.

\begin{definition}[Minimal hedge]
    A hedge $F$ formed for $Q[S]$ in $\mathcal{G}$ is said to be minimal if no subset of $F$ (clearly excluding $F$) is a hedge formed for $Q[S]$ in $\mathcal{G}$.
\end{definition}
As an example, in Figure \ref{fig:semi-markov}, $Q[\{s_1,s_2\}]$ has two hedges: $\{s_1,s_2,v_1,v_2\}$ and $\{s_1,s_2,v_2\}$. In this case, $\{s_1,s_2,v_2\}$ is a minimal hedge. 
Clearly, every non-minimal hedge formed for $Q[S]$ has a subset which is a minimal hedge.
Therefore, hitting all the \emph{minimal} hedges would suffice to identify $Q[S]$.
As a result, for hedge enumeration, whenever we find a subset $F$ that is a hedge formed for $Q[S]$, it is not necessary to consider any super-set of $F$.

Algorithm \ref{alg: ultimate} summarizes our proposed exact algorithm to solve the minimum-cost intervention problem.
It begins with identifying $\PaC{S}$ and the subset $H$ given in \eqref{eq:h-def}.
The main idea of this algorithm is to iterate between discovering a new hedge formed for $Q[S]$ in $\mathcal{G}$, and solving the MWHS problem for the set of already discovered hedges, denoted by $\mathbf{F}$.
The set $\mathbf{F}$ grows each time a new hedge is discovered, up to a point where
the minimum hitting set solution for $\mathbf{F}$ is exactly the solution to the original minimum-cost intervention problem.
This is where the algorithm returns the result by solving the MWHS for $\mathbf{F}$.

The inner loop of Algorithm \ref{alg: ultimate} (lines 7-13) corresponds to the hedge discovery phase.
Within this loop, the algorithm selects a vertex $a$ in $H\setminus S$ with the minimum cost, and removes $a$ from $H$ (resolves the hedge $H$). 
If this hedge elimination makes $Q[S]$ identifiable (i.e., $Hhull(S,\mathcal{G}_{[H\setminus\{a\}]})=S$), it updates $\mathbf{F}$ in line 10. 
Otherwise, it updates $H$ by $Hhull(S,\mathcal{G}_{H\setminus\{a\}})$ in line 13 using Algorithm \ref{alg: hedge hull}.
The reason for updating $\mathbf{F}$ only when $Q[S]$ becomes identifiable is that the hedge discovered in the last step of the inner loop $H$ is a subset of all the hedges discovered earlier within the loop.
Therefore, hitting (eliminating) $H$, also hits all its super-sets. 

\begin{algorithm}[t]
\caption{Exact algorithm for minimum-cost intervention$(S,\mathcal{G})$, where $\mathcal{G}_{[S]}$ is a c-component.}
\label{alg: ultimate}
\begin{algorithmic}[1]
    \Function{MinCostIntervention}{$S,\mathcal{G}, \mathbf{C}(\cdot)$}
        \State $\mathbf{F}\gets\emptyset$ 
        \State $H\gets Hhull(S,\mathcal{G}_{[V\setminus\PaC{S}]})$
        \If{$H=S$}
            \State \textbf{return} $\PaC{S}$
        \EndIf
        \While{True}
            \While{True}
                \State $a\gets\argmin_{a\in H\setminus S}\mathbf{C}(a)$
                \If{$Hhull(S,\mathcal{G}_{[H\setminus\{a\}]})=S$}
                    \State $\mathbf{F}\gets \mathbf{F}\cup\{H\}$
                    \State\Break
                \Else
                    \State $H\gets Hhull(S,\mathcal{G}_{[H\setminus\{a\}]})$ 
                \EndIf
            \EndWhile
            \State $A\gets$ solve min hitting set for $\{F\setminus S\vert F\in\mathbf{F}\}$
            \If{$A\cup\PaC{S}\in\mathbf{ID_1}(S)$}
                \State \textbf{return} $(A\cup\PaC{S})$
            \EndIf
            \State $H\gets Hhull(S,\mathcal{G}_{[V\setminus(A\cup\PaC{S})]})$
        \EndWhile
    \EndFunction
\end{algorithmic}
\end{algorithm}

At the end of the inner loop, it solves a minimum hitting set problem for the constructed $\mathbf{F}$ to find $A$ in line 14. 
If $A\cup\PaC{S}\in\mathbf{ID_1}(S)$, the algorithm terminates and outputs $A\cup\PaC{S}$ as the optimal intervention set.
Otherwise, it updates $H$ using Algorithm \ref{alg: hedge hull} in line 19 and repeats the outer loop by going back to line 6 to discover new hedges formed for $Q[S]$.


In the worst-case scenario, Algorithm \ref{alg: ultimate}  requires exponential number of iterations to form $\mathbf{F}$.
However,  as illustrated in our empirical evaluations in Section \ref{sec:experiment}, the algorithm often finds the solution to the minimum-cost intervention after only a few number of iterations.
This is to say, in practice, discovering only a few hedges and solving the minimum hitting set problem for them suffices to solve the original minimum-cost intervention problem.

\begin{restatable}{lemma}{lemalgultimcorrect}\label{lem: ultimate alg}
    Let $\mathcal{G}$ be a semi-Markovian graph and $S\subseteq V$.
    Algorithm \ref{alg: ultimate} returns an optimal solution to \eqref{eq:optimization scc}.
\end{restatable}


\begin{remark}\label{rem:alg2}
It is noteworthy that this result holds even if $S$ is not a c-component. 
In other words, Algorithm \ref{alg: ultimate} always returns an optimal solution in $\mathbf{ID_1}(S)$. 
We will use this result in Section \ref{sec:genid} to introduce an algorithm for the general setting in which $S$ is an arbitrary subset of variables. \end{remark}

\paragraph{Approximation version.} Note that the minimum hitting set problem in line 15 can be solved approximately using a greedy algorithm \citep{johnson1974approximation,chvatal1979greedy}, which guarantees a logarithmic-factor approximation\footnote{See Appendix \ref{apdx: hit set} for further details.}.
In this case, if polynomially many hedges are discovered before the algorithm stops\footnote{We propose a slightly modified version of Algorithm \ref{alg: ultimate} in Appendix \ref{apdx: hit set} with lower number of calls to the hitting set solver.}, Algorithm \ref{alg: ultimate} returns a logarithmic-factor approximation of the solution in polynomial time.

\paragraph{Anytime version.}
As even the approximation version of Algorithm \ref{alg: ultimate} has exponential time complexity in the worst case, we also propose an anytime version of this algorithm as follows.
Suppose a runtime threshold $\tau$ is specified by the user.
We run Algorithm \ref{alg: ultimate}, and as soon as the time threshold $\tau$ is hit, we return the following set as a solution:
\begin{equation}
    A\cup\PaC{S}\cup Hhull(S,\mathcal{G}_{[V\setminus(A\cup\PaC{S})]})\setminus S,
\end{equation}
where $A$ is the latest set computed in line 14 of the algorithm (the minimum hitting set from the very last iteration of the algorithm).
The following result shows that this set is indeed sufficient to identify $Q[S]$, and provides an upper bound on how far this solution can be from the optimal one.
\begin{restatable}{proposition}{prpanytime}\label{prp:anytime}
    Let $A^*$ be an optimal solution to the minimum-cost intervention problem of Equation \eqref{eq:optimization scc}.
    Let $A$ be the minimum hitting set computed in line 14 of Algorithm \ref{alg: ultimate} in an arbitrary iteration.
    Also, define $H= Hhull(S,\mathcal{G}_{[V\setminus(A\cup\PaC{S})]})\setminus S$.
    Then, $(A\cup\PaC{S}\cup H)\in\mathbf{ID_1}(S)$, and
    \[\mathbf{C}(A^*)\leq\mathbf{C}(A\cup\PaC{S}\cup H)\leq\mathbf{C}(A^*)+\mathbf{C}(H).\]
\end{restatable}
Since the set $H$ and therefore the regret bound $\mathbf{C}(H)$ can be computed efficiently, one can run the algorithm until a desired upper bound is achieved.

\subsection{Heuristic Algorithms}\label{sec: heuristic}
The algorithm discussed in the previous Section provides an exact solution for finding the minimum-cost intervention.
However, it has an exponential runtime in the worst case.
Herein, we develop and present two heuristic algorithms to approximate the solution to the minimum-cost intervention problem in polynomial time. 
We discuss their average-case performance on random graphs.
Also, in Section \ref{sec:experiment}, we evaluate the performance of these algorithms in terms of their runtimes and the optimality of their solutions. 
Further analysis of these heuristic algorithms are provided in Appendix \ref{apdx: heuristic}.
It is noteworthy that these two algorithms utilize the result of Theorem \ref{thm: Hhullpac}, i.e., they initiate with identifying $\PaC{S}$, $H$ in \eqref{eq:h-def}, and then find a minimum-cost intervention set that identifies $Q[S]$ in $\mathcal{G}_{[H]}$. 
Theses two algorithms approximate the minimum-cost intervention problem via a minimum-weight vertex cut (a.k.a. vertex separator) problem.

\begin{definition}(Minimum-weight vertex cut)
    Let $\mathcal{H}$ be a (un)directed graph over the vertices $V$, with a weight function $\omega:V\to\mathbbm{R}^{\geq0}$.
    For two non-adjacent vertices $x,y\in V$, a subset $A\subset V\setminus\{x,y\}$ is said to be a vertex cut for $x-y$, if there is no (un)directed path that connects $x$ to $y$ in $\mathcal{H}_{V\setminus A}$.
    The objective of minimum-weight vertex cut problem is to identify a vertex cut for $x-y$ that minimizes $\sum_{a\in A}\omega(a)$.
\end{definition}
Minimum-weight vertex cut problem can be solved in polynomial time by, for instance, casting it as a max-flow problem\footnote{See Appendix \ref{apdx: heuristic} for details.} and then using algorithms such as Ford-Fulkerson, Edmonds-Karp, or push-relabel algorithm  \citep{ford1956maximal,edmonds1972theoretical,goldberg1988new}. 


\textbf{Heuristic Algorithm 1:} 
For a given graph $\mathcal{G}$ and a subset $S\subseteq V$, this algorithm builds an undirected graph $\mathcal{H}$ with the vertex set $H\cup\{x,y\}$, where $H$ is given in \eqref{eq:h-def}, and $x$ and $y$ are two auxiliary vertices.
For any pair of vertices $\{v_1,v_2\}\in H$, if $v_1$ and $v_2$ are connected with a bidirected edge in $\mathcal{G}$, they will be connected in $\mathcal{H}$.
Vertex $x$ is connected to all the vertices in $\Pa{S}\cap H$, and $y$ is connected to all vertices in $S$.
The output of the algorithm is the minimum-weight vertex cut for $x-y$, with the weight function $\omega(\cdot):=\mathbf{C}(\cdot)$.
Algorithm \ref{alg: heursitic 1} in Appendix \ref{apdx: heuristic} presents the pseudo code for this procedure.
Next result shows that intervening on the output set of this algorithm will identify $Q[S]$, although this set is not necessarily minimum-cost.

\begin{restatable}{lemma}{lemalgheurOnecorrect}\label{lem: alg heuristic 1}
    Let $\mathcal{G}$ be a semi-Markovian graph on $V$ and  $S$ be a subset of $V$ such that $\mathcal{G}_{[S]}$ is a c-component.
    Heuristic Algorithm 1 returns an intervention set $A$ in $\mathcal{O}(\vert V\vert^3)$ such that $A\in\mathbf{ID_1}(S)$.
\end{restatable}

\textbf{Heuristic Algorithm 2:} 
Given a graph $\mathcal{G}$ and a subset $S$, this algorithm builds a directed graph $\mathcal{J}$ as follows: the vertex set is $H\cup\{x,y\}$, where $H$ is given in \eqref{eq:h-def}, and $x$ and $y$ are two auxiliary vertices.
For any pair of vertices $\{v_1,v_2\}\in H$, if $v_1$ is a parent of $v_2$ in $\mathcal{G}$, then $v_1$ will be a parent of $v_2$ in $\mathcal{J}$.
Vertex $x$ is added to the parent set of all vertices in $\BiD{S}\cap H$, and all vertices of $S$ are added to the parent set of  $y$.
The output of this algorithm is the minimum-weight vertex cut for $x-y$, with the weight function  $\omega(\cdot):=\mathbf{C}(\cdot)$.
Algorithm \ref{alg: heursitic 2} in Appendix \ref{apdx: heuristic} summarizes this procedure. 
The following result indicates that intervening on the output set of this algorithm identifies $Q[S]$.

\begin{restatable}{lemma}{lemalgheurTwocorrect}\label{lem: alg heuristic 2}
    Let $\mathcal{G}$ be a semi-Markovian graph on $V$ and  $S$ be a subset of $V$ such that $\mathcal{G}_{[S]}$ is a c-component.
    Heuristic Algorithm 2 returns an intervention set $A$ in  $\mathcal{O}(\vert V\vert^3)$ such that  $A\in\mathbf{ID_1}(S)$.
\end{restatable}

A major difference between the two heuristic algorithms is that Algorithm 2 solves a minimum vertex cut on a directed graph, whereas Algorithm 1 solves the same problem on an undirected graph.
Since the equivalent max-flow problem is easier to solve on directed graphs, Algorithm 2 is preferred, unless the directed edges of $\mathcal{G}$ are considerably denser than its bidirected edges.
As we shall see in experimental evaluations of Section \ref{sec:experiment}, both of these heuristic algorithms perform outstandingly well on randomly generated graphs according to Erdos-Renyi generative model \citep{erdHos1960evolution}, i.e., when every edge is sampled independently.
The following result indicates that our simulation results are theoretically justified.
\begin{restatable}{proposition}{prpheurs}\label{prp:heurs}
    Let $\mathcal{G}$ be a random semi-Markovian graph, where each directed edge exists with probability $p$ and each bidirected edge exists with probability $q$, mutually independently (generalized Erdos-Renyi generative model).
    Suppose $S=\{s\}$, where $s$ is an arbitrary vertex.
    Let $c^*$ be the random variable of the cost of the optimal solution to Equation \eqref{eq:optimization scc}.
    Also let $c_1$ and $c_2$ be the random variables of the cost of the solution returned by Heuristic Algorithm 1 and Heuristic Algorithm 2, respectively.
    Then under this generative model, and with equal cost for vertices,
    \[\begin{cases}
        \mathbbm{E}[c_1]\leq q^{-1}\mathbbm{E}[c^*],\\
        \mathbbm{E}[c_2]\leq p^{-1}\mathbbm{E}[c^*].
    \end{cases}\]
\end{restatable}
\begin{restatable}{corollary}{corheurmin}\label{cor:heurmin}
    Consider an algorithm that runs both of these heuristic algorithms and picks the best solution out of the two.
    Let the cost of this solution be denoted by $c$.
    This algorithm runs in time $\mathcal{O}(\vert V\vert^3)$ in the worst case, and the cost of its solution satisfies the following inequality in the Erdos-Renyi model (see Appendix \ref{apdx: proofs} for the proof.)
    \[\mathbbm{E}[c^*]\leq\mathbbm{E}[c]\leq\min\{p^{-1},q^{-1}\}\mathbbm{E}[c^*].\]
\end{restatable}

We propose another algorithm which uses a greedy approach to solve the minimum-cost intervention problem and discuss its complexity in Appendix \ref{apdx: heuristic}.
This greedy algorithm is preferable to the two aforementioned algorithms in certain special settings.
Additionally, we propose a polynomial-time post-process in Appendix \ref{apdx: heuristic} to improve the solution returned by our three heuristic algorithms.


\section{General Subset Identification}\label{sec:genid}
So far we have discussed the minimum-cost intervention design problem for subset $S$, where the induced subgrah  $\mathcal{G}_{[S]}$ is a c-component.
In this section, we study the general case in which $S$ is an arbitrary subset of variables and show that the minimum-cost intervention design problem for $S$ requires solving a set of instances of the problem for subsets of $S$ such as $S_i$ where $\mathcal{G}_{[S_i]}$ is a single c-component.

The main challenge in the general case is that Theorem \ref{thm:singleton} is no longer valid. 
Thus, the minimum-cost intervention design problem in \eqref{eq:optimization} cannot be reduced to \eqref{eq:optimization scc}. 
As an example, consider Figure \ref{fig: not singlton}.
In this causal graph, the minimum-cost intervention to identify $Q[S]$
for $S:=\{s_1,s_2,s_3\}$, is $\mathbf{A}^*_{S,V\setminus S}=\{\{s_1\},\{s_2\}\}$ with the cost $\mathbf{C}(s_1)+\mathbf{C}(s_2)=2$. 
However, any singleton intervention that can identify $Q[S]$, i.e., $A\in\mathbf{ID_1}(S)$ has a cost of at least 10.
More importantly, the union of the sets in $\mathbf{A}^*_{S,V\setminus S}$, i.e., $\{s_1,s_2\}$ does not belong to $\mathbf{ID_1}(S)$.
In other words, intervening on $\{s_1,s_2\}$ does not identify $Q[S]$
(Lemma \ref{lem: inside S}).

In many applications, it is reasonable to assume that in order to identify $Q[S]=P(S|do(V\setminus S))$, intervening on elements of $S$, i.e., the outcome variables, is not desirable. 
In other words, $\mathbf{C}(s)=\infty$, for all $s\in S$.
Under this assumption, we show that instances similar to Figure \ref{fig: not singlton} cannot occur and results analogous to Theorem \ref{thm:singleton} can be established.

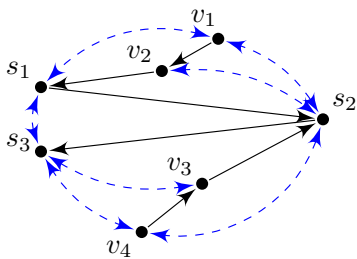
\begin{figure}[t]
    \centering
	\tikzstyle{block} = [circle, inner sep=1.5pt, fill=black]
	\tikzstyle{input} = [coordinate]
	\tikzstyle{output} = [coordinate]
	\resizebox{0.34\textwidth}{!}{
        \begin{tikzpicture}
        \tikzset{edge/.style = {->,> = latex'},line width=1.4pt}
        \node[block](s1) at (-0.5,2.5) {};
        \node[block](s2) at (3,2.1) {};
        \node[block](s3) at (-0.5,1.7) {};
        \node[block](v2) at (1,2.7) {};
        \node[block](v1) at (1.7,3.1) {};
        \node[block](v3) at (1.5,1.3){};
        \node[block](v4) at (0.75,0.7) {};
        
        \node[] ()[above left=-0.1cm and -0.1cm of s1]{$s_1$};
        \node[] ()[above right=-0.1cm and -0.1cm of s2]{$s_2$};
        \node[] ()[above left=-0.2cm and -0.1cm of s3]{$s_3$};
        \node[] ()[above left=-0.1cm and -0.1cm of v2]{$v_2$};
        \node[] ()[above left=-0.05cm and -0.2cm of v1]{$v_1$};
        \node[] ()[above left=-0.1cm and -0.1cm of v3]{$v_3$};
        \node[] ()[below left=-0.1cm and -0.1cm of v4]{$v_4$};
        
        \draw[edge] (s1) to (s2);
        \draw[edge] (s2) to (s3);
        \draw[edge] (v4) to (v3);
        \draw[edge] (v3) to (s2);
        \draw[edge] (v1) to (v2);
        \draw[edge] (v2) to (s1);
        
        \draw[edge, color=blue, dashed, style={<->}, bend left=20] (v3) to (s3);
        \draw[edge, color=blue, dashed, style={<->}, bend left=20] (v4) to (s3);
        \draw[edge, color=blue, dashed, style={<->}, bend right=40] (v4) to (s2);
        \draw[edge, color=blue, dashed, style={<->}, bend right=20] (s1) to (s3);
        \draw[edge, color=blue, dashed, style={<->}, bend left=20] (v2) to (s2);
        \draw[edge, color=blue, dashed, style={<->}, bend left=20] (v1) to (s2);
        \draw[edge, color=blue, dashed, style={<->}, bend right=30] (v1) to (s1);
    \end{tikzpicture}
    }
    \caption{An example where the optimal interventions collection is not a singleton, i.e., a solution to \eqref{eq:optimization scc} is not a solution to \eqref{eq:optimization}. In this example, the cost of intervening on each of $\{s_1,s_2,s_3\}$ is 1, whereas the cost of intervening on each of $\{v_1,v_2,v_3,v_4\}$ is 5.}    \label{fig: not singlton}
\end{figure}

\begin{restatable}{theorem}{thmsingletongeneral} \label{thm:singleton general}
    Suppose $S$ is a subset of variables such that $\mathbf{C}(s)=\infty$ for any $s\in S$.
    Let $\mathbf{A}=\{A_1, A_2,...,A_m\}$ be a collection of subsets such that $\mathbf{A}\in$ $\mathbf{ID}_\mathcal{G}(S,V\setminus S)$ and $m>1$.
    Then there exists a singleton intervention $\tilde{A}$ such that $\mathbf{\Tilde{A}}=\{\Tilde{A}\}\in$ $\mathbf{ID}_\mathcal{G}(S,V\setminus S)$ and
    $\mathbf{C}(\mathbf{\Tilde{A}})\leq \mathbf{C}(\mathbf{A})$.
\end{restatable}

Theorem \ref{thm:singleton general} implies that the general problem can be solved exactly analogous to the case where $\mathcal{G}_{[S]}$ is a c-component, if intervention on $S$ is not allowed.

We now turn to proposing an exact solution to the minimum-cost intervention problem in (\ref{eq:optimization}) in the general case.
Let $S_1,...,S_k$ be subsets of $S$ such that $\bigcup_i S_i=S$ and $\mathcal{G}_{[S_1]},...,\mathcal{G}_{[S_k]}$ are the maximal c-components of $\mathcal{G}_{[S]}$.
It is known that $Q[S]$ is identifiable in $\mathcal{G}$ if and only if $Q[S_i]$s are identifiable in $\mathcal{G}$ for all $1\leq i\leq k$ \citep{tian2002testable}.
This observation is formalized below.
\begin{restatable}{observation}{lemexitpartition}\label{obs:1}
    Let $\mathcal{G}$ be a semi-Markovian graph and $S$ be a subset of its vertices.
    Suppose $\mathbf{A}^*_{S,V\setminus S}$ is a min-cost interventions collection to identify $Q[S]$ in $\mathcal{G}$.
    If $\mathcal{G}_{[S_j]}$ is a maximal c-component of $\mathcal{G}_{[S]}$, then there exists $A\in\mathbf{A}^*_{S,V\setminus S}$ such that $A\in\mathbf{ID_1}(S_j)$ (see Theorem 1 of \citealp{kivva2022revisiting}.)
\end{restatable}

\begin{algorithm}[t]
\caption{Naive general algorithm for minimum-cost intervention, when $\mathcal{G}_{[S]}$ is not necessarily a c-component.}
\label{alg: general heuristic}
\begin{algorithmic}[1]
    \Function{NaiveMinCost}{$S,\mathcal{G},\mathbf{C}(\cdot)$}
        \State $\mathbf{A}^* \gets null$
        \State minCost $\gets\infty$
        \State $\{S_1,...,S_k\}\gets $ maximal c-components of $\mathcal{G}_{[S]}$
        \For{any partition of $\{S_1,...,S_k\}$ as $S^{(1)},...,S^{(t)}$ }
            \State $\mathbf{A}\gets \{\}$
            \State cost $\gets 0$
            \For{$i$ from $1$ to $t$}
                \State $A_i\gets $ min-cost intervention set in $\mathbf{ID_1}(\underline{S}^{(i)})$
                \State $\mathbf{A}\gets\mathbf{A}\cup \{A_i\}$
                \State cost $\gets$ cost $+\mathbf{C}(A_i)$
            \EndFor
            \If{cost $<$ minCost}
                \State $\mathbf{A}^*\gets \mathbf{A}$
                \State minCost $\gets$ cost
            \EndIf
        \EndFor
        \State\Return $\mathbf{A}^*$
    \EndFunction
\end{algorithmic}
\end{algorithm}

Recall that $S_1,\dots,S_k$ are the subsets of $S$ such that $\mathcal{G}_{[S_i]}$ is a maximal c-component for each $1\leq i\leq k$.
Let $S^{(1)},...,S^{(\ell)}$ denote a partitioning of  $\{S_1,...,S_k\}$. 
That is, for each $j$, $S^{(j)}$ is a subset of the set $\{S_1,...,S_k\}$, $S^{(j)}\!\cap \!S^{(i)}\!\!\!=\!\!\emptyset$ for $i\!\!\neq\!\! j$, and $\bigcup_{j=1}^\ell S^{(j)}\!=\!\{S_1,...,S_k\}$. 
Furthermore, we denote the set of all vertices in partition $S^{(j)}$ by $\underline{S}^{(j)}$.
As an example, in Figure \ref{fig: not singlton}, set  $S\!=\!\{\!s_1,s_2,s_3\!\}$ consists of two maximal c-components $\mathcal{G}_{[S_1]}$ and $\mathcal{G}_{[S_2]}$, where $S_1\!\!=\!\!\{s_1,s_3\}$ and $S_2\!\!=\!\!\{s_2\}$. 
There are two different ways to partition $\{S_1, S_2\}$. 
One is $S^{(1)}\!=\!\{S_1,S_2\}$. 
The other is $S^{(1)}\!\!=\!\{S_1\}$ and $S^{(2)}\!\!=\!\{S_2\}$.
For the first partition, we have $\underline{S}^{(1)}\!\!=\!\{s_1,s_2,s_3\}$. 
Similarly, the second partition will result in $\underline{S}^{(1)}\!\!=\!\!\{s_1,s_3\}$ and $\underline{S}^{(2)}\!\!=\!\!\{s_2\}$.
Suppose $\mathbf{A}^*_{S,V\setminus S} = \{A_1,\dots,A_t\}$.
Based on Observation \ref{obs:1}, the set $\{S_1,\dots,S_k\}$ can be partitioned into $t$ groups, denoted by $S^{(1)},\dots,S^{(t)}$, such that each partition $S^{(j)}$ is identified by intervention on $A_j$ for $1\leq j\leq t$.
More precisely, if $S_i\in S^{(j)}$, then $A_j\in\mathbf{ID_1}(S_i)$.
\begin{restatable}{lemma}{lempartition}\label{lem:partition}
    Suppose $\mathbf{A}^*_{S,V\setminus S} = \{A_1,\dots,A_t\}$ is the minimum-cost intervention to identify $Q[S]$.
    Let $S^{(1)},\dots,S^{(t)}$ be the partitioning of maximal c-components of $\mathcal{G}_{[S]}$, such that c-components of each partition $S^{(j)}$ are identified by intervention on $A_j$ for $1\leq j\leq t$.
    Then, 
    \[A_j \in\argmin_{A\in \mathbf{ID_1}(\underline{S}^{(j)})}\mathbf{C}(A), \quad\quad\forall 1\leq j\leq t.\]
\end{restatable}
Lemma \ref{lem:partition} illustrates the fact that if the partitioning $S^{(1)},\dots,S^{(t)}$ were known a priori, then the task of recovering the optimal collection of interventions in $\mathbf{ID}_\mathcal{G}(S, V\setminus S)$ would reduce to $t$ independent instances of finding optimal intervention sets in $\mathbf{ID_1}(\underline{S}^{(j)})$ for each $1\leq j\leq t$.
We already know from Remark \ref{rem:alg2} that these instances can be solved through Algorithm \ref{alg: ultimate}.
However, the optimal partitioning is not known.
Even the number of groups ($t$) is not known in advance.
A naive approach would be to enumerate every possible partitioning of the maximal c-components $\{S_1,\dots,S_k\}$, and form the optimal intervention collection corresponding to each partitioning through successive runs of Algorithm \ref{alg: ultimate}.
This approach is summarized as Algorithm \ref{alg: general heuristic}.
The following result proves the soundness of this algorithm.

\begin{restatable}{proposition}{thmalggeneral}\label{thm: general alg}
    Given a semi-Markovian graph $\mathcal{G}$ and a subset $S$ of its vertices, Algorithm \ref{alg: general heuristic} with Algorithm \ref{alg: ultimate} used as a subroutine in line (9) returns an optimal solution to the min-cost interventions collection to identify $Q[S]$ in $\mathcal{G}$.
\end{restatable}

The number of ways to partition the maximal c-components grows as $\mathcal{O}(m^{\log\log m})$, where $m=2^k$ is the number of subsets of $\{S_1,\dots,S_k\}$\footnote{To be more precise, the number of partitions (generally referred to as Bell number) has been shown to have a growth rate of $\mathcal{O}(m^{\log(\frac{0.792\log m}{\log\log m})})$ \citep{berend2010improved}.}.
Further, Algorithm \ref{alg: ultimate} has to be run as many times as the number of groups in each partition to find the optimal collection of interventions corresponding to that partitioning; that is, if $\{S_1,\dots,S_k\}$ is partitioned into $\ell$ groups, Algorithm \ref{alg: ultimate} is run $\ell$ times for one such partitioning.
Although the number of c-components and therefore number of different partitionings is small in general, it is crucial to minimize the number of runs of Algorithm \ref{alg: ultimate} as its runtime might be limiting.
In what follows, we introduce two algorithms, including an exact, and an approximation algorithm which require running Algorithm \ref{alg: ultimate} only $m-1$ times.

\subsection{Minimum Set Cover Formulation}
Let $\Gamma=\{S_1,...,S_k\}$ be the subsets of $S$ such that $\mathcal{G}_{[S_j]}$ for $1\leq j\leq k$ is a maximal c-component in $\mathcal{G}_{[S]}$.
Also let $\Gamma_1,\dots,\Gamma_{2^k-1}=\Gamma_{m-1}$ denote the non-empty subsets of $\Gamma$ in an arbitrary order.
Define
\begin{equation}\label{eq:ai*}
    A_i^* = \argmin_{A\in\mathbf{ID_1}(\cup_{S_j\in\Gamma_i} S_j)}\mathbf{C}(A),\quad \forall 1\leq i\leq (m-1),
\end{equation}
that is, $A_i^*$ is the optimal single intervention set to identify $Q[S_j]$ for every $S_j\in\Gamma_i$.
Having access to $A)i^*$s, the minimum-cost intervention problem can be cast as a weighted minimum set cover (WMSC) problem.
\begin{definition}[WMSC]
    Let $\Gamma=\{S_1,...,S_k\}$ be a set of objects along with a weight function $\omega:2^\Gamma\to\mathbbm{R}^{\geq0}$, which assigns a weight to each subset of $\Gamma$.
    A set cover for $\Gamma$ is a subset $\mathbf{B}\subseteq 2^\Gamma$ (a subset of the power set of $\Gamma$) such that $\mathbf{B}$ covers the set $\Gamma$, i.e., $\cup_{A\in\mathbf{B}}A=\Gamma$.
    The weighted minimum set cover problem's objective is to find a set $\mathbf{B}^*$ among all set covers that minimizes $\sum_{A\in \mathbf{B}}\omega(A)$.
\end{definition}
The following Lemma indicates that after computing $A_i^*$s as defined in Equation \eqref{eq:ai*}, the rest of the problem is exactly an instance of WMSC.
\begin{restatable}{lemma}{lemmcfpform}\label{lem:mcfpform}
    Let $A_i^*$ be defined as in Equation \eqref{eq:ai*}.
    Let $\mathbf{B}$ denote the solution to the WMSC problem with $\omega(\Gamma_i)=A_i^*$ for every non-empty subset of $\Gamma$ and $\omega(\emptyset)=1$.
    Then $\mathbf{A}=\{A_i^*\vert\Gamma_i\in\mathbf{B}\}$ is an optimal solution to Equation \eqref{eq:optimization}.
\end{restatable}
Although Lemma \ref{lem:mcfpform} suggests an exact algorithm to solve the minimum-cost intervention problem in the general case, the WMSC problem itself is NP-hard\footnote{The greedy approximation } \citep{karp1972reducibility}, and requires time exponential in $m=2^k-1$ in the worst case.
In what follows, we suggest an approximation algorithm which requires time polynomial in $m$, and guarantees a $k$-factor approximation.

\subsection{Approximation Algorithm Based on Minimum-cost Flow}
In this section, we describe an algorithm to approximate the optimal collection of interventions to identify $Q[S]$ for an arbitrary set $S$, where $\mathcal{G}_{[S]}$ might comprise multiple c-components.
This algorithm, which is based on solving an instance of the minimum-cost flow problem, approximates the solution up to a factor of $k$, where $k$ is the number of c-components in $\mathcal{G}_{[S]}$.
We begin with a formal definition of the minimum-cost flow problem (MCFP).
\begin{definition}[MCFP]
    Let $\mathcal{H}=(V,E)$ be a directed graph with a source vertex $w\in V$ and a sink vertex $z\in V$.
    Given a capacity function $\gamma:E\to\mathbbm{R}^{\geq0}$ and a cost function $\zeta:E\to\mathbbm{R}$, the objective of minimum-cost flow problem is to find the way with the least cost to send a given amount $d^*$ of flow from the source to the sink.
    More precisely, if the flow on the edge $(x,y)$ is denoted by $f(x,y)$, it solves the optimization problem
    \[
        \begin{split}
            &\min_{f(\cdot,\cdot)}\sum_{(x,y)\in E}f(x,y)\zeta(x,y),\\
            &\textbf{s.t.}\\
            &0\leq f(x,y)\leq\gamma(x,y),\quad\forall (x,y)\in E,\\
            &\sum_{x\in V}f(x,v) - \sum_{y\in V}f(v,y) = d(v),\quad\forall v\in V,
        \end{split}
    \]
    where $d(z) =-d(w)= d^*$, and  $d(v)=0$ for $v\neq w,z$.
\end{definition}

\begin{figure}[t]
    \centering
	\tikzstyle{block} = [circle, inner sep=1.5pt, fill=black]
	\tikzstyle{input} = [coordinate]
	\tikzstyle{output} = [coordinate]
	\resizebox{0.65\textwidth}{!}{
        \begin{tikzpicture}
        \tikzset{edge/.style = {->,> = latex'},line width=1.4pt}
        \node[block](Ss1) at (0,7.5) {};
        \node[block](Ss2) at (0,4.5) {};
        \node[block](Ss3) at (0,-1.5) {};
        \node[block](Ss1s2) at (0,6) {};
        \node[block](Ss1s3) at (0,1.5) {};
        \node[block](Ss2s3) at (0,0) {};
        \node[block](Ss1s2s3) at (0,3) {};
        \node[block](auxW) at (-7,3) {};
        \node[block](s1) at (3.5,5) {};
        \node[block](s2) at (3.5,3) {};
        \node[block](s3) at (3.5,1) {};
        \node[block](auxZ) at (6,3) {};
        \node[] ()[above right=-0.1cm and -0.1cm of s1]{$S_1$};
        \node[] ()[above right=-0.1cm and -0.1cm of s2]{$S_2$};
        \node[] ()[below right=-0.2cm and -0.1cm of s3]{$S_3$};
        \node[] ()[above left=-0.0cm and -0.4cm of Ss1]{$\Gamma_1=\{S_1\}$};
        \node[] ()[above left=-0.0cm and -0.4cm of Ss2]{$\Gamma_3=\{S_2\}$};
        \node[] ()[above left=-0.0cm and -0.4cm of Ss3]{$\Gamma_7=\{S_3\}$};
        \node[] ()[above left=-0.0cm and -0.4cm of Ss1s2]{$\Gamma_2=\{S_1,S_2\}$};
        \node[] ()[above left=-0.0cm and -0.4cm of Ss1s3]{$\Gamma_5=\{S_1,S_3\}$};
        \node[] ()[above left=-0.0cm and -0.4cm of Ss2s3]{$\Gamma_6=\{S_2,S_3\}$};
        \node[] ()[above left=-0.0cm and -0.4cm of Ss1s2s3]{$\Gamma_4=\{S_1,S_2,S_3\}$};
        \node[] ()[above left=-0.2cm and -0.1cm of auxW]{$w$};
        \node[] ()[above right=-0.2cm and -0.1cm of auxZ]{$z$};
        \path[->]
            (auxW) edge[ bend left=30] node[fill=white, anchor=center, pos=0.5, rotate=35] {\color{black} $1\backslash\mathbf{C}(A_1^*)$} (Ss1);
        \path[->]
            (auxW) edge[bend left=10] node[fill=white, anchor=center, pos=0.5, rotate=14] {\color{black} $1\backslash\mathbf{C}(A_3^*)$} (Ss2);
        \path[->]
            (auxW) edge[ bend right=30] node[fill=white, anchor=center, pos=0.5, rotate=-35] {\color{black} $1\backslash\mathbf{C}(A_7^*)$} (Ss3);
        \path[->]
            (auxW) edge[ bend left=20] node[fill=white, anchor=center, pos=0.5, rotate=27] {\color{black} $2\backslash\frac{\mathbf{C}(A_2^*)}{2}$} (Ss1s2);
        \path[->]
            (auxW) edge[ bend right=10] node[fill=white, anchor=center, pos=0.5, rotate=-14] {\color{black} $2\backslash\frac{\mathbf{C}(A_5^*)}{2}$} (Ss1s3);
        \path[->]
            (auxW) edge[ bend right=20] node[fill=white, anchor=center, pos=0.5, rotate=-27] {\color{black} $2\backslash\frac{\mathbf{C}(A_6^*)}{2}$} (Ss2s3);
        \path[->]
            (auxW) edge[] node[fill=white, anchor=center, pos=0.5] {\color{black} $3\backslash\frac{\mathbf{C}(A_4^*)}{3}$} (Ss1s2s3);
        \draw[edge] (Ss1) to (s1);
        \draw[edge] (Ss2) to (s2);
        \draw[edge] (Ss3) to (s3);
        \draw[edge] (Ss1s2) to (s1);
        \draw[edge] (Ss1s2) to (s2);
        \draw[edge] (Ss1s3) to (s1);
        \draw[edge] (Ss1s3) to (s3);
        \draw[edge] (Ss2s3) to (s2);
        \draw[edge] (Ss2s3) to (s3);
        \draw[edge] (Ss1s2s3) to (s1);
        \draw[edge] (Ss1s2s3) to (s2);
        \draw[edge] (Ss1s2s3) to (s3);
        \draw[edge, bend left=15] (s1) to (auxZ);
        \draw[edge] (s2) to (auxZ);
        \draw[edge, bend right=15] (s3) to (auxZ);
    \end{tikzpicture}
    }
    \caption{Minimum cost flow problem formulation for the general subset $S$, comprising maximal c-components $S=S_1\cup S_2\cup S_3$. All the non-specified edges have capacity of $1$ and cost of $0$, i.e., $\gamma\backslash\zeta=1\backslash0$.}    \label{fig:maxflow}
\end{figure}
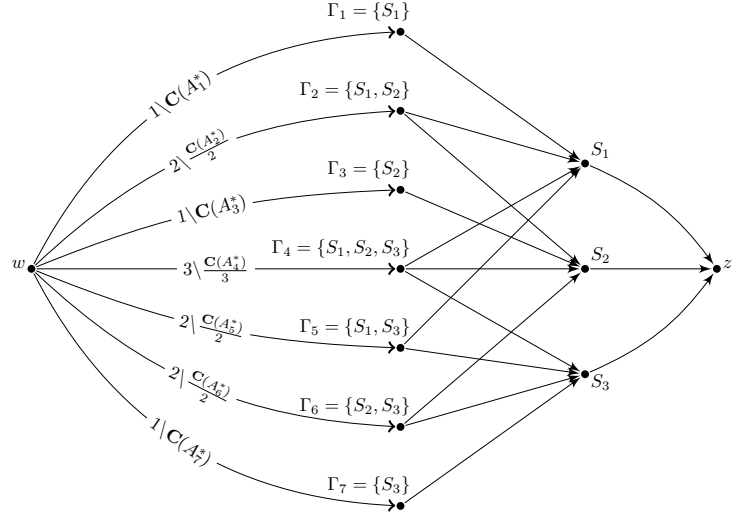

With this definition, we now describe an instance of the minimum-cost flow problem, which is in the core of our approximation algorithm.
We draw a graph where each $S_j$ is represented by a vertex.
Moreover, each $\Gamma_i$ is also represented by a vertex.
We also add two auxiliary vertices $w$ and $z$.
For each subset $\Gamma_i$, we draw a directed edge with capacity $\gamma=1$ and cost $\zeta=0$ from $\Gamma_i$ to $S_j$ if $S_j\in\Gamma_i$.
Analogously, we draw a directed edge with capacity $1$ and cost $0$ from each $S_j$ to $z$.
Finally, for each $\Gamma_i$, we draw a directed edge from $w$ to $\Gamma_i$ with capacity $\vert\Gamma_i\vert$ and cost $\frac{\mathbf{C}(A_i^*)}{\vert\Gamma_i\vert}$.
An example is shown in Figure \ref{fig:maxflow}, where $k=3$.
The pair of numbers on each edge represent the capacity and the cost corresponding to that edge, respectively.
We then solve MCFP on this network, with flow $d^*=k$ from $w$ to $z$.
\begin{remark}
    By definition, all of the capacities are integer values, and the demanded flow $d^*$ is also an integer value.
    Moreover, this problem has a trivial feasible solution where an amount of flow $d^*=k$ is sent from $w$ to the vertex corresponding to the full set $\{S_1,\dots,S_k\}$, and then it is split to each $S_j$ with value of $1$, and then sent from each $S_j$ to $z$.
    As a result, it follows from the integrality theorem that this problem has an integral solution which can be computed in polynomial time \citep{ahuja1988network,kovacs2015minimum}.
\end{remark}
We first find the integral solution ($f^*(\cdot,\cdot)$) to the MCFP described above.
Let the set of $\Gamma_i$s that receive non-zero in-flow in this solution be $\Tilde{\mathbfcal{B}}=\{\Gamma_1,\dots,\Gamma_t\}$ without loss of generality.
We then construct $\mathbfcal{B}^*$ as follows.
\[\mathbfcal{B}^*=\{\Gamma_{i'}=\{S_j\vert 1\leq j\leq k, f^*(\Gamma_i,S_j)>0\}\vert 1\leq i\leq t\}.\]
In words, we replace each subset $\Gamma_i$ with a subset $\Gamma_{i'}\subseteq\Gamma_i$ such that only those $S_j$s are kept in $\Gamma_{i'}$ that $f^*(\Gamma_i,S_j)>0$.
Finally, the solution $\mathbf{A}$ to identify $Q[S]$ is the set of $A_{i'}^*$s that their corresponding $\Gamma_{i'}$ appears in $\mathbfcal{B}^*$.
That is, $\mathbf{A} = \{A^*_{i'}\vert \Gamma_{i'}\in\mathbfcal{B}^*\}$.
This approach is summarized as Algorithm \ref{alg:mcfp}.
We refer to the proof of Proposition \ref{prp:mcfpalg} in Appendix \ref{apdx: proofs} for detailed discussion on correctness and approximation ratio of this algorithm.

\begin{algorithm}[t]
\caption{MCFP-based approximation algorithm for minimum-cost intervention, when $\mathcal{G}_{[S]}$ is not necessarily a c-component.}
\label{alg:mcfp}
\begin{algorithmic}[1]
    \Function{}{}
        \State $\Gamma=\{S_1,...,S_k\}\gets $ maximal c-components of $\mathcal{G}_{[S]}$
        \State $\Gamma_1,\dots,\Gamma_{2^k-1}\gets$ non-empty subsets of $\Gamma$
        \State $\mathcal{H}\gets$ a directed graph with one vertex for each $S_j$, one vertex for each $\Gamma_i$, and two auxiliary vertices $w,z$
        \For{$j$ from $1$ to $k$}
            \State add edge $(S_j,z)$ with capacity $\gamma=1$ and cost $\zeta=0$
        \EndFor
        \For{$i$ from $1$ to $2^k-1$ }
            \State compute $\mathbf{A}_i^*$ based on Eq.~\eqref{eq:ai*}, using Algorithm \ref{alg: ultimate}
            \State add edge $(w,\Gamma_i)$ with capacity $\gamma=\vert\Gamma_i\vert$ and cost $\zeta=\frac{\mathbf{C}(A_i^*)}{\vert\Gamma_i\vert}$
            \State add edge $(\Gamma_i,S_j)$ for each $j$ if $S_j\in\Gamma_i$, with capacity $\gamma=1$ and cost $\zeta=0$
        \EndFor
        \State $f^*\gets$ integral solution for MCFP on $\mathcal{H}$, sending an amount of flow $k$ from $w$ to $z$
        \State $\Tilde{\mathbfcal{B}}\gets\{\Gamma_i\vert f^*(w, \Gamma_i)>0\}$
        \State $\mathbfcal{B}^*\gets\{\Gamma_{i'}=\{S_j\vert 1\leq j\leq k, f^*(\Gamma_i,S_j)>0\}\vert\Gamma_i\in\Tilde{\mathbfcal{B}}\}$
        \State $\mathbf{A}\gets\{A_{i'}^*\vert \Gamma_{i'}\in\mathbfcal{B}^*\}$
        \State\Return $\mathbf{A}$
    \EndFunction
\end{algorithmic}
\end{algorithm}

\begin{restatable}{proposition}{prpmcfpalg}\label{prp:mcfpalg}
    Algorithm \ref{alg:mcfp} returns a solution $\mathbf{A}\in\mathbf{ID}_\mathcal{G}(S,V\setminus S)$, such that $\mathbf{C}(\mathbf{A})\leq k\mathbf{C}(\mathbf{A}_{S,V\setminus S}^*)$, where $k$ and $\mathbf{A}_{S,V\setminus S}^*$ are the number of c-components in $\mathcal{G}_{[S]}$, and an optimal solution to Equation \eqref{eq:optimization}, respectively.
\end{restatable}
It is noteworthy that Algorithm \ref{alg:mcfp} requires calling Algorithm \ref{alg: ultimate} exactly $(m-1)$ times, and then solving MCFP which has an overhead of $\mathcal{O}(m^{2.5})$ operations, as it can be solved as a linear program \citep{vaidya1989speeding}.
Since the number of maximal c-components of $\mathcal{G}_{[S]}$ is in general small compared to the total number of variables in the system, the bottleneck of the computational complexity would be determining $A_i^*$s.

\section{Non-linear Cost Models}\label{sec:nonlin}
The results derived so far rely on the core assumption that the intervention costs are linear; that is, $\mathbf{C}(\mathbf{A})=\sum_{A\in\mathbf{A}}\mathbf{C}(A)=\sum_{A\in\mathbf{A}}\sum_{a\in A}\mathbf{C}(a)$, for every collection of subsets of variables $\mathbf{A}$.
In general, however, the intervention costs can be arbitrarily defined for each subset of variables.
As an example, take for instance a study where the examiner can combine two costly experiments using a similar experimental setup to reduce the experiment cost up to fifty percent; or on the contrary, if these two experiments affect each other, they can prevent the examiner from performing them simultaneously, which can be modeled as an infinite cost for the simultaneous experiment.
In the most general form, instead of defining the cost function $\mathbf{C}$ on the set of variables (as we did in Section \ref{sec:termino}),
it must be defined as $\mathbf{C}:2^V\to \mathbbm{R}$, i.e., for each subset of variables separately\footnote{Note that this general form already makes the size of the problem exponential in $\vert V\vert$.
Therefore, the hardness results do not hold anymore.}.
In an even more general setup, one could relax the additivity of costs for collections of subsets: $\mathbf{C}(\mathbf{A})\neq\sum_{A\in\mathbf{A}}\mathbf{C}(A)$.

In this section, we discuss the minimum-cost intervention problem under nonlinear (non-additive) costs, which is a relaxation of our assumptions.
This is in contrast to what we discuss in Appendix \ref{apdx: special case}, where we show how certain improvements can be achieved if further assumptions are made on the cost function.
We first note that certain results of the previous sections can be extended under assumptions milder than linearity of costs.
For instance, Lemma \ref{lem:intersect} and Remark \ref{rem: cost-singleton} hold under sub-additivity, i.e., if $\mathbf{C}(\cdot)$ is such that $\mathbf{C}(A\cup B)\leq \mathbf{C}(A)+\mathbf{C}(B)$ for any $A,B\subseteq V$.
Applying the definition of sub-additivity to $A_i$s in Remark \ref{rem: cost-singleton} recursively yields the same conclusion.

More interestingly, Theorem \ref{thm:singleton} holds under the milder (and more plausible) assumption of non-decreasing costs over collections of interventions\footnote{This assumption implies the positivity of costs.}:
\begin{assumption}[Non-decreasing costs]\label{ass:nondecreasing}
Performing a set of interventions does not decrease the cost of other interventions.
More precisely, for any collection of subsets of $V$ such as $\mathbf{A}$ and any subset $A'\subseteq V$,
\[\mathbf{C}(\mathbf{A}\cup\{A'\})\geq\mathbf{C}(\mathbf{A}).\]
\end{assumption}
\begin{restatable}{proposition}{prpthmfiveasstwo}\label{prp:thm5ass2}
    Theorem \ref{thm:singleton} holds under Assumption \ref{ass:nondecreasing}, even if Assumption \ref{ass:lincost} is violated.
\end{restatable}
As a result, analogous to what we did in Section \ref{sec:singlec}, the search for minimum-cost intervention can still be restricted from $\mathbf{ID}_\mathcal{G}(S,V\setminus S)$ to $\mathbf{ID_1}(S)$ under Assumption \ref{ass:nondecreasing}.
It is also worth mentioning that the minimum hitting set formulation of Section \ref{sec: hitsetform} is valid regardless of the non-linearity of costs.
As a result, Algorithm \ref{alg: ultimate} is still applicable, with the only difference that the weights of minimum hitting set problem must also be defined for subsets rather than objects. Note that the greedy algorithm to solve the minimum hitting set does not guarantee logarithmic-factor approximation anymore.
To sum up, although presenting the results and algorithms discussed in this paper was more straightforward under Assumption \ref{ass:lincost}, most of these results are applicable in general settings with non-linear costs, under mild assumptions such as Assumption \ref{ass:nondecreasing}.

There are limited settings where Assumption \ref{ass:nondecreasing} can be violated, such as when an intervention has negative cost.
This can for instance happen when a recommender system receives sponsorship profits for recommending certain products.
Under linearity of costs (Assumption \ref{ass:lincost}), one would simply perform every possible intervention with negative cost, and reduce the remaining problem into a smaller instance of the same problem with positive costs.
However, if costs are not linear, we might have to investigate every possible combination of interventions in a brute-force manner.

\section{Evaluations}\label{sec:experiment}
\begin{figure}
    \centering
    \begin{subfigure}[b]{0.48\textwidth}
        \centering
        \includegraphics[width=0.96\textwidth]{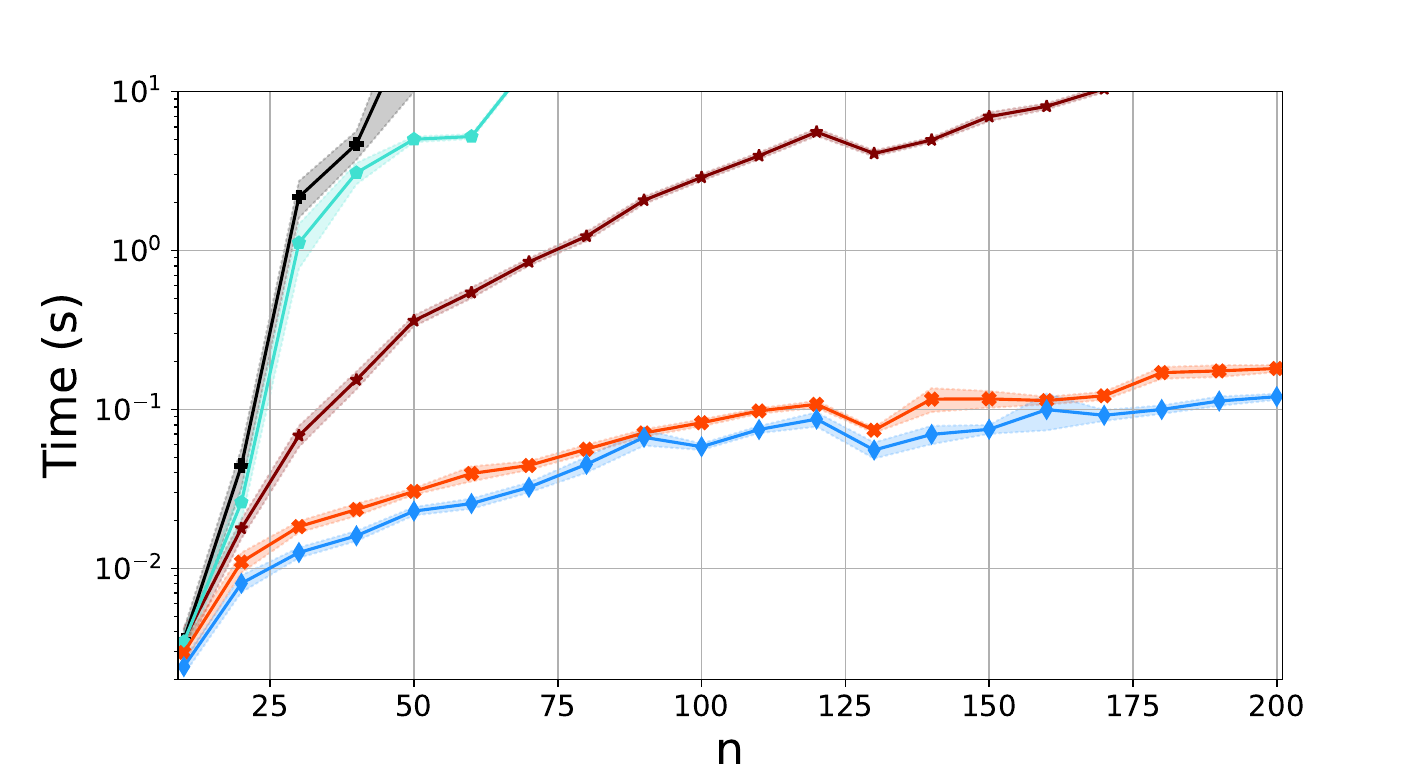}
    \end{subfigure}
    \begin{subfigure}[b]{0.48\textwidth}
        \centering
        \includegraphics[width=0.96\textwidth]{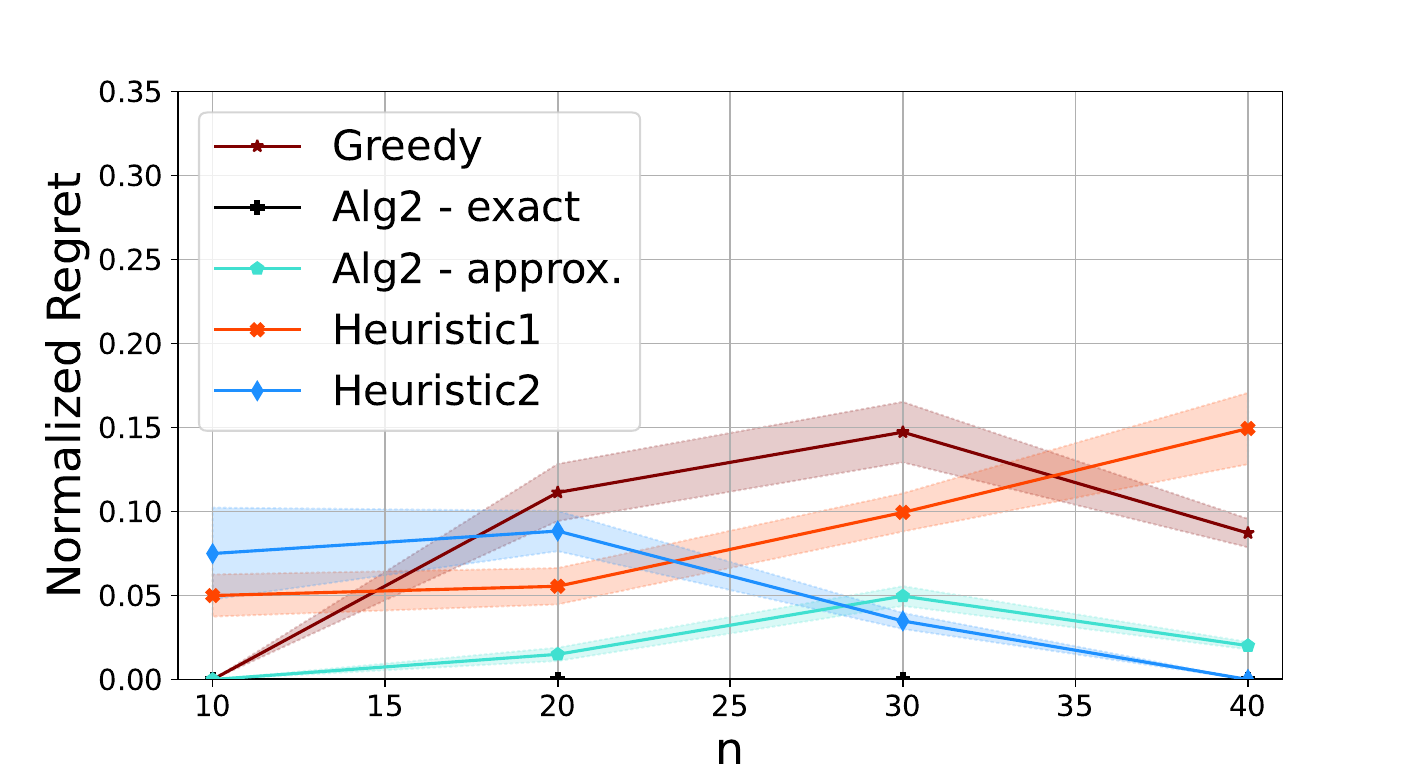}
    \end{subfigure}
    \caption{Evaluation of our algorithms in terms of runtime (Left) and normalized regret (Right). $n$ represents the number of vertices in the causal graph, Alg2-exact and Alg2-approx. solve the minimum hitting set exactly and approximately, respectively, and the greedy algorithm is presented in Appendix \ref{apdx: heuristic}. Alg2-exact obtains zero regret.}
    \label{fig:exp}
\end{figure}

Our evaluation consists of three parts.
First, we evaluate our heuristic algorithms and the approximation version of Algorithm \ref{alg: ultimate} against the exact algorithm, in terms of runtime and normalized regret, when solving the target query $Q[S]$ is such that $\mathcal{G}_{[S]}$ is a c-component.
Throughout this section, the normalized regret of a solution $A$ is defined as $(\mathbf{C}(A)-\mathbf{C}^*)/\mathbf{C}^*$, where $\mathbf{C}^*$ denotes the cost of the optimal minimum-cost solution.
Second, we evaluate our exact approach, Algorithm \ref{alg: ultimate} in terms of the number of hedges it requires to discover for recovering the optimal solution, against the number of total existing hedges formed for $Q[S]$ in $\mathcal{G}$.
This indeed translates to the number of iterations of Algorithm \ref{alg: ultimate}.
Finally, we compare the two algorithms proposed for minimum-cost intervention design when $\mathcal{G}_{[S]}$ comprises multiple c-components, namely Algorithms \ref{alg: general heuristic} and \ref{alg:mcfp}, in terms of runtime.

\subsection{Heuristic and Approximation Algorithms}
For evaluation, we generated causal graphs using the Erdos-Renyi generative model \citep{erdHos1960evolution} as follows.
For a given number of vertices $\vert V\vert =n$, we fixed a causal order over the vertices.
Then, directed edges were sampled with probability $p=0.35$ and bidirected edges were sampled with probability $q=0.25$ between the vertices, mutually independently.
The set $S$ was selected randomly among the last 5\% of the vertices in the causal order such that $\mathcal{G}_{[S]}$ is a c-component.
Intervention costs of vertices were chosen independently at random from $\{1,2,3,4\}$.

For various $n$, we sampled different causal graphs and different subsets $S$ using the above procedure and ran our algorithms\footnote{\url{https://github.com/SinaAkbarii/min_cost_intervention/tree/main}} on them to find the minimum-cost intervention set for identifying $Q[S]$.  
Our performance measures are runtime and normalized regret.
The results are depicted in Figure \ref{fig:exp}. Each curve and its confidence interval is obtained by averaging over $40$ trials. 
As illustrated in Figure \ref{fig:exp}, our proposed heuristic algorithms achieve negligible regret in most of the cases while their runtime are considerably faster than the exact algorithm, i.e., Algorithm \ref{alg: ultimate}.
It is noteworthy that the regret is not necessarily a monotone function of $n$, and it depends on the structure of the causal graph and intervention costs.
For further evaluations of our algorithms (e.g., their sensitivities with respect to $p$ and $q$), see Appendix \ref{apdx: empirical}.

\subsection{Exact Algorithm}
\begin{figure}
    \centering
        \includegraphics[width=0.6\textwidth]{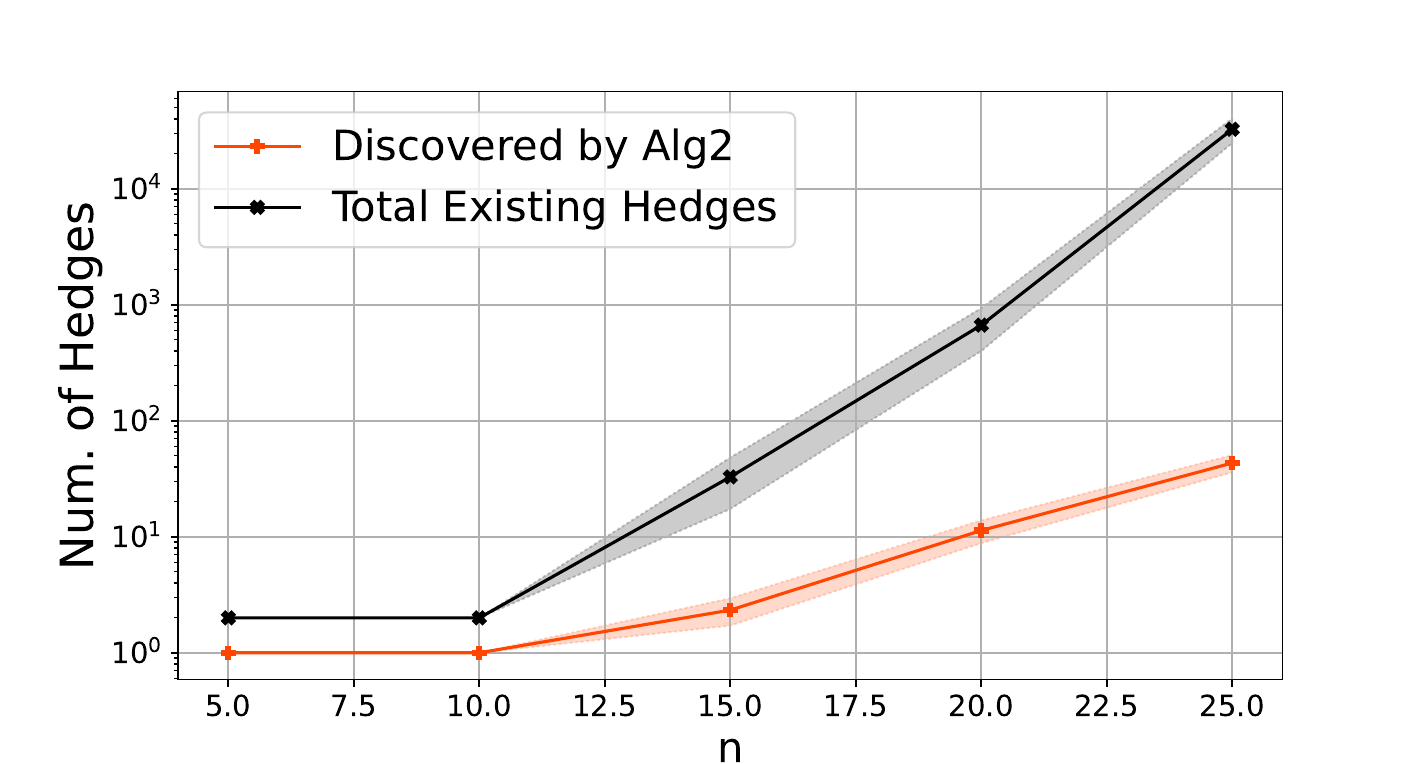}
        \caption{Number of the total hedges formed for $Q[S]$ vs the number of hedges Algorithm \ref{alg: ultimate} discovers until finding the optimal solution.}
    \label{fig:numhedge}
\end{figure}

Recall that Algorithm \ref{alg: ultimate} discovers hedges formed for $Q[S]$ iteratively, one hedge per iteration.
This process continues until \emph{enough} hedges are discovered.
That is, instead of enumerating all of the hedges, Algorithm \ref{alg: ultimate} enumerates only a portion of them, which suffices to find the optimal solution.
Figure \ref{fig:numhedge} demonstrates the number of hedges formed for $Q[S]$ in random graphs of different sizes generated with parameters $p=0.35$ and $q=0.25$ (same setting as above), as opposed to the number of hedges that Algorithm \ref{alg: ultimate} discovered before finding the optimal minimum-cost intervention solution in logarithmic scale.
The number of hedges formed for $Q[S]$ was counted after removing $\PaC{S}$.

\begin{figure}[ht]
    \centering
        \includegraphics[width=0.55\textwidth]{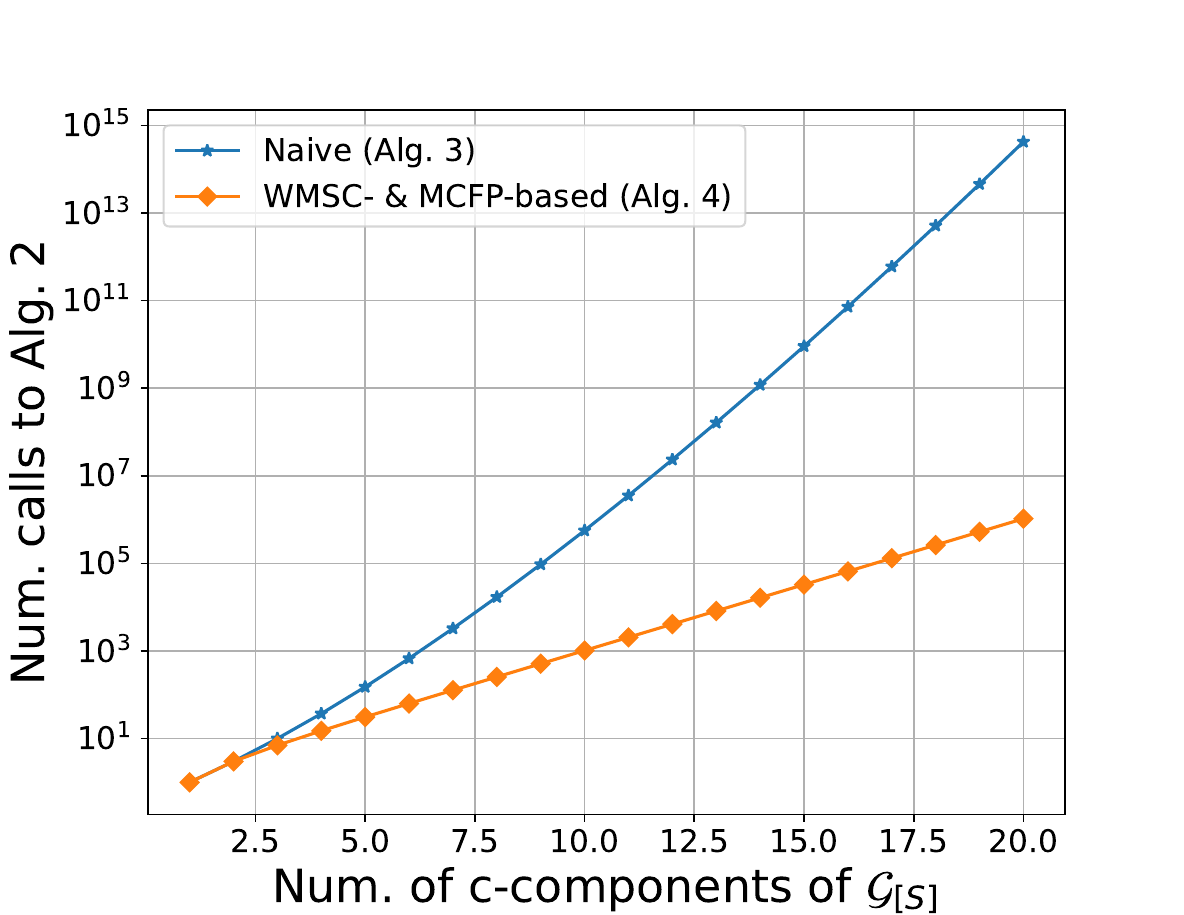}
        \caption{Number of calls to Algorithm \ref{alg: ultimate} as subroutine, vs the number of c-components of $\mathcal{G}_{[S]}$.}
    \label{fig:generalcompare}
\end{figure}
\subsection{General Algorithms}
We discussed two exact algorithms for solving the minimum-cost intervention design problem, namely Algorithm \ref{alg: general heuristic} and the WMSC-based algorithm suggested by Lemma \ref{lem:mcfpform}.
We also proposed an approximation version of the latter as Algorithm \ref{alg:mcfp}.
The time-consuming complex computations in these algorithms boil down to the use of Algorithm \ref{alg: ultimate} as a sub-routine, which is an exponential-time algorithm called exponentially many times.
Herein, we compare their number of calls to Algorithm \ref{alg: ultimate} as a measure of time complexity (line 9 of Algorithm \ref{alg: general heuristic} and line 8 of Algorithm \ref{alg:mcfp}, where the latter shares the same number of calls with the WMSC-based algorithm).
Figure \ref{fig:generalcompare} illustrates the number of calls to Algorithm \ref{alg: ultimate} versus the number of maximal c-components in $\mathcal{G}_{[S]}$ in logarithmic scale.
As can be seen in this figure, the number of calls to Algorithm \ref{alg: ultimate} becomes considerably different as the number of c-components grows.

\section{Concluding Remarks}
We discussed the problem of designing the minimum-cost intervention for causal effect identification in this paper.
We established the NP-completeness of this problem by relating it to two well-known NP-complete problems: minimum vertex cover and minimum hitting set.
When the target query consisted of a single c-component, we proposed a formulation of the minimum-cost intervention problem based on minimum hitting set.
We utilized this formulation to devise an algorithm to solve the minimum-cost intervention problem exactly.
Highlighting the computational complexity of the problem, we also proposed several heuristic algorithms to design an intervention in polynomial time.
For a general query, we reduced the problem to several instances of the single-c-component case.
We proposed an algorithm based on minimum set cover to solve the minimum-cost intervention problem, and a similar algorithm based on minimum-cost flow to approximate it within a factor of the number of c-components.
\section*{Acknowledgements}
This research was in part supported by the Swiss National Science Foundation under NCCR Automation, grant agreement 51NF40\textunderscore180545 and Swiss SNF project 200021\textunderscore204355 /1.
\bibliography{mincost}
\clearpage

\onecolumn
\appendix

This appendix is organized as follows.
\begin{itemize}
    \item In Section \ref{apdx: prelim}, we discuss our definition of a hedge formed for a causal query, and show it is equivalent to the original definition used in the literature.

    \item In Section \ref{apdx:pearl}, we recite the three rules of Pearl's do calculus for the sake of completeness.
    
    \item In Section \ref{apdx: proofs}, we provide the proofs for all of our results stated in the main text and the appendix.
    It is noteworthy that the results stated later in the appendix are also proved in this section.
    
    \item In Section \ref{apdx: special case}, we discuss various cases where the minimum-cost intervention problem can be solved more efficiently (i.e., we provide polynomial time algorithms) under assumptions on the structure of the causal graph, or the cost function.
    
    \item In Section \ref{apdx: heuristic}, we provide further details of our proposed heuristic algorithms.
    
    \item Section \ref{apdx: hit set} includes further details on the approximation algorithm used to solve the hitting set algorithm, as well as a slight modification to Algorithm \ref{alg: ultimate} to reduce the number of calls to solve the hitting set problem.
    
    \item Section \ref{apdx: empirical} provides further details of the experimental setup of this paper, along with further evaluations of the proposed algorithms in this work.
    
\end{itemize}

\section{Definition of Hedge}\label{apdx: prelim}
We used a definition of hedge (Definition \ref{def: hedge}) throughout the paper which was slightly different from the original definition in \citep{shpitser2006identification}.
We provide a formal proof that these two definitions are equivalent.
Following Pearl's notation, for two sets of variables $X$ and $Y$, the graph $\mathcal{G}_{\overline{X}\underline{Y}}$ is defined as the edge subgraph of $\mathcal{G}$, where the edges going into $X$ and the edges going out of $Y$ are deleted.
The definitions of root set, c-component, c-forest and hedge are all adopted from \citep{shpitser2006identification}.
\begin{definition}[C-component]
    Let $\mathcal{G}$ be a semi-Markovian graph.
    $\mathcal{G}$ is a c-component (confounded-component) if a subset of its bidirected edges form a spanning tree over all vertices of $\mathcal{G}$.
\end{definition}
\begin{remark}
    If $\mathcal{G}$ is not a c-component, it can be uniquely partitioned into maximal c-components \citep{tian2002testable}.
\end{remark}
\begin{definition}[Root set]
    We say $R$ is a root set in $\mathcal{G}$ if for every $r\in R$, the set of descendants of $r$ in $\mathcal{G}$ is empty.
\end{definition}
\begin{definition}[C-forest]
    Let $\mathcal{G}$ be a semi-Markovian graph with the maximal root set $R$.
    $\mathcal{G}$ is a $R$-rooted c-forest if $\mathcal{G}$ is a c-component and all the observed variables have at most one child.
\end{definition}
\begin{definition}[Hedge]\label{def: hedge original}
    Let $X,Y$ be two set of vertices in the semi-Markovian graph $\mathcal{G}$.
    Also let $F,F'$ be two $R$-rooted c-forests such that $F'\subset F$, $F\cap X\neq\emptyset$, $F'\cap X=\emptyset$ and $R$ is a subset of ancestors of $Y$ in $\mathcal{G}_{\overline{X}}$.
    Then $F,F'$ form a hedge for $P(Y\vert do(X))$ in $\mathcal{G}$.
\end{definition}
\begin{theorem}[\citealp{shpitser2006identification}]\label{thm: hedge criterion}
If there exists a hedge formed for $P(Y\vert do(X))$ in $\mathcal{G}$, then $P(Y\vert do(X))$ is not identifiable in $\mathcal{G}$.
\end{theorem}
\begin{remark}\label{rem: hedge edge subgraph}
    As mentioned in Theorem 6 of \citep{shpitser2006identification}, if an edge subgraph of $\mathcal{G}$ contains a hedge formed for $P(Y\vert do(X))$, then $P(Y\vert do(X))$ is not identifiable in $\mathcal{G}$.
    In other words, if $P(Y\vert do(X))$ is not identifiable in $\mathcal{G}$, it is not identifiable in any edge super-graph of $\mathcal{G}$ either.
\end{remark}
For the purposes of this paper where we are predominantly considering interventional distributions of the form $Q[S]=P(S\vert do(V\setminus S))$, 
we adapt the definition of hedge and the hedge criterion (Theorem \ref{thm: hedge criterion}) as follows.
Let $S$ be a subset of the vertices of $\mathcal{G}$ such that $\mathcal{G}_{[S]}$ is a c-component.
First note that if $F,F'$ form a hedge for $P(Y\vert do(X))$, then $(F\cup Y), (F'\cup Y)$ clearly form a hedge for $P(Y\vert do(X))$ by definition.
Further, if the two $R$-rooted c-forests $F,F'$ form a hedge for $Q[S]$, the set $R$ must be $S$ itself, as $S$ has no other ancestors in $\mathcal{G}_{\overline{V\setminus S}}$ that can be a member of the root set.
Consequently, $F'=R=S$, and $F$ is a subset of $\mathcal{G}$ containing $S$.
Also taking Remark \ref{rem: hedge edge subgraph} into consideration, a hedge formed for $Q[S]$ in $\mathcal{G}$ can be thought of as the following structure, which is the definition used throughout this paper.
\defhedge*
Definitions \ref{def: hedge original} and \ref{def: hedge} coincide when $\mathcal{G}_{[S]}$ is a c-component, and we used Definition \ref{def: hedge} to simplify the text.
\begin{claim}
A hedge w.r.t. Definition \ref{def: hedge} is formed for $Q[S]$ if and only if a hedge w.r.t. Definition \ref{def: hedge original} is formed for $Q[S]$.
\end{claim}
\begin{proof}
Let $F$ be a hedge formed for $Q[S]$ w.r.t. Definition \ref{def: hedge}.
Taking $F'=S$, $X=V\setminus S$, $Y=S$, and $R=S$, the pair $F,F'$ forms a hedge for $Q[S]=P_X(Y)$ w.r.t. Definition \ref{def: hedge original}.
Conversely, if the pair $F,F'$ is a hedge by definition \ref{def: hedge original}, as argued above, $F'=R=S$.
In this case, by definition of R-rooted c-forest, $F$ is the set of ancestors of $S$ in $\mathcal{G}_{[F]}$, and $\mathcal{G}_{[F]}$ is a c-component, which means that $F$ forms a hedge for $Q[S]$ w.r.t. Defintion \ref{def: hedge}.
\end{proof}
We also make use of the following theorem along our proofs.
\begin{theorem}[\citealp{tian2002testable}]\label{thm: tian}
Let $\mathcal{G}$ be a semi-Markovian graph, and let $H$ be a subset of the observable vertices.
Let $H_1,...,H_k$ denote the maximal c-components of $\mathcal{G}_{[H]}$.
Then $Q[H]$ is identifiable in $\mathcal{G}$, if and only if $Q[H_1],...,Q[H_k]$ are identifiable in $\mathcal{G}$.
\end{theorem}

\section{Pearl's do Calculus}\label{apdx:pearl}
For the sake of completeness, we recite the three rules of Pearl's do calculus \citep{pearl2012calculus}.\\\\
\emph{Rule 1}) Insertion or deletion of observations. If $(Y \perp\!\!\!\perp Z|X,W)_{\mathcal{G}_{\overline{X}}}$, then
\[P(Y|do(X), Z, W) = P (Y|do(X), W).
\]
\emph{Rule 2}) Action and observation exchange. If $(Y \perp\!\!\!\perp Z|X,W)_{\mathcal{G}_{\overline{X}\underline{Z}}}$, then 
\[P(Y|do(X,Z), W) = P (Y|do(X), Z,W).\]
\emph{Rule 3}) Insertion or deletion of actions. If $(Y \perp\!\!\!\perp Z|X,W)_{\mathcal{G}_{\overline{X}\overline{Z(W)}}}$, where $Z(W)$ are vertices in $Z$ that have no descendants in $W$ in $\mathcal{G}_{\overline{X}}$, then
\[P(Y|do(X, Z), W) = P (Y|do(X), W).\]
\section{Proofs} \label{apdx: proofs}
The proofs of the results are presented in the order that they appear in the main text.
The proofs of the statements that appear later in the appendix are also included in this section.
\subsection{Results Appearing in the Main Text}
\lemintersect*
\begin{proof}
Suppose $B_i = V\setminus A_i$ for $1\leq i\leq m$, and define $B_\cap=\cap_{i=1}^kB_i$.
Also suppose that $Q[S]$ is identifiable from the collection $\{Q[B_1],...,Q[B_m]\}$.
We claim that $Q[S]$ is also identifiable from $Q[B_\cap]$.
Suppose not.
Then there exists two structural equation models $M_1$ and $M_2$ on the set of variables $V$ such that $Q^{M_1}[B_\cap]=Q^{M_2}[B_\cap]$, but $Q^{M_1}[S]\neq Q^{M_2}[S]$, where $Q^{M_j}$ is the interventional distribution under model $M_j$.

Now we build two models $M_1'$ and $M_2'$ as follows.
For any $x\in B_\cap$, $x$ has the same equation in $M_j'$ as in $M_j$.
Both in $M_1'$ and $M_2'$, any $x\notin B_\cap$, $x$ is uniformly distributed in $\mathcal{D}(x)$, where $\mathcal{D}(x)$ is the domain of the variable $x$.
Since every $x\notin B_\cap$ is drawn independently of every other variable, for $1\leq i\leq m$ we can write:
\begin{equation}\label{eq: proof lem1-1}\begin{split}
    Q^{M_1'}[B_i] &= Q^{M_1'}[B_\cap]\cdot\Pi_{x\in B_i\setminus B_\cap} Q^{M_1'}[x]\\
    &= Q^{M_1}[B_\cap]\cdot\Pi_{x\in B_i\setminus B_\cap} \frac{1}{\vert\mathcal{D}(x)\vert}\\
    &= Q^{M_2}[B_\cap]\cdot\Pi_{x\in B_i\setminus B_\cap} \frac{1}{\vert\mathcal{D}(x)\vert}\\
    &= Q^{M_2'}[B_\cap]\cdot\Pi_{x\in B_i\setminus B_\cap} Q^{M_2'}[x]\\
    &= Q^{M_2'}[B_i],
\end{split}
\end{equation}
where the second equality follows from the fact that every variable in $B_\cap$ has the same model in $M_1$ and $M_1'$, the third equality follows from the assumption that $Q^{M_1}[B_\cap]=Q^{M_2}[B_\cap]$, and the fourth one is because every variable in $B_\cap$ has the same model in $M_2$ and $M_2'$.

With the same line of reasoning as above,
\begin{equation}\label{eq: proof lem1-2}Q^{M_1'}[S]=Q^{M_1}[S]\neq Q^{M_2}[S]=Q^{M_2'}[S].
\end{equation}
Equations \eqref{eq: proof lem1-1} and \eqref{eq: proof lem1-2} illustrate that $Q[S]$ is unidentifiable from the collection $\{Q[B_1],...,Q[B_m]\}$, which contradicts the assumption of the lemma.
Therefore, $Q[S]$ must be identifiable from $Q[B_\cap]$, or equivalently, $\mathbf{A}=\{A_\cup\}\in\mathbf{ID}_\mathcal{G}(S,V\setminus S)$.
\end{proof}

\thmsingleton*
\begin{proof}
Suppose without loss of generality that $A_i\cap S=\emptyset$ for $1 \leq i\leq k$, and $A_i\cap S\neq\emptyset$ for $k<i\leq m$, for some integer $k$.
We first claim that the following collection is in $\mathbf{ID}_\mathcal{G}(S,V\setminus S)$.
\[\mathbf{\hat{A}}=\{A_1,...,A_k\}\in\mathbf{ID}_\mathcal{G}(S,V\setminus S).\]
Suppose this claim does not hold.
Then from theorem 1 of \citep{kivva2022revisiting}, $Q[S]$ is not identifiable from any of $Q[A_i]$s for $1\leq i\leq k$.
Since for $i>k$, we have $S\not\subseteq A_i$, applying Theorem 1 of \citep{kivva2022revisiting} again, $\mathbf{A}\notin\mathbf{ID}_\mathcal{G}(S,V\setminus S)$, which is a contradiction.

Now defining $A_\cup = (\cup_{i=1}^kA_i)$, from Lemma \ref{lem:intersect}, we know that 
\[\mathbf{\Tilde{A}}=\{A_\cup\}\in\mathbf{ID}_\mathcal{G}(S,V\setminus S).\]
It suffices to show that $\mathbf{C}(\mathbf{\Tilde{A}})\leq \mathbf{C}(\mathbf{A})$, which follows from an identical reasoning to Remark \ref{rem: cost-singleton}:
    \[\begin{split}
    \mathbf{C}(\mathbf{\Tilde{A}}) &=
    \mathbf{C}(A_\cup) =
    \sum_{a\in A_\cup}\mathbf{C}(a)
    \leq\sum_{i=1}^k\sum_{a\in A_i}\mathbf{C}(a)\\&
    =\sum_{i=1}^k\mathbf{C}(A_i)
    =\mathbf{C}(\mathbf{\hat{A}})
    \leq
    \mathbf{C}(\mathbf{A}).
    \end{split}
    \]
\end{proof}

\leminsideS*
\begin{proof}
    Suppose $A\cap S$ is nonempty, and $s\in A\cap S$ is an arbitrary variable.
    Define two models $M_1$ and $M_2$ as follows.
    Every variable in $M_1$ is uniformly drawn from $\{0,1\}$.
    Also, every variable in $M_2$ except $s$ is uniformly drawn from $\{0,1\}$, and $s$ is drawn from $\{0,1\}$ with probabilities $0.4$ and $0.6$, respectively.
    Let $Q^{M_i}$ denote the interventional distributions under model $M_i$, for $i\in\{1,2\}$.
    Clearly, $Q^{M_1}[V\setminus A]=Q^{M_2}[V\setminus A]=\frac{1}{2^{\vert V\vert-\vert A\vert}}$, whereas $Q^{M_1}[S_{s=0}]=\frac{1}{2^{\vert S\vert}}\neq \frac{1}{0.4*2^{\vert S\vert-1}}=Q^{M_2}[S_{s=0}]$, which shows that $Q[S]$ is not identifiable in $\mathcal{G}_{[V\setminus A]}$.
    
\end{proof} 

\thmNPcomplete*
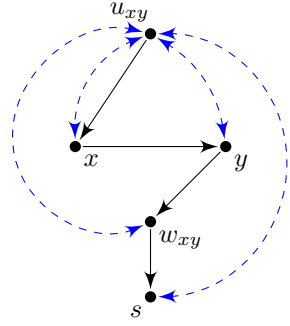
\begin{figure}[t]
    \centering
	\tikzstyle{block} = [circle, inner sep=1.5pt, fill=black]
	\tikzstyle{input} = [coordinate]
	\tikzstyle{output} = [coordinate]
    \begin{tikzpicture}
        \tikzset{edge/.style = {->,> = latex'},line width=1.4pt}
        \node[block](s) at (1.6,0.5) {};
        \node[block](x) at (0.6,2.5) {};
        \node[block](y) at (2.6,2.5) {};
        \node[block](u) at (1.6,4) {};
        \node[block](v) at (1.6,1.5) {};
        
        \node[] (dm1) [below right=0.1cm and 0.6cm of y]{};
        \node[] (dm2) [below left=-0.1cm and 0.6cm of x]{};
        
        \node[] ()[below left=-0.1cm and -0.1cm of s]{$s$};
        \node[] ()[below right=-0.1cm and -0.1cm of x]{$x$};
        \node[] ()[below right=-0.1cm and -0.1cm of y]{$y$};
        \node[] ()[above left=-0.05cm and -0.2cm of u]{$u_{xy}$};
        \node[] ()[below right=-0.1cm and -0.1cm of v]{$w_{xy}$};
        
        \draw[edge] (u) to (x);
        \draw[edge] (x) to (y);
        \draw[edge] (y) to (v);
        \draw[edge] (v) to (s);
        
        \draw[edge, color=blue, dashed, style={<->}, bend left=45] (u) to (dm1) to (s);
        \draw[edge, color=blue, dashed, style={<->}, bend right=35] (u) to (x);
        \draw[edge, color=blue, dashed, style={<->}, bend left=25] (u) to (y);
        \draw[edge, color=blue, dashed, style={<->}, bend right=55] (u) to (dm2) to (v);
    \end{tikzpicture}
    \caption{Reduction from the weighted vertex cover problem to the minimum-cost intervention problem.
    Each edge $\{x,y\}$ in the undirected graph $\mathcal{H}$ is represented by a hedge structure in the semi-Markovian graph $\mathcal{G}$. 
    }
    \label{fig: np reduction}
\end{figure}
\begin{proof}
Suppose an undirected graph $\mathcal{H}=(V_\mathcal{H},E_\mathcal{H})$ along with a weight function $\omega:V_\mathcal{H}\to\mathbbm{R}^{\geq0}$ is given.
We construct a semi-Markovian graph $\mathcal{G}$ along with a cost function $\mathbf{C}$ and prove that the min vertex cover problem in $\mathcal{H}$ is equivalent to the min-cost intervention problem in $\mathcal{G}$ for some set $S$.
The construction is as follows.

We first begin with defining the vertex set of $\mathcal{G}$.
For any vertex $x\in V_\mathcal{H}$, add a vertex $x$ in $\mathcal{G}$.
For any edge $\{x,y\}\in E_\mathcal{H}$, add two vertices $u_{xy}$ and $w_{xy}$ in $\mathcal{G}$.
We will denote the set of all such vertices by $U$ and $W$, respectively.
Finally, add a vertex $s$.
The number of vertices of $\mathcal{G}$ (denoted by $V=V_\mathcal{H}\cup U\cup W\cup\{s\}$) is therefore equal to $(\vert V_\mathcal{H}\vert+2\vert E_\mathcal{H}\vert+1)$.
Assume a random ordering $\sigma$ on the vertices of $\mathcal{H}$.
Now take an edge $\{x,y\}\in E_\mathcal{H}$, and assume without loss of generality that $x$ precedes $y$ in $\sigma$.
Add the directed edges $u_{xy}\to x, x\to y, y\to w_{xy},$ and $w_{xy}\to s$.
Also draw a bidirected edge between $u_{xy}$ and all of the vertices $\{x,y,w_{xy},s\}$.
Graph $\mathcal{G}$ has therefore $4\vert E_\mathcal{H}\vert$ directed and $4\vert E_\mathcal{H}\vert$ bidirected edges ($4$ edges for each edge in $\mathcal{H}$).
Figure \ref{fig: np reduction} demonstrates the structure corresponding to the the edge $\{x,y\}$ constructed in $\mathcal{G}$.
Finally, the cost function $\mathbf{C}$ is defined as follows.
For $x\in V_\mathcal{H}$, $\mathbf{C}(x) = \omega(x)$.
for every other vertex $y\in U\cup W\cup\{s\}$, $\mathbf{C}(y)=z$, where
\[z = \vert V_\mathcal{H}\vert\cdot\max_{x\in V_\mathcal{H}}\omega(x) + 1.\]
First note that constructing the graph $\mathcal{G}$ and the cost function $\mathbf{C}$ given $\mathcal{H}$ and $\omega$ can be done in polynomial time, as it only needs a sweep over the vertices and the edges of $\mathcal{H}$, which can be performed in time $\mathcal{O}(V_\mathcal{H}+E_\mathcal{H})$.
To complete the proof of the theorem, we will show that a subset $A\subseteq V_\mathcal{H}$ is a weighted minimal vertex cover for $\mathcal{H}$ if and only if $A$ is a minimum-cost intervention to identify $Q[\{s\}]$ in $\mathcal{G}$.
We begin with the following claims.

\emph{Claim 1:} $\{V_\mathcal{H}\}\in\mathbf{ID}_\mathcal{G}(\{s\})$.
To see this, we simply provide the identification formula.
\[\begin{split}
    Q[\{s\}] &= P(s\vert do(V_\mathcal{H}, U, W))\\
    &= P(s\vert W, do(V_\mathcal{H},U))\:\:\:\:\:\text{(do calculus rule 2)}\\
    &= P(s\vert W, do(V_\mathcal{H})).\:\:\:\:\:\:\:\:\:\:\text{(do calculus rule 3)}
\end{split}\]
As seen in the expression above, $Q[\{s\}]$ can be identified by intervening on (fixing) only the variables $V_\mathcal{H}$.

\emph{Claim 2:} if $A$ is a minimum-cost intervention to identify $Q[\{s\}]$, then $A\subseteq V_\mathcal{H}$.
First note that from claim 1, we know that the cost of the min-cost intervention is at most $\mathbf{C}(V_\mathcal{H})\leq \vert V_\mathcal{H}\vert\cdot \max_{x\in V_\mathcal{H}}\mathbf{C}(x)\leq (z-1)$.
Since the cost of every variable in $V\setminus V_\mathcal{H}$ is equal to $z$, the min-cost intervention clearly does not include any such variable.

\emph{Claim 3:} if $A$ is a min-cost intervention to identify $Q[\{s\}]$ in $\mathcal{G}$, then $A$ is a vertex cover for $\mathcal{H}$.
Take an arbitrary edge $\{x,y\}\in E_\mathcal{H}$.
To prove this claim, it suffices to show that either $x\in A$ or $y\in A$.
The structure $\mathcal{G}_{[x,y,u_{xy},w_{xy},s]}$ (as depicted in Figure \ref{fig: np reduction}) is a hedge formed for $Q[\{s\}]$ in $\mathcal{G}$.
Since $\{A\}\in\mathbf{ID}_\mathcal{G}(S,V\setminus S)$, at least one of the variables $\{x,y,u_{xy},w_{xy}\}$ is included in $A$, as otherwise the aforementioned hedge precludes the identification of $Q[\{s\}]$.
However, from claim 2 we know that $A\subseteq V_\mathcal{H}$, and therefore at least one of $x,y$ is in $A$, which completes the proof of the claim.

\emph{claim 4:} if $A$ is a vertex cover for $\mathcal{H}$, then $\{A\}\in \mathbf{ID}_\mathcal{G}(\{s\},V\setminus\{s\})$ in $\mathcal{G}$.
Again an identification formula based on the do-calculus rules can be derived.
We first begin with the third rule of do calculus to derive $Q[\{s\}] = P(s\vert do(V_\mathcal{H}, U, W)) = P(s\vert do(A, U, W)).$
This is based on the fact that $s\perp\!\!\!\perp(V\setminus A)\vert A,U,W$ in the graph where incoming edges to $U,W,V\setminus A$ are deleted.
Now similar to the arguments of claim 1,
\begin{equation}\label{eq: claim 4}
\begin{split}
    Q[\{s\}] &= P(s\vert do(V_\mathcal{H}, U, W)) = P(s\vert do(A, U, W))
    \\&=P(s\vert W,do(A, U))\:\:\:\:\:\text{(do calculus rule 2)}
    \\&=P(s\vert W,do(A)).\:\:\:\:\:\:\:\:\:\:\text{(do calculus rule 3)}
\end{split}
\end{equation}
In the last equality, we used the fact that $A$ is a vertex cover for $\mathcal{H}$, and therefore in every structure like the one shown in Figure \ref{fig: np reduction}, at least one of the vertices $x$ or $y$ is included in $A$, i.e., non of the vertices in $U$ have a direct path to $s$ in the graph where the incoming edges to $A$ are deleted.
Equation \ref{eq: claim 4} proves claim 4, as intervention on $A$ suffices to identify $Q[\{s\}]$.

Now suppose $A$ is a minimum vertex cover for $\mathcal{H}$.
From claim 4, $\{A\}\in\mathbf{ID}_\mathcal{G}(\{s\},V\setminus\{s\})$.
We claim that this intervention is a minimum-cost intervention to identify $Q[\{s\}]$.
Suppose not.
Then there exists a min-cost intervention $\hat{A}$ and $\mathbf{C}(\hat{A})<\mathbf{C}(A)$.
From claim 3, $\hat{A}$ is a vertex cover for $\mathcal{H}$.
By definition of $\mathbf{C}(\cdot)$, $\sum_{a\in\hat{A}}\omega(a)=\mathbf{C}(\hat{A})<\mathbf{C}(A)=\sum_{a\in A}\omega(a)$, which contradicts the assumption that $A$ is the min vertex cover.

Conversely, suppose $A$ is the min-cost intervention to identify $Q[\{s\}]$ in $\mathcal{G}$.
From claim 3, $A$ is also a vertex cover for $\mathcal{H}$.
We claim that this vertex cover is a minimum vertex cover.
Suppose not.
Then there exists a minimum vertex cover $\hat{A}$ for $\mathcal{H}$ and $\sum_{a\in\hat{A}}\omega(a)<\sum_{a\in A}\omega(a)$.
From claim 4, $\{\hat{A}\}\in\mathbf{ID}_\mathcal{G}(\{s\},V\setminus\{s\})$.
By definition of $\mathbf{C}(\cdot)$, $\mathbf{C}(\hat{A})=\sum_{a\in\hat{A}}\omega(a)<\sum_{a\in A}\omega(a)=\mathbf{C}(A)$, which contradicts the assumption that $A$ is the min-cost intervention.
\end{proof}

\remconstantcost*
\begin{proof}
The proof is analogous to the proof of Theorem \ref{thm:NP-hard} with slight modifications as follows.
The graph $\mathcal{G}$ is constructed exactly in the same manner, but the cost function is forced to the constant $\mathbf{C}(\cdot)=1$.
With the exact same arguments of the proof of Theorem \ref{thm:NP-hard}, $A$ is a minimum vertex cover for $\mathcal{H}$ only if it is a min-cost intervention to identify $Q[\{s\}]$ in $\mathcal{G}$.
The other direction does not necessarily hold, as claim 2 of that proof does not hold anymore.
However, we show how a min-cost intervention solution $A$ can be turned into a minimum vertex cover for $\mathcal{H}$.
Suppose $A$ is a min-cost intervention to identify $Q[\{s\}]$ in $\mathcal{G}$.
Substitute any vertex $u_{xy}\in A\cap U$ or $w_{xy}\in A\cap W$ with one of the vertices $x,y$ arbitrarily, to form the set $\hat{A}$.
First note that $\mathbf{C}(\hat{A})\leq\mathbf{C}(A)$, since the cost of all variables are the same.
Also $\hat{A}\subseteq V_\mathcal{H}$.
We claim that $\hat{A}$ is a vertex cover for $\mathcal{H}$.
Take an arbitrary edge $\{x,y\}\in E_\mathcal{H}$.
It suffices to show that at least one of $x,y$ is included in $\hat{A}$, and follows from the fact that at least one of the variables $x,y,u_{xy}, w_{xy}$ must appear in $A$, since otherwise $\{x,y,u_{xy},w_{xy},s\}$ is a hedge formed for $Q[\{s\}]$ in $G$ after intervention on $A$, which contradicts the fact that $\{A\}\in \mathbf{ID}_\mathcal{G}(\{s\},V\setminus\{s\})$.

Note that $\hat{A}$ is also a minimum vertex cover for $\mathcal{H}$, since otherwise any vertex cover with smaller weight would also be an intervention with smaller cost than $\mathcal{A}$ to identify $Q[\{s\}]$, which is a contradiction.
\end{proof}

\thmreductiontwo*
\begin{proof}
    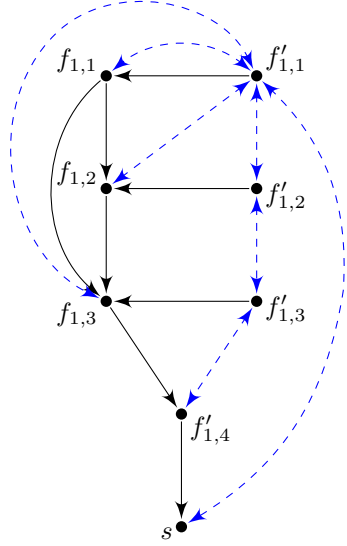
\begin{figure}[t]
        \centering
    	\tikzstyle{block} = [circle, inner sep=1.5pt, fill=black]
    	\tikzstyle{input} = [coordinate]
    	\tikzstyle{output} = [coordinate]
        \begin{tikzpicture}
            \tikzset{edge/.style = {->,> = latex'},line width=1.4pt}
            \node[block](f11) at (0,6) {};
            \node[block](f12) at (0,4.5) {};
            \node[block](f13) at (0,3) {};
            \node[block](f_11) at (2,6) {};
            \node[block](f_12) at (2,4.5) {};
            \node[block](f_13) at (2,3) {};
            \node[block](f_14) at (1,1.5) {};
            \node[block](s) at (1,0){};
            
            \node[] (dm1) [above left=0.2cm and 0.6cm of f11]{};
            \node[] (dm2) [right= 0.9cm of f_13]{};
            
            \node[] ()[below left=-0.2cm and -0.1cm of s]{$s$};
            \node[] ()[below right=-0.2cm and -0.1cm of f_14]{$f'_{1,4}$};
            \node[] ()[below right=-0.3cm and -0.1cm of f_13]{$f'_{1,3}$};
            \node[] ()[below right=-0.3cm and -0.1cm of f_12]{$f'_{1,2}$};
            \node[] ()[above right=-0.2cm and -0.1cm of f_11]{$f'_{1,1}$};
            \node[] ()[above left=-0.2cm and -0.1cm of f11]{$f_{1,1}$};
            \node[] ()[above left=-0.2cm and -0.1cm of f12]{$f_{1,2}$};
            \node[] ()[below left=-0.2cm and -0.1cm of f13]{$f_{1,3}$};
            
            \draw[edge] (f_11) to (f11);
            \draw[edge] (f_12) to (f12);
            \draw[edge] (f_13) to (f13);
            \draw[edge] (f13) to (f_14);
            \draw[edge] (f_14) to (s);
            \draw[edge] (f11) to (f12);
            \draw[edge, bend right=50] (f11) to (f13);
            \draw[edge] (f12) to (f13);
            
            \draw[edge, color=blue, dashed, style={<->}] (f_11) to  (f_12);
            \draw[edge, color=blue, dashed, style={<->}] (f_12) to  (f_13);
            \draw[edge, color=blue, dashed, style={<->}] (f_13) to  (f_14);
            \draw[edge, color=blue, dashed, style={<->}, bend left=25] (f_11) to (dm2) to (s);
            \draw[edge, color=blue, dashed, style={<->}, bend right=40] (f_11) to  (f11);
            \draw[edge, color=blue, dashed, style={<->}] (f_11) to  (f12);
            \draw[edge, color=blue, dashed, style={<->}, bend right=55] (f_11) to (dm1) to (f13);
        \end{tikzpicture}
        \caption{Reduction from the minimum weighted hitting set problem to the minimum-cost intervention problem.
        Each set $F_i=\{f_{i,1},\dots,f_{i,m}\}$ is represented by a hedge structure in the semi-Markovian graph $\mathcal{G}$. 
        }
        \label{fig:mwhsreductionproof}
    \end{figure}
    Suppose we have a universe $V=\{v_1,v_2,\dots,v_n\}$, a set of subsets of $V$ denoted as $F_1,...,F_k$, and a weight function $\omega:V\to \mathbbm{R}^{\geq0}$.
    We consider the problem of finding the subset of $V$ with the minimum aggregate weight such as $A$ such that for any $1\leq i\leq k$, $A\cap F_i\neq\emptyset$.
    We propose a polynomial time reduction from this problem to an instance of the minimum-cost intervention problem.
    To this end, we construct a semi-Markovian graph $\mathcal{G}$ as follows.

    For each member of $V$ such as $v_i$, we add a vertex representing $v_i$ in $\mathcal{G}$.
    We fix an arbitrary ordering between the members of $V$, which without loss of generality we assume is $v_1\prec v_2\prec \dots\prec v_n$.
    We draw directed edges between each pair of nodes in $\{v_1,\dots,v_n\}$ to build a complete graph over the nodes representing $V$, with respect to the ordering.
    We add an auxiliary node $s$, for which we will consider the hedges formed in $\mathcal{G}$.
    Then for each subset $F_i=\{f_{i,1}\prec ...\prec f_{i,m}\}\subseteq\{v_1\prec\dots\prec v_n\}$,
    \begin{itemize}
        \item we add $(m+1)$ new nodes to $\mathcal{G}$, denoted by $f_{i,j}'$ for $1\leq j\leq (m+1)$.
        \item We draw a directed edge from $f_{i,j}'$ to $f_{i,j}$ for $1\leq j\leq m$.
        \item We draw a directed edge from $f_{i,m}$ to $f_{i,m+1}'$, and a directed edge from $f_{i,m+1}'$ to $s$.
        \item For $1\leq j\leq m$, we draw a bidirected edge between $f_{i,j}'$ and $f_{i,j+1}'$.
        \item Finally, we draw $m$ bidirected edges from $f_{i,1}'$ to all nodes in $F_i$, as well as one bidirected edge to $s$.
    \end{itemize}
    We assign $\mathbf{C}(v)=\omega(v)$ for any $v\in V$, and $\mathbf{C}(w)=\infty$ for the rest of (auxiliary) vertices.
    An example for $F_1=\{f_{1,1}\prec f_{1,2}\prec f_{1,3}\}$ is visualized in Figure \ref{fig:mwhsreductionproof}.
    First note that building the graph $\mathcal{G}$ requires only traversing $V$ and each set $F_i$ once, and performing at most $\mathcal{O}(\vert F_i\vert)$ operations for each of the subsets.
    As a result, the reduction is polynomial time (the size of the resulting graph is $1+\vert V\vert+\sum_{i=1}^k (\vert F_i\vert+1)$.)
    It suffices now to prove that the minimum-cost intervention to identify $Q[\{s\}]$ in $\mathcal{G}$ is exactly the minimum hitting set for $\{F_1,\dots,F_k\}$.
    This is proved through the following claims.

    \textbf{Claim 1.} For each $F_i=\{f_{i,1},\dots,f_{i,m}\}$, the set $X_i=F_i\cup\{f'_{i,j}\vert 1\leq j\leq (m+1)\}\cup\{s\}$ forms a hedge for $Q[\{s\}]$.
    The reason to this is straightforward.
    Since a complete graph is formed on the vertices corresponding to $V$, every vertex in $F_i$ is a parent of $f_{i,m}$, which itself has a directed path to $s$ through $f_{i,m+1}'$.
    Every $f_{i,j}$ for $1\leq j\leq m$ is also a parent of a vertex in $F_i$.
    As a result, every vertex in $X_i$ is an ancestor of $s$ in $\mathcal{G}_{[X_i]}$.
    On the other hand, every vertex in $F_i$ along with $s$ is connected to $f_{i,1}'$ by a bidirected edge, and by construction, every vertex $f_{i,j}'$ has a bidirected path to $f_{i,1}'$, which means $\mathcal{G}_{[X_i]}$ is a c-component.

    \textbf{Claim 2.} Any finite-cost intervention set $A$ that suffices to identify $Q[S]$ is a valid hitting set for $\{F_1,\dots, F_k\}$.
    That is, if we solve the minimum-cost intervention problem for $Q[\{s\}]$, the solution will be a valid hitting set for all the sets $F_i$ for $1\leq i\leq k$.
    To see this, note that $A$ has non-empty intersection with every hedge formed for $Q[\{s\}]$ in $\mathcal{G}$ such as $X_i$.
    However, every vertex in $X_i$ except those in $F_i$ have infinite cost.
    Therefore, $A\cap F_i\neq\emptyset$ for every $1\leq i\leq m$.
    
    \textbf{Claim 3.}
    If $A$ is a valid hitting set for $F_i$s, there is no hedge formed for $Q[\{s\}]$ in $\mathcal{G}$ after intervention on $A$.
    That is, any valid hitting set for $F_i$s is an intervention set which suffices to identify $Q[\{s\}]$ in $\mathcal{G}$.
    To prove this, it suffices to show that if an arbitrary hedge $X$ is formed for $Q[\{s\}]$, then there exists an index $1\leq i\leq k$ such that no variable in $F_i$ is intervened upon.
    Since by construction the only parents of $s$ are the vertices $f'_{j,m+1}$ corresponding to some $F_j$, $X$ includes one of these vertices (a hedge cannot be formed without including a parent of $s$.)
    Suppose without loss of generality that $F_1=\{f_{1,1},\dots,f_{1,m}\}$ and that $f_{1,m+1}'\in X$.
    Since $f_{1,m+1}'$ has a bidirected edge only to $f_{1,m}'$, $f_{1,m}'\in X$.
    Applying the same argument recursively, we have that $\{f_{1,1}',\dots,f_{1,m}'\}\subseteq X$.
    However, note that each vertex $f_{1,j}'$ has only one child by construction of $\mathcal{G}$, which is $f_{1,j}$.
    As a result, $f_{1,j}$ is not intervened upon for $1\leq j\leq m$ (as otherwise the hedge $X$ would be resolved as at least one of $f_{1,j}'$s would not be an ancestor of $s$ in $\mathcal{G}_{[X]}$.)
    This completes the proof of Claim 3, as non of vertices in $F_1$ are intervened upon, that is, $A\cap F_1=\emptyset$.

    Combining the three claims above, solving the minimum hitting set for $F_i$s is equivalent to solving the minimum cost intervention set for $Q[\{s\}]$.
\end{proof}
\lemhitsetform*
\begin{proof}
First note that if $A$ is an intervention set that makes $Q[S]$ identifiable, it hits all the hedges formed for $Q[S]$ in $\mathcal{G}$, as otherwise from hedge criterion (Theorem \ref{thm: hedge criterion}) $Q[S]$ would not be identifiable.
Conversely, if $A$ hits all the hedges formed for $Q[S]$ in $\mathcal{G}$, intervening on $A$ makes $Q[S]$ identifiable.
As a result, a set $A$ is an intervention to identify $Q[S]$ in $\mathcal{G}$ if and only if it is a hitting set for $\{F_1\setminus S,...,F_m\setminus S\}$.
Since the set of interventions and the hitting sets coincide, a minimum intervention set is a minimum hitting set and vice-versa.
\end{proof}

\lemdirP*
\begin{proof}
First, from Lemma \ref{lem: inside S} we know that $S\cap A=\emptyset$.
Now define $B=S\cup(\PaC{S}\setminus A)$.
If $B\setminus S\neq \emptyset$, then by definition, $B$ is a hedge formed for $Q[S]$ in $\mathcal{G}_{[V\setminus A]}$, and from Theorem \ref{thm: hedge criterion}, $Q[S]$ is not identifiable in $\mathcal{G}_{[V\setminus A]}$, which is a contradiction.
As a result, $B\setminus S=\emptyset$, or equivalently, $\PaC{S}\setminus A = \emptyset$.
\end{proof}

\lemhedgehull*
\begin{proof}
By definition of hedge hull, all the hedges formed for $Q[S]$ are a subset of $Hhull(S)$.
From Lemma \ref{lem: hitsetform}, $A^*\subseteq Hhull(S)$.
The result now follows from Lemma \ref{lem: inside S}, which states that $A^*\cap S = \emptyset$.
\end{proof}

\lemhedgehullcorrect*
\begin{proof}
First note that $F_2$ is always a subset of $F$ throughout the algorithm.
Therefore, every time that $F_2\neq F$, at least one vertex is excluded from $F$ to form $F_2$.
Hence, the while loop is performed at most $\vert V\vert$ times, and the algorithm terminates.
Inside every loop, two depth-first searches are executed, one to find the connected component of $S$ in the edge-induced subgraph of $\mathcal{G}_{[F]}$ over its bidirected edges, and the other to find the ancestors of $S$ in $\mathcal{G}_{[F]}$.
DFS is quadratic-time in the worst case, i.e., each iteration runs in time $2\vert F\vert^2\leq2\vert V\vert^2$ in the worst case.
Therefore, the algorithm ends in time $\mathcal{O}(\vert V\vert^3)$.

Let $\Tilde{F}$ be the output of Algorithm \ref{alg: hedge hull}.
Since in the last iteration $\Tilde{F}=F_1=F_2$, $\Tilde{F}$ is the set of ancestors of $S$ in $\mathcal{G}_{[\Tilde{F}]}$, and also $\mathcal{G}_{[\Tilde{F}]}$ is a c-component.
By definition, $\Tilde{F}$ is a hedge formed for $Q[S]$ in $\mathcal{G}$, and therefore $\Tilde{F}\subseteq Hhull(S,\mathcal{G})$.
Now suppose $F'$ is a hedge formed for $Q[S]$ in $\mathcal{G}$.
It suffices to show that $F'\subseteq \Tilde{F}$.
At the beginning of the algorithm, $F=V$, that is, $F'$ is included in $F$.
At each iteration, since every vertex in $F'$ has a bidirected path to $S$ through only the vertices in $F'$, it also has a bidirected path to $S$ in every subgraph of $\mathcal{G}$ which includes $F'$.
As a result, when constructing the connected component of $S$ in line 3 of Algorithm \ref{alg: hedge hull}, $F'\subseteq F_1$.
Further, by definition of hedge, every vertex in $F'$ has a directed path to $S$ that goes through only vertices of $F'$.
By the same argument, every vertex in $F'$ is included in $F_2$ in line 4.
Therefore, $F'\subseteq \Tilde{F}$.
\end{proof}

\lemhhullpac*
\begin{proof}
First note that the set of hedges formed for $Q[S]$ in $\mathcal{G}$ can be partitioned into hedges that intersect with $\PaC{S}$, and the hedges that do not intersect with $\PaC{S}$, denoted by $\mathbf{F_1}$ and $\mathbf{F_2}$, respectively.
The set of hedges formed for $Q[S]$ in $\mathcal{G}_{[H']}$ is then $\mathbf{F_2}$.
Using Lemma \ref{lem: hitsetform}, the lemma is equivalent to the claim that $A^*$ is a minimum hitting set for the hedges $\mathbf{F_1}\cup\mathbf{F_2}$ if and only if $A^*\setminus\PaC{S}$ is a minimum hitting set for hedges $\mathbf{F_2}$.

Suppose $A^*$ is a min-cost intervention to identify $Q[S]$ in $\mathcal{G}$.
From Lemma \ref{lem: hitsetform}, $A^*$ hits all the hedges formed for $Q[S]$ in $\mathcal{G}_{H'}$, i.e., $\mathbf{F_2}$.
However, since none of these hedges intersect with $\PaC{S}$, $A^*\setminus\PaC{S}$ hits all of these hedges.
We claim that $A^*\setminus\PaC{S}$ is the minimum hitting set for the hedges $\mathbf{F_2}$.
Suppose there exists another set $\Tilde{A}$ such that $\Tilde{A}$ hits all the hedges $\mathbf{F_2}$, and $\mathbf{C}(\Tilde{A})<\mathbf{C}(A^*\setminus\PaC{S})$.
Since all the hedges $\mathbf{F_1}$ intersect with $\PaC{S}$, $\Tilde{A}\cup\PaC{S}$ hits all the hedges formed for $Q[S]$ in $\mathcal{G}$, and
\[\begin{split}
    \mathbf{C}(\Tilde{A}\cup\PaC{S})&\leq\mathbf{C}(\Tilde{A})+\mathbf{C}(\PaC{S})
    \\&<\mathbf{C}(A^*\setminus\PaC{S})+\mathbf{C}(\PaC{S})\\&=\mathbf{C}(A^*),
\end{split}
\]
which contradicts the fact that $A^*$ is the minimum-cost intervention to identify $Q[S]$ in $\mathcal{G}$.

Conversely, let $A^*\setminus\PaC{S}$ be a minimum hitting set for hedges $\mathbf{F_2}$.
If $A^*$ is not the min-cost intervention to identify $Q[S]$ in $\mathcal{G}$, then there exists $\Tilde{A}$ such that $\mathbf{C}(\Tilde{A})<\mathbf{C}(A^*)$ and $\Tilde{A}$ hits the hedges $\mathbf{F_1}\cup\mathbf{F_2}$.
From Lemma \ref{lem: dirP}, $\PaC{S}\subseteq\Tilde{A}$.
Since hedges $\mathbf{F_2}$ do not intersect with $\PaC{S}$, $\Tilde{A}\setminus \PaC{S}$ hits all the hedges $\mathbf{F_2}$, and
\[\begin{split}
    \mathbf{C}(\Tilde{A}\setminus\PaC{S})&=\mathbf{C}(\Tilde{A})-\mathbf{C}(\PaC{S})
    \\&<\mathbf{C}(A^*)-\mathbf{C}(\PaC{S})\\&=\mathbf{C}(A^*\setminus\PaC{S}),
\end{split}
\]
which contradicts the fact that $A^*\setminus\PaC{S}$ is the minimum-cost intervention to identify $Q[S]$ in $\mathcal{G}_{[H']}$.
\end{proof}

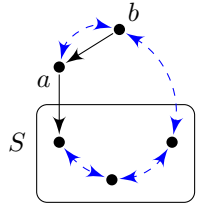
\begin{figure}[t]
    \centering
	\tikzstyle{block} = [circle, inner sep=1.5pt, fill=black]
	\tikzstyle{input} = [coordinate]
	\tikzstyle{output} = [coordinate]
    \begin{tikzpicture}
        \tikzset{edge/.style = {->,> = latex'},line width=1.4pt}
        \node[block](s1) at (0.7,.7){};
        \node[block](s2) at (2.2,.7){};
        \node[block](s3) at (1.4,0.2){};
        \node[block](a) at (0.7,1.7){};
        \node[block](b) at (1.5,2.2){};
        \draw[rounded corners] ([xshift=-.3cm, yshift=-0.8cm]s1) rectangle (2.5cm,1.2cm);
        
        \node[] ()[below left=-0.1cm and -0.1cm of a]{$a$};
        \node[] ()[above right=-0.1cm and -0.1cm of b]{$b$};
        \node[] ()[left=000.2cm of s1]{$S$};
        
        \draw[edge] (a) to (s1);
        \draw[edge] (b) to (a);
        
        \draw[edge, color=blue, dashed, style={<->}, bend left=45] (a) to (b);
        \draw[edge, color=blue, dashed, style={<->}, bend left=35] (b) to (s2);
        \draw[edge, color=blue, dashed, style={<->}, bend right=15] (s1) to (s3);
        \draw[edge, color=blue, dashed, style={<->}, bend left=25] (s2) to (s3);
    \end{tikzpicture}
    \caption{Hedges of size 2 must be in the form depicted above. Exactly one vertex $a$ is a member of $\Pa{S}\setminus \BiD{S}$, and the other vertex $b$ is a member of $\BiD{S}\setminus\Pa{S}$.
    }
    \label{fig: hedge size2}
\end{figure}

\lemalgultimcorrect*
\begin{proof}
First let $\mathcal{G}_{[S]}$ be a c-component.
Note that every time that the set $A$ is constructed in line 14 and intervening on $A\cup\PaC{S}$ does not suffice to identify $Q[S]$ in $\mathcal{G}$, the algorithm continues on the subgraph of $\mathcal{G}$ over $Hhull(S,\mathcal{G}_{V\setminus(A\cup\PaC{S})})$.
As a result, the newly discovered hedges do not intersect with $A$, i.e., Algorithm \ref{alg: ultimate} never discovers redundant hedges.
Since every hedge is a subset of the graph and the number of such subsets is finite, the algorithm halts within finite time.
At each iteration (while loop of lines 7-13), a new hedge is added to the set of hedges formed for $Q[S]$ and adds it to the set $\mathbf{F}$.
From the minimum hitting set formulation, it is clear that any min-cost intervention for identifying $Q[S]$ in $\mathcal{G}$ must intersect with all the hedges in $\mathbf{F}$.
As a result, the output of Algorithm \ref{alg: ultimate} ($A$) is always a subset of a min-cost intervention.
Further, by construction $A\in\mathbf{ID_1}(S)$.
As a result, $A$ is a min-cost intervention for identifying $Q[S]$ in $\mathcal{G}$.

Now suppose $\mathcal{G}_{[S]}$ is not a c-component.
$\mathcal{G}_{[S]}$ can be uniquely partitioned into its maximal c-components $\mathcal{G}_{[S_1]},...,.\mathcal{G}_{[S_k]}$.
The arguments above hold for any of the maximal c-components $\mathcal{G}_{[S_i]}$.
The result follows from the fact that $Q[S]$ is identifiable in $\mathcal{G}$ if and only if all of its maximal c-components are identifiable in $\mathcal{G}$.
\end{proof}

\prpanytime*
\begin{proof}
    Let $A$ be the minimum hitting set computed in line 14 of Algorithm \ref{alg: ultimate} in $i$-th iteration.
    Let $\mathbf{F}$ and $\mathbf{F'}$ denote the set of all hedges formed for $Q[S]$ in $\mathcal{G}_{[V\setminus\PaC{S}]}$, and the set of hedges discovered up to $i$-th iteration.
    Clearly, $A$ \emph{hits} all the hedges in $\mathbf{F'}$.
    Moreover, by definition of hedge hull, any hedge formed for $Q[S]$ which is hit neither by $A$ nor by $\PaC{S}$, is a subset of $H$.
    As a result, any hedge formed for $Q[S]$ in $\mathcal{G}$ is hit by at least one of the variables in $A$, $\PaC{S}$, or $H$.
    Therefore, $A\cup\PaC{S}\cup H\in\mathbf{ID_1}(S)$.

    For the second part of the claim, note that the first inequality is trivial since $A^*$ is the optimal set in $\mathbf{ID_1}(S)$.
    As for the second inequality, recall that $A^*$ is the minimum hitting set for $\mathbf{F}$ along with $\PaC{S}$, whereas $A$ is the minimum hitting set for $\mathbf{F'}$.
    Since $\mathbf{F'}\subseteq\mathbf{F}$, any hitting set for $\mathbf{F'}$ is also a hitting set for $\mathbf{F'}$.
    This is to say, $\mathbf{A^*}\setminus\PaC{S}$ is a hitting set for $\mathbf{F'}$, whereas $A$ is the minimum one.
    As a result, $\mathbf{C}(A)\leq\mathbf{C(A^*\setminus\PaC{S})}$.
    From linearity of the cost function,
    \[
        \begin{split}
            \mathbf{C}(A\cup\PaC{S}\cup H)&=\mathbf{C}(A)+\mathbf{C}(\PaC{S})+\mathbf{C}(H)\\&\leq \mathbf{C}(A^*\setminus\PaC{S}))+\mathbf{C}(\PaC{S})+\mathbf{C}(H)=\mathbf{C}(A^*)+\mathbf{C}(H)
        \end{split}
    \]
\end{proof}

\lemalgheurOnecorrect*
\begin{proof}
\emph{Correctness.} Let the output of the algorithm be $A$.
We will utilize the minimum hitting set formulation to show the correctness of the algorithm.
Let $F$ be a hedge formed for $Q[S]$ in $\mathcal{G}$.
It suffices to show that $F\cap A\neq\emptyset$.
If $F\cap \PaC{S}\neq\emptyset$, then the claim holds since $\PaC{S}\subseteq A$.
Otherwise, $F$ is a hedge formed for $Q[S]$ in $\mathcal{G}_{[V\setminus \PaC{S}]}$, i.e., $F\subseteq H$, where $H$ is given by Equation \eqref{eq:h-def}.
Now let $a$ be an arbitrary vertex in $ (F\setminus S)\cap\Pa{S}$.
Such a vertex exists by definition of hedge.
Further, $\mathcal{G}_{[F]}$ is a c-component, i.e., there exists a path from $a$ to $S$ through bidirected edges in $F$.
As a result, in the undirected graph $\mathcal{H}$ built in heuristic Algorithm 1, there exists a path from $x$ to $y$ that passes only through vertices in $F$.
Any solution to minimum vertex cut for $x-y$ must include at least one vertex of $F$.
Therefore, $F\cap A\neq\emptyset$.

\emph{Runtime.} Heuristic Algorithm 1 begins with constructing the set $\PaC{S}$, and the set $H$ given by Equation \eqref{eq:h-def}, which are performed in time $\mathcal{O}(\vert V\vert)$, and $\mathcal{O}(\vert V\vert^3)$ in the worst case.
Constructing the graph $\mathcal{H}$ requires iterating over the bidirected edges of $\mathcal{G}_{[H]}$, which can be done in time $\mathcal{O}(\vert H\vert^2)$ in the worst case.
The reduction from minimum vertex cut to minimum edge cut discussed in Appendix \ref{apdx: heuristic} is linear-time.
The final step of the algorithm is to solve a minimum edge cut, which can be done in time $\mathcal{O}(\vert H\vert^3)$ using the push-relabel algorithm to solve the equivalent maximum flow problem \citep{goldberg1988new}.
Noting that $\vert H\vert\leq\vert V\vert$, the runtime of the algorithm is $\mathcal{O}(\vert V\vert^3)$.

\end{proof}

\lemalgheurTwocorrect*
\begin{proof}
\emph{Correctness.} Let the output of the algorithm be $A$.
We will utilize the hitting set formulation to show the correctness of the algorithm.
Let $F$ be a hedge formed for $Q[S]$ in $\mathcal{G}$.
It suffices to show that $F\cap A\neq\emptyset$.
If $F\cap \PaC{S}\neq\emptyset$, then the claim holds since $\PaC{S}\subseteq A$.
Otherwise, $F$ is a hedge formed for $Q[S]$ in $\mathcal{G}_{[V\setminus \PaC{S}]}$, i.e., $F\subseteq H$, where $H$ is given by Equation \eqref{eq:h-def}.
Now let $b$ be an arbitrary vertex in $ (F\setminus S)\cap\BiD{S}$.
Such a vertex exists by definition of hedge.
Further, $F$ are the ancestors of $S$ in $\mathcal{G}_{[F]}$, i.e., there exists a directed path from $b$ to $S$ through directed edges in $\mathcal{G}_{[F]}$.
As a result, in the directed graph $\mathcal{J}$ built in heuristic Algorithm 2, there exists a path from $x$ to $y$ that passes only through vertices in $F$.
Any solution to minimum vertex cut for $x-y$ must include at least one vertex of $F$.
Therefore, $F\cap A\neq\emptyset$.

\emph{Runtime.} Heuristic Algorithm 2 begins with constructing the set $\PaC{S}$, and the set $H$ given by Equation \eqref{eq:h-def}, which are performed in time $\mathcal{O}(\vert V\vert)$, and $\mathcal{O}(\vert V\vert^3)$ in the worst case.
Constructing the graph $\mathcal{J}$ requires iterating over the directed edges of $\mathcal{G}_{[H]}$, which can be done in time $\mathcal{O}(\vert H\vert^2)$ in the worst case.
The reduction from minimum vertex cut to minimum edge cut discussed in Appendix \ref{apdx: heuristic} is linear-time.
The final step of the algorithm is to solve a minimum edge cut, which can be done in time $\mathcal{O}(\vert H\vert^3)$ using the push-relabel algorithm to solve the equivalent maximum flow problem \citep{goldberg1988new}.
Noting that $\vert H\vert\leq\vert V\vert$, the runtime of the algorithm is $\mathcal{O}(\vert V\vert^3)$.
\end{proof}

\prpheurs*
\begin{proof}
    Since the costs of intervention are assumed to be equal for every vertex, both sides of the inequalities can be normalized with the constant cost.
    Therefore, without loss of generality, suppose that the cost of intervention on each vertex is $1$.
    That is, the cost of intervention on a set $X$ is $\vert X\vert$.
    Let $N$ be the set of vertices of $\mathcal{G}$ that precede $s$ in the causal order.
    Each of these vertices s in $\BiD{S}$ with probability $q$, and is in $\Pa{S}$ with probability $p$.
    Since these two events are independent, any such vertex appears in $\PaC{S}$ with probability $pq$.
    Let $\mathbbm{1}_{\PaC{S}}(v)$ be the indicator function corresponding to the membership of a vertex $v\in N$ in $\PaC{S}$.
    Then,
    \[
    \mathbbm{E}[\vert\PaC{S}\vert]=\mathbbm{E}[\sum_{v\in N}\mathbbm{1}_{\PaC{S}}(v)]=\sum_{v\in N}Pr(v\in\PaC{S}) = \sum_{v\in V}pq=\vert N\vert pq.
    \]
    From Lemma \ref{lem: dirP}, $\PaC{S}\subseteq A^*$, where $A^*$ is the optimal solution to minimum-cost intervention (and $c^*=\mathbf{C}(A^*))$.
    Therefore, $\vert\PaC{S}\vert\leq\vert A^*\vert$, and consequently, 
    \begin{equation}\label{eq:prp281}
        \mathbbm{E}[c^*]=\mathbbm{E}[\mathbf{C}(A^*)]=\mathbbm{E}[\vert A^*\vert]\geq \mathbbm{E}[\vert\PaC{S}\vert]=\vert N\vert pq.
    \end{equation}
    Now consider the minimum-weight vertex cut problem that Heuristic algorithm 1 solves.
    The vertex $x$ is only connected to $\Pa{S}\cap H$, and therefore $\Pa{S}\cap H$ is a trivial cut.
    As a result, 
    \begin{equation}\label{eq:prp282}
        \mathbbm{E}[c_1]\leq \mathbbm{E}[\vert\Pa{S}\cap H\vert]\leq\mathbbm{E}[\vert\Pa{S}\vert]=\vert N\vert p.
    \end{equation}
    By a similar argument, we can write
    \begin{equation}\label{eq:prp283}
        \mathbbm{E}[c_2]\leq \mathbbm{E}[\vert\BiD{S}\cap H\vert]\leq\mathbbm{E}[\vert\BiD{S}\cap N\vert]=\vert N\vert q.
    \end{equation}
    The result follows from Equations \eqref{eq:prp281}, \eqref{eq:prp282} and \eqref{eq:prp283}.
\end{proof}

\corheurmin*
\begin{proof}
    Since both heuristic algorithms run in time $\mathcal{O}(\vert V\vert^3)$, running both of them also requires time $\mathcal{O}(\vert V\vert^3)$.
    Also, since we are choosing the better solution among the two, if the costs of the solutions returned by Heuristic Alg. 1 and Heuristic Alg. 2 are denoted by $c_1$ and $c_2$, respectively, then $c=\min\{c_1,c_2\}$.
    Therefore, $\mathbbm{E}[c]\leq\min\{\mathbbm{E}[c_1],\mathbbm{E}[c_2]\}$.
    Plugging in the inequalities from Proposition \ref{prp:heurs},
    \[
    \mathbbm{E}[c]\leq\min\{\mathbbm{E}[c_2],\mathbbm{E}[c_1]\}\leq\min\{p^{-1}\mathbbm{E}[c^*],q^{-1}\mathbbm{E}[c^*]\}=\min\{p^{-1},q^{-1}\}\mathbbm{E}[c^*],
    \]
    which proves the second inequality.
    The first inequality is trivial since $c^*$ is the optimal cost.
\end{proof}

\thmsingletongeneral*
\begin{proof}
First note that if there exists $1\leq i\leq m$ such that $A_i\cap S\neq\emptyset$, then $\mathbf{C}(\mathbf{A})=\infty$.
In this case, $\Tilde{A}=V\setminus S$ satisfies the desired property.
Otherwise, we can assume that $S\cap(\cup_{i=1}^mA_i)=\emptyset$.
We claim that $\Tilde{A}=A_\cup=\cup_{i=1}^mA_i$ is the desired intervention set.
To prove this claim, first note that
\[\begin{split}
    \mathbf{C}(\{\Tilde{A}\}) &=
    \mathbf{C}(A_\cup) =
    \sum_{a\in A_\cup}\mathbf{C}(a)
    \leq\sum_{i=1}^m\sum_{a\in A_i}\mathbf{C}(a)\\&
    =\sum_{i=1}^m\mathbf{C}(A_i)
    =
    \mathbf{C}(\mathbf{A}).
\end{split}
\]
Therefore, it suffices to show that $\{\Tilde{A}\}\in\mathbf{ID}_\mathcal{G}(S,V\setminus S)$.
Let $S_1, ..., S_k$ denote the maximal c-components of $\mathcal{G}_{[S]}$.
From lemma 2 of \citep{tian2002testable} (restated here as Theorem \ref{thm: tian}), $Q[S]$ is identifiable in $\mathcal{G}_{V\setminus\Tilde{A}}$, if and only if $Q[S_1],...,Q[S_k]$ are identifiable in $\mathcal{G}_{V\setminus\Tilde{A}}$.
Therefore, it suffices to show that $\{\Tilde{A}\}\in\mathbf{ID}_\mathcal{G}(S_i,V\setminus S_i)$, for all $1\leq i\leq k$.
This result follows from Lemma \ref{lem:intersect}, since $\mathcal{G}_{[S_i]}$ is a c-component, $\{A_1,...,A_m\}\in\mathbf{ID}_\mathcal{G}(S_i,V\setminus S_i)$ (from Theorem \ref{thm: tian}), and $A_\cup\cap S_i=\emptyset$.
\end{proof}

\lempartition*
\begin{proof}
    By definition of $S^{(j)}$s, for each $1\leq j\leq t$, we have $A_j\in\mathbf{ID_1}(\underline{S}^{(j)})$.
    Suppose that the claim does not hold.
    That is, there exists an index $j$ such that $A_j \notin\argmin_{A\in \mathbf{ID_1}(\underline{S}^{(j)})}\mathbf{C}(A)$.
    From a new set family (collection of subsets) $\mathbf{A}$ by replacing the set $A_j$ in $\mathbf{A}_{S,V\setminus S}^*$ with $A_j'$, where
    \[A_j' \in\argmin_{A\in \mathbf{ID_1}(\underline{S}^{(j)})}\mathbf{C}(A).\]
    Clearly, the c-components in $S^{(j)}$ are identified from intervention on $A_j'$, and the rest of the c-components (in all other partitions) are identified from the sets in $\mathbf{A}\setminus\{A_j\}$.
    As a result,
    \begin{equation}\label{eq:lem311}
        \mathbf{A}\in\mathbf{ID}_{\mathcal{G}}(S,V\setminus S).
    \end{equation}
    Since $A_j,A_j'\in\mathbf{ID_1}(S)$ and $A_j'$ has the optimal cost, $\mathbf{C}(A_j')<\mathbf{C}(A_j)$ (equality cannot happen since $A_j$ is assumed not to be optimal.)
    Since the rest of the sets in $\mathbf{A}$ and $\mathbf{A}^*_{S,V\setminus S}$ are the same, and since the costs are additive,
    \begin{equation}\label{eq:lem312}
        \mathbf{C}(\mathbf{A})=\mathbf{C}(\mathbf{A}\setminus\{A_j'\}) + \mathbf{C}(A_j') <\mathbf{C}(\mathbf{A}^*_{S,V\setminus S}\setminus \{A_j\}) + \mathbf{C}(A_j) = \mathbf{C}(\mathbf{A}^*_{S,V\setminus S}).
    \end{equation}
    Equations \eqref{eq:lem311} and \eqref{eq:lem312} contradict the optimality of $\mathbf{A}^*_{S,V\setminus S}$, which completes the proof.
\end{proof}

\thmalggeneral*
\begin{proof}
Let $\mathbf{A}$ be the output of Algorithm \ref{alg: general heuristic}.
We first claim that $\mathbf{A}\in\mathbf{ID}_\mathcal{G}(S,V\setminus S)$.
For any maximal c-component of $\mathcal{G}_{[S]}$ such as $\mathcal{G}_{[S_i]}$, there exists at least one set $A_j\in\mathbf{A}$ such that $A_j\in\mathbf{ID_1}(S)$.
Therefore, $\mathbf{A}\in\mathbf{ID}_\mathcal{G}(S_i,V\setminus S_i)$.
Since $Q[S_i]$ is identifiable from interventions on $\mathbf{A}$ for any c-component of $\mathcal{G}_{[S]}$, by the result form \citep{tian2002testable}, $Q[S]$ is also identifiable, i.e., $\mathbf{A}\in\mathbf{ID}_\mathcal{G}(S,V\setminus S)$.

Now suppose $\mathbf{A}^*_S=\{A^*_1,...,A^*_t\}$ is a min-cost intervention collection to identify $Q[S]$ in $\mathcal{G}$.
It suffices to show that $\mathbf{C}(\mathbf{A})\leq\mathbf{C}(\mathbf{A}^*_S)$.
Consider an arbitrary partitioning of the c-components of $S$ to $t$ parts such as $S^{(1)},...S^{(t)}$, such that $A^*_i\in\mathbf{ID_1}(\underline{S}^{(i)})$ for $1\leq i\leq t$.
Note that such a partitioning is possible due to Observation \ref{obs:1}.
Algorithm \ref{alg: general heuristic} considers this partitioning in one of its iterations, and due to optimality of Algorithm \ref{alg: ultimate} (see Lemma \ref{lem: ultimate alg}), it constructs an intervention collection $\Tilde{\mathbf{A}}=\{\Tilde{A}_1,...,\Tilde{A}_t\}$ where $\Tilde{A}_i$ is the optimal intervention set in $\mathbf{ID_1}(\underline{S}^{(i)})$ for $1\leq i\leq t$.
As a result,
\[\mathbf{C}(\Tilde{\mathbf{A}})=\sum_{i=1}^t\mathbf{C}(\Tilde{A}_i)\leq\sum_{i=1}^t\mathbf{C}(A^*_i)=\mathbf{C}(\mathbf{A}^*_S).\]
Since Algorithm \ref{alg: general heuristic} outputs the minimum-cost collection among all constructed intervention collections, clearly $\mathbf{C}(\mathbf{A})\leq\mathbf{C}(\Tilde{\mathbf{A}})\leq\mathbf{C}(\mathbf{A}^*_S)$.
\end{proof}

\lemmcfpform*
\begin{proof}
    It suffices to show that $\mathbf{A}\in\mathbf{ID}_\mathcal{G}(S,V\setminus S)$, and $\mathbf{C}(\mathbf{A})\leq \mathbf{C}(\mathbf{A}_{S,V\setminus S}^*)$, where $\mathbf{A}_{S,V\setminus S}^*$ is any solution to Equation \eqref{eq:optimization}.
    Since $\mathbf{B}$ is a set cover for $\Gamma$, for every $1\leq j\leq k$, there exists $\Gamma_i\in\mathbf{B}$ such that $S_j\in\Gamma_i$, or equivalently, $A_i^*\in\mathbf{A}$ such that $A_i^*\in\mathbf{ID_1}(S_j)$.
    As a result, for any $1\leq j\leq k$, $Q[S_j]$ is identifiable from $\mathbf{A}$, and therefore $\mathbf{A}\in\mathbf{ID}_\mathcal{G}(S,V\setminus S)$ \citep{tian2002testable}.
    To show optimality, let $\mathbf{A}_{S,V\setminus S}^*$ be an arbitrary solution to Equation \eqref{eq:optimization}.
    From Lemma \ref{lem:partition}, there exists a partitioning of the c-components such as $S^{(1)},\dots,S^{(t)}$ such that $\mathbf{A}_{S,V\setminus S}^*=\{A_1,\dots,A_t\}$, where 
    \[A_j \in\argmin_{A\in \mathbf{ID_1}(\underline{S}^{(j)})}\mathbf{C}(A), \quad\quad\forall 1\leq j\leq t.\]
    The partitions $S^{(j)}$ are disjoint subsets of the set of c-components $\Gamma=\{S_1,\dots,S_k\}$.
    Without loss of generality, assume $\Gamma_j=S^{(j)}$ for $1\leq j\leq t$.
    Then the set $\{\Gamma_1,\dots,\Gamma_t\}$ is a set cover for $\Gamma$.
    Since $\mathbf{B}$ is the minimum-weight set cover for $\Gamma$, 
    \[\mathbf{C}(\mathbf{A})=\sum_{A\in\mathbf{B}}\omega(A)\leq\sum_{1\leq i\leq t}\omega(\Gamma_i)=\sum_{1\leq i\leq t}\mathbf{C}(A_i)=\mathbf{C}(\mathbf{A}_{S,V\setminus S}^*).\]
    
\end{proof}

\prpmcfpalg*
\begin{proof}
    Let $\{S_1,\dots,S_k\}$ be subsets of $S$ such that $\mathcal{G}_{[S_i]}$ for $1\leq i\leq k$ are the maximal c-components in $\mathcal{G}_{[S]}$.
    The proof consists of two parts.
    First, we show that the set family $\mathbf{A}$ returned by Algorithm \ref{alg:mcfp} is a member of $\mathbf{ID}_\mathcal{G}(S,V\setminus S)$ (identifiability).
    Then, we prove that it has a cost at most $k$ times the cost of the minimum cost among all members of $\mathbf{ID}_\mathcal{G}(S,V\setminus S)$ (approximation ratio).

    \textbf{Proof of identifiability.}
        Let $f^*(\cdot,\cdot)$ be the optimal integral solution to the MCFP in line 11 of Algorithm \ref{alg:mcfp}.
        The vertex $z$ receives a flow of $d^*=k$, and it has exactly $k$ incoming edges with capacity $1$ in the flow network $\mathcal{H}$.
        As a result, $f^*(S_j,z)=1$ for every $1\leq j\leq k$.
        By conservation of flow, and from integrality of the solution, we know that for every $j$, there exists exactly one vertex $\Gamma_i$ such that $f^*(\Gamma_i,S_j)=1$.
        Therefore, for each $1\leq j\leq k$, there exists exactly one set $\Gamma_i$ in the set family $\Tilde{\mathbfcal{B}}$ constructed in line 12 such that $f^*(\Gamma_i,S_j)=1>0$.
        But by construction of the flow network, an edge exists between $\Gamma_i$ and $S_j$ if and only if $S_j\in\Gamma_i$.
        As a result, for each $1\leq j\leq k$, there exists exactly one set $\Gamma_i$ in $\Tilde{\mathbfcal{B}}$ such that $S_j\in\Gamma_i$, and $f^*(\Gamma_i,S_j)>0$.
        Therefore, by replacing $\Gamma_i$s with $\Gamma_{i'}$s in line 13, we know that $S_j$ will appear in the corresponding $\Gamma_{i'}$,.
        That is, for each $1\leq j\leq k$, there exists exactly one set $\Gamma_{i'}\in\mathbfcal{B}^*$ such that $S_j\in\Gamma_{i'}$.
        By definition of $A_i^*$s, the intervention set $A_{i'}$ corresponding to $\Gamma_{i'}$ (which is included in $\mathbf{A}^*$) suffices to identify $Q[S_j]$ in $\mathcal{G}$.
        The identifiability result then follows from the fact that $Q[S]$ is identifiable in $\mathcal{G}$ from $\mathbf{A}^*$ if and only if $Q[S_j]$ is identifiable from $\mathbf{A}$ for any $1\leq j\leq k$ \citep{tian2002testable}.
        
    \textbf{Proof of approximation ratio.}
        We first claim that the cost of the optimal flow $f^*()$ is at least as large as a fraction $\frac{1}{k}$ of the cost of intervention collection $\mathbf{A}$ returned by the algorithm.
        To see this, note that the cost of a flow is implied by the amount of flow going through edges between the source $w$ and the vertices $\Gamma_i$, since all other edges have $0$ flow cost.
        As a result, the cost of the optimal flow $f^*$ is
        \begin{equation}\label{eq:flowproof}\sum_{\Gamma_i\in\Tilde{\mathbfcal{B}}}f^*(w,\Gamma_i)\frac{\mathbf{C}(A_i^*)}{\vert\Gamma_i\vert}=\sum_{\Gamma_i\in\Tilde{\mathbfcal{B}}}\frac{\vert\Gamma_{i'}\vert}{\vert\Gamma_i\vert}\mathbf{C}(A_i^*).
        \end{equation}
        We explain why the equality above holds.
        Replacing $\Gamma_i$ with the corresponding $\Gamma_{i'}$ is in a way such that the number of elements in $\Gamma_{i'}$ is equal to the flow into $\Gamma_i$ (which is equal to out-flow of $\Gamma_i$, which in turn is the number of $S_j$s such that $f^*(\Gamma_i,S_j)>0$.)
        As a result, $f^*(w,\Gamma_i)=\vert\Gamma_{i'}\vert$.
        Now note that $\vert\Gamma_i\vert\leq k$, since $\Gamma_i\subseteq\Gamma$.
        Moreover, $\vert\Gamma_{i'}\geq 1$, since $\Gamma_i$ has positive flow and therefore at least one $S_j$ receives positive flow from $\Gamma_i$.
        Finally, since $\Gamma_{i'}\subseteq\Gamma_i$, from linearity of costs, $\mathbf{C}(A_i^*)\geq\mathbf{C}(A_{i'}^*)$.
        Plugging these inequalities into Equation \eqref{eq:flowproof}, the cost of the optimal flow $f^*$ is at least
        \begin{equation}\label{eq:lowerboundflow}
        \sum_{\Gamma_i\in\Tilde{\mathbfcal{B}}}\frac{1}{k}\mathbf{C}(A_{i'}^*) = \frac{1}{k}\sum_{\Gamma_{i'}\in\mathbfcal{B}^*}\mathbf{C}(A_{i'}^*) = \frac{1}{k}\sum_{A_{i'}\in\mathbf{A}}\mathbf{C}(A_{i'}^*)=\frac{1}{k}\mathbf{C}(\mathbf{A}).
        \end{equation}
        Now from Lemma \ref{lem:mcfpform}, there exists a subset of $A_i^*$s, namely $\mathbf{A}^*$ which is an optimal solution to the minimum-cost intervention problem.
        Let the corresponding $\Gamma_i$s be denoted by $\mathbf{B}'$.
        Define a flow $f'(\cdot,\cdot)$ as follows.
        \[
        \begin{split}
            f'(w,\Gamma_i)=\begin{cases}
                \vert\Gamma_i\vert,\quad&\text{if }\Gamma_i\in\mathbf{B}',\\
                0,\quad\text{o.w.}
            \end{cases},\quad f'(\Gamma_i,S_j)=\begin{cases}
                1,\quad&\text{if }\Gamma_i\in\mathbf{B}',S_j\in\Gamma_i\\
                0,&\quad\text{o.w.}
            \end{cases}, f'(S_j,z)=1.
        \end{split}
        \]
        Flow $f'$ satisfies all the flow constraints, and induces cost:
        \[\sum_{\Gamma_i\in\mathbf{B}'}f'(w,\Gamma_i)\frac{\mathbf{C}(A_i^*)}{\vert\Gamma_i\vert}=\sum_{A_i^*\in\mathbf{A}^*}\mathbf{C}(A_i^*)=\mathbf{C}(\mathbf{A}^*).\]
        Since $f^*$ is the minimum-cost flow, the cost of $f^*$ is at most equal to $f'$, that is, $\mathbf{C}(\mathbf{A}^*)$.
        Combining this result with Equation \eqref{eq:lowerboundflow}, we get
        \[\frac{1}{k}\mathbf{C}(\mathbf{A})\leq\mathbf{C}(\mathbf{A}^*).\]
        Multiplying both sides by $k$ completes the proof.
\end{proof}

\prpthmfiveasstwo*
\begin{proof}
    Let $\mathbf{A} =\{A_1,\dots A_m\}$ be a collection of interventions in $\mathbf{ID}_\mathcal{G}(S, V\setminus S)$ for a set $S$ such that $\mathcal{G}_{[S]}$ is a c-component.
    From Theorem 1 of \citet{kivva2022revisiting}, there exists $1\leq i\leq m$ such that $Q[S]$ is identifiable from $Q[V\setminus A_i]$.
    That is, $A_i\in\mathbf{ID_1}(S)$.
    Without loss of generality, suppose $i=1$.
    Define $\mathbf{\Tilde{A}}=\{A_1\}$.
    Clearly, $\mathbf{\Tilde{A}}\in\mathbf{ID}_\mathcal{G}(S, V\setminus S)$.
    As for the cost, applying Assumption \ref{ass:nondecreasing} successively, we get:
    \[\mathbf{C}(\mathbf{\Tilde{A}} ) = \mathbf{C}(\{A_1\})\leq\mathbf{C}(\{A_1,A_2\})\leq\dots\leq\mathbf{C}(\{A_1,A_2,\dots,A_m\})=\mathbf{C}(\mathbf{A}).\]
\end{proof}

\subsection{Results Appearing in the Appendix}
\begin{restatable*}{lemma}{lemtree}\label{lem: tree}
    Let $\mathcal{G}$ be a semi-Markovian graph such that the edge induced subgraphs of $\mathcal{G}$ over its directed edges and over its bidirected edges are trees.
    For an arbitrary vertex $s$ and a vertex $x\in Hhull(s)\setminus\{s\}$, $NC_s(x)$ is a hedge formed for $Q[\{s\}]$ in $\mathcal{G}$.
\end{restatable*}
\begin{proof}
It suffices to show that $\mathcal{G}_{[NC_s(x)]}$ is a c-component and every vertex in $NC_s(x)$ is an ancestor of $s$ in $\mathcal{G}_{[NC_s(x)]}$.
Take an arbitrary vertex $y\in NC_s(x)$.
By definition of $NC_s(\cdot)$, $nec_s(y)\subseteq NC_s(x)$.
As a result, $y$ has both a directed and a bidirected path to $s$ in $\mathcal{G}_{[NC_s(x)]}$, i.e., $y$ is an ancestor of $s$ in $\mathcal{G}_{[NC_s(x)]}$, and in the same c-component as $s$ in $\mathcal{G}_{[NC_s(x)]}$.
Repeating the same argument for every vertex in $NC_s(x)$ completes the proof.
\end{proof}

\begin{restatable*}{lemma}{lemtreecorrect}\label{lem: tree correctness}
    Let $\mathcal{G}$ be a semi-Markovian graph over $V$ such that the edge induced subgraphs of $\mathcal{G}$ over its directed edges and over its bidirected edges are trees.
    For a vertex $s$ in $\mathcal{G}$, Algorithm \ref{alg: tree} returns the min-cost intervention to identify $Q[\{s\}]$ in $\mathcal{G}$ in time $\mathcal{O}(\vert V\vert^3)$.
\end{restatable*}
\begin{proof}
First note that from Lemma \ref{lem: tree}, all the sets in $\mathbf{F}$ in line 9 of the algorithm are hedges formed for $Q[\{s\}]$ in $\mathcal{G}$.
Therefore, any intervention to identify $Q[\{s\}]$ must hit all of these sets.
On the other hand, if all of these sets are hit, by definition of $NC_s(\cdot)$, a hedge formed for $Q[\{s\}]$ cannot include any of the vertices in $Hhull(s)$, i.e., there does not exist any hedge formed for $Q[\{s\}]$ in $\mathcal{G}_{[V\setminus A]}$, where $A$ is a hitting set for $\mathbf{F}$.
As a result, the solution to min-cost intervention problem is the solution to minimum hitting set for $\mathbf{F}$ in line 9 of the algorithm.
Further, we can eliminate any hedge $F'$ from $\mathbf{F}$, if there is a hedge $F\in\mathbf{F}$ such that $F\subseteq F'$.
This is due to the fact that if $F$ is hit, $F'$ is also hit.
Observing that if $x\in NC_s(y)$ then $NC_s(x)\subseteq NC_s(y)$, we can eliminate all such sets $NC_s(y)$ from $\mathbf{F}$.
At the end of this process (after the for loop of lines 10-13), we claim that for any two sets $F,F'$ remaining in $\mathbf{F}$, $F\cap F'=\{s\}$.
Suppose not. 
Then there exists $x\neq s$ such that $x\in F\cap F'$.
Since $F$ and $F'$ are both sets of the Form $NC_s(\cdot)$, as mentioned above, $NC_s(x)\subseteq F$ and $NC_s(x)\subseteq F'$.
Now if $NC_s(x)\in\mathbf{F}$, then both sets $F$ and $F'$ (or at least one of them, in the case that one of them is $NC_s(x)$ itself) would be eliminated from $\mathbf{F}$ during the for loop of lines 12-13.
Otherwise, $NC_s(x)\notin\mathbf{F}$, which means that there exists a vertex $y\in H$ such that $NC_s(y)\subseteq NC_s(x)$, and therefore $NC_s(x)$ has been removed from $\mathbf{F}$.
But in this case, $NC_s(y)\subseteq NC_s(x)\subseteq F,F'$, and in the same iteration where $NC_s(x)$ was removed from $\mathbf{F}$, both $F$ and $F'$ would also be eliminated.
The contradiction shows that when the algorithm reaches line 14, all the sets $\{F\setminus\{s\}\vert F\in\mathbf{F}\}$ are disjoint.
Clearly, the minimum hitting set for disjoint sets includes the minimum-cost member of each of the sets, which is the output of Algorithm \ref{alg: tree}, along with $\PaC{s}$ (Lemma \ref{lem: dirP}).
 
As discussed earlier, constructing the hedge hull of $s$ ($Hhull(s)$) requires at most $\mathcal{O}(\vert V\vert)$ times running a depth-first search, which has linear complexity in trees.
Therefore, constructing the hedge hull has a worst-case time complexity of $\mathcal{O}(\vert V\vert^2)$.
$\PaC{s}$ can also be constructed in linear time, using two one-step breadth-first searches.
Let $H$ denote the hedge hull of $s$ in the graph $\mathcal{G}_{V\setminus\PaC{s}}$. 
Constructing $nec_s(x)$ for the vertices in $H$ can be performed by solving two all-pair shortest paths, which requires two breadth-first search from each vertex with time complexity $\mathbf{O}(\vert H\vert^2)$ in the worst case (BFS in trees requires linear time.)
The while loop of lines 6-8 is performed at most $\vert H\vert$ times, each with linear complexity.
As a result, the $NC_s$ sets and therefore $\mathbf{F}$ are also constructed in time $\mathcal{O}(\vert H\vert^3)$.
The for loop of lines 12-13 requires a sweep over the sets in $\mathbf{F}$, which are at most $\vert H\vert$ many sets, each with at most $\vert H\vert$ members; which can be performed in time $\mathcal{O}(\vert H\vert^2)$, and therefore the for loop of lines 10-13 has complexity $\mathcal{O}(\vert H\vert^3)$.
Finally, calculating the minimum of a set with at most $\vert H\vert$ members can be done in (sub)linear time.
Therefore, the complexity of Algorithm \ref{alg: tree} is $\mathbf{O}(\vert V\vert^3)$ in the worst case.
\end{proof} 

\begin{restatable*}{lemma}{lembipartite} \label{lem: bipartite}
    Let $\mathcal{G}$ be a semi-Markovian graph and $S$ be a subset of its vertices such that $\mathcal{G}_{[S]}$ is a c-component.
    Construct an undirected graph $\mathcal{H}$ on the same set of vertices as $\mathcal{G}\setminus\PaC{S}$ as follows.
    For any hedge of size 2 formed for $Q[S]$ such as $F$, connect the two vertices in $F\setminus S$ with an edge.
    The resulting graph $\mathcal{H}$ is bipartite.
\end{restatable*}
\begin{proof}
First, let $F=\{a,b\}\cup S$ be a hedge formed for $Q[S]$ in the graph $\mathcal{G}\setminus\PaC{S}$.
By definition of hedge, both vertices $a,b$ are ancestors of $S$ in $\mathcal{G}_{[F]}$.
Therefore, at least one of these vertices must be a parent of $S$.
Without loss of generality, assume $a\in\Pa{S}$.
Since $a\notin\PaC{S}$, $a$ does not have a bidirected edge to any vertex in $S$.
However, by definition of hedge, $F$ is a c-component. 
Therefore, $a$ must have a bidirected edge to $b$, and $b$ has a bidirected edge to $S$.
Further, $b\notin\PaC{S}$, and therefore $b\notin\Pa{S}$.
Since $b$ must be an ancestor of $S$ in $\mathcal{G}_{[F]}$, $b\in\Pa{a}$.
As a result, all hedges of size 2 are in the form depicted in Figure \ref{fig: hedge size2}, where $a\in\Pa{S}\setminus \BiD{S}$ and $b\in \BiD{S}\setminus \Pa{S}$.
Accordingly, for any edge drawn in $\mathcal{H}$ such as $\{a,b\}$, exactly one of them is in $\BiD{S}\setminus\Pa{S}$, and the other one is in $\Pa{S}\setminus \BiD{S}$.
Partitioning the vertices of $\mathcal{H}$ into the aforementioned sets, it is clear that $\mathcal{H}$ is bipartite.
\end{proof} 

\begin{restatable*}{lemma}{lemgreedyheuristic}\label{lem: alg heuristic 3}
Given a semi-Markovian graph $\mathcal{G}$ on $V$ and a subset of its vertices $S$ such that $\mathcal{G}_{[S]}$ is a c-component, Algorithm \ref{alg: heursitic 3} returns a set $A$ such that $\{A\}\in\mathbf{ID}_\mathcal{G}(S,V\setminus S)$ in time $\mathcal{O}(\vert V\vert^5)$ in the worst case. 
\end{restatable*}
\begin{proof}
By construction, Algorithm \ref{alg: heursitic 3} outputs a set $A$ such that there is no hedge formed for $Q[S]$ in $\mathcal{G}_{[V\setminus A]}$.
As a result, $\{A\}\in\mathbf{ID}_\mathcal{G}(S,V\setminus S)$.
It only suffices to show that the algorithm halts in time $\mathcal{O}(\vert V\vert^5)$.
Constructing the hedge hull in line 1 is performed in cubic time in the worst case.
The while loop of lines 3-12 can only be executed $\vert H\vert$ times in the worst case, as at each iteration at least one vertex is removed from $H$.
At each iteration of this loop, at most $\vert H\vert$ hedge hulls are constructed, where each of these operations can be done in time $\mathcal{O}(\vert H\vert^3)$.
Summing these up, the algorithm runs in time $\mathcal{O}(\vert V\vert^3+\vert Hhull(S,\mathcal{G})\vert^5)$.
\end{proof}

\begin{restatable*}{lemma}{lemheuristicgeneral}\label{lem: heuristic general}
Given a semi-Markovian graph on $V$ and a subset $S$ of its vertices, Algorithms \ref{alg: heursitic 1}, \ref{alg: heursitic 2} and \ref{alg: heursitic 3} return a subset $A$ of the vertices of $\mathcal{G}$ such that $\{A\}\in \mathbf{ID}_\mathcal{G}(S,V\setminus S)$, in time $\mathcal{O}(\vert V\vert^3)$, $\mathcal{O}(\vert V\vert^3)$ and $\mathcal{O}(\vert V\vert^5)$, respectively.
\end{restatable*}
\begin{proof}
It is straightforward that the arguments used to prove the correctness of these algorithms for the case where $\mathcal{G}_{[S]}$ is a c-component still hold for any maximal c-component of $\mathcal{G}_{[S]}$ for an arbitrary subset $S$ (see the proofs of Lemmas \ref{lem: alg heuristic 1}, \ref{lem: alg heuristic 2} and \ref{lem: alg heuristic 3}.)
Also, $Q[S]$ is identifiable in $\mathcal{G}$ if and only if all of its maximal c-components are identifiable \citep{tian2002testable}.
The result follows immediately.
It is worthy to note that the only overhead in the case that $\mathcal{G}_{[S]}$ is not a c-component is to partition $S$ into its c-components, which can be done using DFS in time $\mathcal{O}(\vert V\vert^3)$ in the worst case, i.e., it does not alter the computational complexity of any of the heuristic algorithms.
\end{proof}

\section{Special Cases \& Improvements}\label{apdx: special case}
\begin{algorithm}[t]
\caption{Polynomial time algorithm for tree-like structures.}
\label{alg: tree}
\begin{algorithmic}[1]
    \State $H\gets Hhull(S,\mathcal{G}_{[V\setminus\PaC{S}]})$
    \State $nec_s(x)\gets (a\stackrel{\mathclap{\normalfont\mbox{$p_u$}}}{\longleftrightarrow} b\cup a\stackrel{\mathclap{\normalfont\mbox{$p_u$}}}{\longrightarrow} b)$ for every $x\in H$
    \State $NC_s(x)\gets \{x\}$ for every $x\in H$
    \For{$x\in H$}
        \State $updated(y)\gets false$ for every $y\in H$
        \While{$\exists y\in NC_s(x)$ s.t. $updated(y)=false$}
            \State $NC_s(x)\gets NC_s(x)\cup nec_s(y)$
            \State $updated(y)\gets true$
        \EndWhile
    \EndFor
    \State $\mathbf{F}\gets\{NC_s(x)\vert x\in H\}$
    \For{$x\in H$}
        \If{$NC_s(x)\in\mathbf{F}$}
            \For{$y\in Hhull(s)$ s.t. $x\in NC_s(y)$}
                \State $\mathbf{F}\gets\mathbf{F}\setminus\{NC_s(y)\}$
            \EndFor
        \EndIf
    \EndFor
    \State $A\gets\{\arg\min_{x\in F\setminus\{s\}}\mathbf{C}(x)\vert F\in \mathbf{F}\}$
    \State {\bfseries return} $A\cup\PaC{s}$
\end{algorithmic}
\end{algorithm}
In this section, we discuss a few special cases of the min-cost intervention problem, and how these cases can be solved efficiently.
We show that under the assumption that the expert has certain knowledge about the structure of the causal graph $\mathcal{G}$, or the cost function $\mathbf{C}(\cdot)$, the problem of designing the minimum-cost intervention can be solved efficiently in polynomial time.
Some of these assumptions might seem restrictive.
However, as we shall discuss, they provide useful insight towards solving the min-cost intervention problem efficiently in more practical settings.
\subsection{Tree-like structure of $\mathcal{G}$}
We begin with a special structure of the semi-Markovian graph $\mathcal{G}$, where both the edge induced subgraphs of $\mathcal{G}$ over the directed edges and over the bidirected edges are trees.
Between any pair of vertices in a tree, there is a unique path.
As a result, for any two vertices $a,b$ in $\mathcal{G}$, there is a unique path using bidirected edges, and if $a$ is an ancestor of $b$, there is also a unique path from $a$ to $b$ using directed edges.
We denote these unique bidirected and directed paths by $a\stackrel{\mathclap{\normalfont\mbox{$p_u$}}}{\longleftrightarrow} b$ and $a\stackrel{\mathclap{\normalfont\mbox{$p_u$}}}{\longrightarrow} b$, respectively.
Note that $a\stackrel{\mathclap{\normalfont\mbox{$p_u$}}}{\longleftrightarrow} b$ and $a\stackrel{\mathclap{\normalfont\mbox{$p_u$}}}{\longrightarrow} b$ for every pair of vertices can be found using an all-pair shortest path algorithm (two separate breadth-first search from each vertex) in time $\mathcal{O}(\vert V\vert^3)$.
Now suppose we want to solve the min-cost intervention set problem for $Q[\{s\}]$.
Take an arbitrary variable $x\neq s$ from $Hhull(s)$.
Let $F$ be a hedge formed for $Q[\{s\}]$ such that $x\in F$.
Since $F$ is a c-component and $x$ is an ancestor of $s$ in $\mathcal{G}_{[F]}$, all of the variables on both $a\stackrel{\mathclap{\normalfont\mbox{$p_u$}}}{\longleftrightarrow} b$ and $a\stackrel{\mathclap{\normalfont\mbox{$p_u$}}}{\longrightarrow} b$ must be members of $F$.
We therefore call the union of all these variables, the necessary set of $x$ to form a hedge for $s$, and we denote this set by $nec_s(x)$.
Clearly, if we intervene on at least one vertex from $nec_s(x)$, then no hedge formed for $Q[\{s\}]$ contains $x$.
Further, we observe that if $y\in nec_s(x)$, then with the same arguments, if a hedge formed for $Q[\{s\}]$ contains $x$, it must contain $y$ and therefore all the variables in $nec_s(y)$ as well.
We define the closure of necessary variables for $x$ to form a hedge for $Q[\{s\}]$ as follows.
\begin{definition}[Necessary closure]
    Let $\mathcal{G}$ be a semi-Markovian graph such that the edge induced subgraphs of $\mathcal{G}$ over its directed edges and over its bidirected edges are trees.
    We say a subset $A$ of vertices of $\mathcal{G}$ is a closure of necessary variables for $x$ to form a hedge for $Q[\{s\}]$, if $x\in A$, and for every $y\in A$, $nec_s(y)\subseteq A$.
    We denote the minimum closure of necessary variables for $x$ by $NC_s(x)$.
\end{definition}
The following lemma indicates that minimum closure of necessary variables for $x$ is a hedge formed for $Q[\{s\}]$.

\lemtree
All of the proofs are provided in Appendix \ref{apdx: proofs}. 
One observation is that to solve the min-cost intervention, we can enumerate $NC_s(x)$ for every $x\in Hhull(s)$ and solve the hitting set problem for these hedge.
Although the number of such hedges is exactly $\vert Hhull(s)\vert -1$, the hitting set problem is still complex to solve.
However, we can further reduce the complexity of the problem as follows.
First note that if $y\in NC_s(x)$, by definition of $NC_s(\cdot)$, $NC_s(y)\subseteq NC_s(x)$.
Therefore, when considering the hitting set problem, if $NC_s(y)$ is hit, $NC_s(x)$ will also be hit.
As a result, we can eliminate $NC_s(x)$ from the sets we are considering.
Using the same argument, we begin with some random ordering over the variables $Hhull(S)$ and for every $x\in Hhull(S)$, if $x$ appears in $NC_s(y)$ for some $y\in Hhull(S)$ and the set $NC_s(x)$ is not eliminated yet, we eliminate $NC_s(y)$.
At the end of this procedure, we are left with a collection of hedges $\mathbf{F}$ that satisfies the following properties.
\begin{enumerate}
    \item The min-cost intervention to identify $Q[\{s\}]$ in $\mathcal{G}$ is the min-cost hitting set solution to $\{F\setminus\{s\}\vert F\in \mathbf{F}\}$.
    \item For any two hedges $F,F'\in \mathbf{F}$, $F\cap F'=\{s\}$.
    \item $\vert \mathbf{F}\vert \leq\vert Hhull(s)\vert$\footnote{For a formal proof of these properties, refer to the proof of Lemma \ref{lem: tree correctness}.}.
\end{enumerate}
Now we observe that the collection $\mathbf{F}$ of hedges are mutually disjoint, and thus the minimum hitting set is simply the union of the minimum cost vertex in each hedge.
The following result indicates the correctness and the time complexity of algorithm \ref{alg: tree}, under the assumption that $\mathcal{G}$ has a tree-like structure.
\lemtreecorrect
There are further considerations to Algorithm \ref{alg: tree} that we would like to mention.
The first one is that this algorithm together with the definitions of $nec_s$ and $NC_s$, suggest an alternative formulation of the min-cost intervention problem, which is taking into account the set of variables that must be combined together with each variable $x$ to form a hedge for $Q[S]$.
The definition of $nec_s(x)$ can be generalized to the case that $\mathcal{G}$ is not a tree anymore, although $nec_s(x)$ will not be a set anymore, but a collection of sets where if an intervention is made upon at least one vertex of all of these sets, no remaining hedge formed for $Q[\{s\}]$ includes $x$.
This indeed suggests a method to enumerate the hedges formed for $Q[S]$ in $\mathcal{G}$.
As we saw in this section, for tree-like structures, this enumeration can be executed in polynomial time. 
However, in general structures, this enumeration method would still take exponential time in the worst case.
Another point to mention is that one-step generalizations of the assumption of tree-ness and Algorithm \ref{alg: tree} can be thought of, such as the assumption that the number of paths between each pair of vertices in $\mathcal{G}$ is at most 2 (or $k$, where $k$ is a constant.)
Although the tree assumption made in this section might appear restrictive, such generalizations might yield efficient solutions of the min-cost intervention problem that can be used in practice.
\subsection{Bounded hedge size}
\begin{algorithm}[t]
\caption{Polynomial algorithm for bounded hedges.}
\label{alg: bounded hedge}
\begin{algorithmic}[1]
    \State $\mathcal{H}\gets$ empty undirected graph on $V\setminus\PaC{S}$
    \For{any pair of vertices $\{a,b\}\subseteq V\setminus (S\cup\PaC{S})$}
        \If{$\{a,b\}\cup S$ is a hedge formed for $Q[S]$}
            \State draw an edge between $a$ and $b$ in $\mathcal{H}$
        \EndIf
    \EndFor
    \State $A\gets$ the min-weight vertex cover for $\mathcal{H}$
    \State {\bfseries return} $A\cup\PaC{S}$
\end{algorithmic}
\end{algorithm}
Following the hitting set formulation for the min-cost intervention problem, the two main challenges were enumerating the hedges and solving the hitting set problem afterwards.
For a hedge $F$ formed for $Q[S]$ in $\mathcal{G}$, let $(\vert F\vert-\vert S\vert)$ be the size of this hedge, which is exactly the size of the set to be hit in the hitting set equivalent.
If an upper bound on the size of the hedges formed for $Q[S]$ such as $(\vert F\vert-\vert S\vert)\leq k$ is know where $k$ is a constant, then the task of enumerating the hedges can be performed in polynomial time, as we only need to check the subsets of up to size $k$.
Note that as discussed in Section \ref{sec: hitsetform}, this argument is still valid if the upper bound works for the set of \emph{minimal} hedges.
However, the hitting set task still remains exponential in the worst case.
On the other hand, for certain values of $k$, the min-cost intervention problem can be solved in polynomial time without using the hitting set formulation.
For $k=1$, the set of minimal hedges formed for $Q[S]$ reduces to the hedge structures composed of $S$ and one variable in $\PaC{S}$.
From lemma \ref{lem: dirP}, we know that in such a structure, the optimal intervention is $A^*=\PaC{S}$, and $\PaC{S}$ can be constructed in linear time.
In this section, we show that for $k=2$, that is, given that every minimal hedge formed for $Q[S]$ has size at most 2, the min-cost intervention problem can be solved in polynomial time.
We begin with the following property of the formed hedges, which will help us model the min-cost intervention problem as a maximum matching problem in a bipartite graph through Konig's theorem \citep{konig1931graphok}.
\lembipartite
First, note that for any hedge $F$ formed for $Q[S]$ in $\mathcal{G}_{[V\setminus\PaC{S}]}$, there exists an edge between the two vertices $F\setminus S$ in the undirected graph $\mathcal{H}$.
Since the min-cost intervention to identify $Q[S]$ in $\mathcal{G}$ is the union of $\PaC{S}$ and the minimum hitting set for the sets $F\setminus S$ (Lemma \ref{lem: hitsetform}), the min-cost intervention can also be given as the union of $\PaC{S}$ and the minimum vertex cover for the undirected graph $\mathcal{H}$.
Lemma \ref{lem: bipartite} states that $\mathcal{H}$ is bipartite.
It is known that in bipartite graphs, the minimum-weight vertex cover problem is equivalent to a maximum matching (when the costs are uniform), or a maximum flow problem (when the costs are not uniform) \citep{konig1931graphok}.
There are various polynomial time algorithms to solve these problems, such as Ford-Fulkerson, Edmonds-Karp and push-relabel algorithms to name a few \citep{ford1956maximal,edmonds1972theoretical,goldberg1988new}.
Consequently, under the assumption that for any minimal hedge $F$ formed for $Q[S]$, $\vert F\vert -\vert S\vert\leq 2$, we propose Algorithm \ref{alg: bounded hedge} to solve the min-cost intervention problem in polynomial time.
Any appropriate algorithm can be used as a subroutine in line (5) of Algorithm \ref{alg: bounded hedge}.

\subsection{Special cost functions}
\begin{algorithm}[t]
\caption{Polynomial time algorithm for special $\mathbf{C}(\cdot)$.}
\label{alg: special cost}
\begin{algorithmic}[1]
    \State initialize $I\gets\emptyset$
    \If{$I\in\mathbf{ID_1}(S)$}
        \State \textbf{return} $I$
    \EndIf
    \State $V'\gets$ ancestors of $S$ in $\mathcal{G}$, $\{v_1,...,v_k\}$
    \State sort the vertices $V'$ s.t. $\mathbf{C}(v_i)<\mathbf{C}(v_{i+1})$ $\forall 1\leq i<k$
    \While{true}
        \State $i\gets 0$
        \State $I'\gets I$
        \While{true}
            \State $i\gets i+1$
            \State $I'\gets I'\cup\{v_i\}$
            \If{$I'\in\mathbf{ID_1}(S)$}
                \State \textbf{break}
            \EndIf
        \EndWhile
        \State $I\gets I\cup\{v_i\}$
        \If{$I\in\mathbf{ID_1}(S)$}
            \State \textbf{return} $I$
        \EndIf
    \EndWhile
\end{algorithmic}
\end{algorithm}
We have discussed special graph structures so far.
However, in certain cases, knowledge about the form of the cost function $\mathbf{C}(\cdot)$ can help us solve the min-cost intervention problem efficiently.
One such case is when the costs of intervening on variables are far enough from each other. 
As a concrete example, let the vertices of $\mathcal{G}$ be $v_1, ..., v_n$, with the cost function $\mathbf{C}(v_i)=2^i$ for $1\leq i\leq n$.
We begin with testing the sets $\{v_1\},\{v_1,v_2\},...,\{v_1,v_2,...,v_n\}$, until we reach at the first set $I_j=\{v_1,...,v_j\}\in\mathbf{ID_1}(S)$.
Since the cost of this intervention is $\mathbf{C}(I_j)=\sum_{i=1}^j 2^i<2^{j+1}$, the min-cost intervention does not include any of the variables $v_{j+1},...,v_n$, as the cost of any of these variables is at least $2^{j+1}$.
Further, as intervening on more variables cannot induce new hedges, and by definition of $I_j$, no subset of $\{v_1,...,v_{j-1}\}$ is in $\mathbf{ID_1}(S)$.
This implies that if $A^*$ is a min-cost intervention to identify $Q[S]$ in $\mathcal{G}$, then $v_j\in A^*$ and $v_l\notin A^*$ for any $l>j$.
We then restart the procedure, testing the sets $\{v_1\}\cup\{v_j\},\{v_1, v_2\}\cup\{v_j\},...,\{v_1,...,v_{j-1}\}\cup\{v_j\}$ to find the first set $I_k=\{v_1,...,v_k\}\cup\{v_j\}\in\mathbf{ID_1}(S)$.
Again with the same arguments, we can conclude that $v_k\in A^*$ and $v_l\notin A^*\setminus\{v_j\}$ for any $l>k$.
Continuing in the same manner, we construct the min-cost vertex cover (which is unique in this setting) after $\mathcal{O}(\vert V\vert)$ iterations in the worst case.
Each iteration tests whether a set is a hedge at most $\vert V\vert$ times, which can be performed using two depth-first searches ($\mathcal{O}(\vert V\vert^2)$).
As a result, the min-cost intervention can be solved in time $\mathcal{O}(\vert V\vert^4)$ in the worst case.

Note that the property we used throughout our reasoning was the fact that having sorted the variables based on their intervention costs as $v_1,...,v_n$, for any $1\leq j<n$, $\mathbf{C}(\{v_1,...,v_i\})<\mathbf{C}(v_{i+1})$.
With such a cost function, Algorithm \ref{alg: special cost} solves the min-cost intervention problem in time $\mathcal{O}(\vert V\vert^4)$ in the worst case, i.e., regardless of the structure of $\mathcal{G}$.
Note that this algorithm has the same worst-case time complexity for both when $\mathcal{G}_{[S]}$ is a c-component and when it is not.
Also note that as an optional step, we can begin with constructing the hedge hull of $S$ in $\mathcal{G}_{V\setminus\PaC{S}}$ (denoted by $H$) if $\mathcal{G}_{[S]}$ is a c-component.
In this case, sorting the variables in $H$ based on their cost as $h_1,...,h_m$, we only need the assumption that $\mathbf{C}(\{h_1,...,h_i\})<\mathbf{C}(h_{i+1})$ for $1\leq i\leq m-1$, and the worst-case time complexity would be $\mathcal{O}(\vert V\vert^3+\vert H\vert^4)$.
\begin{lemma}
Let $\mathcal{G}$ be a semi-Markovian graph on vertices $V$, along with a cost function $\mathbf{C}(\cdot)$.
Let $S$ be subset of vertices of $\mathcal{G}$, and $V'=\{v_1,...,v_k\}$ be the set of ancestors of $S$ in $\mathcal{G}$.
If $\mathbf{C}(\{v_1,...,v_i\})<\mathbf{C}(v_{i+1})$ for every $1\leq i<k$, then Algorithm \ref{alg: special cost} returns the min-cost intervention to identify $Q[S]$ in $\mathcal{G}$ in time $\mathcal{O}(\vert V\vert^4)$.
\end{lemma}

\section{Heuristic Algorithms}\label{apdx: heuristic}
In this section, we first present the three heuristic algorithms proposed in Section \ref{sec: heuristic}.
We discuss their correctness, their running times, and how they compare to each other.
Later, we propose a polynomial-time improvement that can be utilized as a post-process to improve the output of these algorithms.

The first heuristic algorithm is depicted as Algorithm \ref{alg: heursitic 1}.
We begin with removing $\PaC{S}$ from the graph, as we already know that this set must be included in the output.
We then build an undirected graph $\mathcal{H}$ over the vertices of $H=Hhull(S,\mathcal{G}_{[V\setminus\PaC{S}]})$, along with two extra vertices $x$ and $y$.
For every bidirected edge $\{v_1,v_2\}$ in $\mathcal{G}_{[H]}$, we draw a corresponding edge between $v_1$ and $v_2$ in $\mathcal{H}$.
Finally, we connect $x$ to $\Pa{S}\cap H$ and $y$ to $S$ with an edge.
Note that every undirected path between $x$ and $y$ in $\mathcal{H}$ corresponds to a bidirected path that connects a vertex in $S$ to a vertex in $\Pa{S}\cap H$ in $\mathcal{G}$.
If we intervene on a subset of variables $A$ such that no such path exists anymore, the hedge hull of $S$ in the remaining graph will be $S$ itself, as none of the vertices $\Pa{S}$ are in the same c-component of $S$.
Consequently, the effect $Q[S]$ becomes identifiable.
With that being said, we solve for the minimum-weight vertex cut for $x-y$ in $\mathcal{H}$ in line (4) of the algorithm.
We set the weights of the vertices in $S$ to infinity to ensure that we do not intervene on them.
Note that the min-weight vertex cut in an undirected graph can be turned into an equivalent problem in a directed graph, by simply substituting every undirected edge with two directed edges in the opposite direction.
Further, min-weight vertex cut can be reduced to min-weight edge cut through a trivial reduction: 
We replace every vertex $v$ with two vertices $v_1,v_2$, add an edge from $v_1$ to $v_2$ with the same weight as the weight of $v$ in the original graph, and connect every edge that goes into $v$ to $v_1$, and every edge that goes out of $v$ to $v_2$.
The resulting problem can be solved using any of the standard max-flow-min-cut algorithms.
We used the push-relabel algorithm to solve the max-flows throughout our simulations \citep{goldberg1988new}.
\begin{algorithm}[t]
\caption{Heuristic algorithm 1.}
\label{alg: heursitic 1}
\begin{algorithmic}[1]
    \State {\bfseries input:} $\mathcal{G}, S, \mathbf{C}(\cdot)$, {\bfseries output:} $A\in\mathbf{ID_1}(S)$
    \State $H\gets Hhull(S, \mathcal{G}_{V\setminus\PaC{S}})$
    \State Build $\mathcal{H}$ on $H\cup\{x,y\}$: draw an undirected edge between $v_1,v_2\in H\setminus S$ if there is a bidirected edge between them in $\mathcal{G}$. Connect $x$ to $\Pa{S}\cap H$ and $y$ to $S$.
    \State $MC\gets$ minimum-weight vertex cut for $x-y$ in $\mathcal{H}$, with weights $\omega(v)=\mathbf{C}(v)$ for $v\notin S$ \& $\omega(s)=\infty$ for $s\in S$
    \State $A\gets MC\cup\PaC{S}$
    \State {\bfseries return} $A$
\end{algorithmic}
\end{algorithm}

The second heuristic algorithm, depicted as Algorithm \ref{alg: heursitic 2}, relies on similar ideas.
Again, we begin with removing $\PaC{S}$ from the graph, as we already know that this set must be included in the output.
We then build a directed graph $\mathcal{H}$ over the vertices of $H=Hhull(S,\mathcal{G}_{[V\setminus\PaC{S}]})$, along with two extra vertices $x$ and $y$.
For every directed edge $v_1\to v_2$ in $\mathcal{G}_{[H]}$, we draw a corresponding edge between $v_1\to v_2$ in $\mathcal{H}$.
Finally, we draw an edge from $x$ to all vertices in $\BiD{S}\cap H$ and from all vertices in $S$ to $y$.
Note that every directed path from $x$ to $y$ in $\mathcal{H}$ corresponds to a directed path that connects a vertex in $\BiD{S}$ to a vertex in $S$ in $\mathcal{G}$.
If we intervene on a subset of variables $A$ such that no such path exists anymore, the hedge hull of $S$ in the remaining graph will be $S$ itself, as none of the vertices $\BiD{S}$ have a directed path to $S$.
Consequently, the effect $Q[S]$ becomes identifiable.
With that being said, we solve for the minimum-weight vertex cut for $x-y$ in $\mathcal{H}$ in line (4) of the algorithm.
We set the weights of the vertices in $S$ to infinity to ensure that we do not intervene on them.
As mentioned above, we reduce the min-weight vertex cut to min-weight edge cut, and then use max-flow algorithms to solve it.

\begin{algorithm}[t]
\caption{Heuristic algorithm 2.}
\label{alg: heursitic 2}
\begin{algorithmic}[1]
    \State {\bfseries input:} $\mathcal{G}, S, \mathbf{C}(\cdot)$, {\bfseries output:} $A\in\mathbf{ID_1}(S)$
    \State $H\gets Hhull(S, \mathcal{G}_{V\setminus\PaC{S}})$
    \State Build $\mathcal{H}$ on $H\cup\{x,y\}$: for $v_1,v_2\in H\setminus S$, draw $v_1\to v_2$ in $\mathcal{H}$ if this edge exists in $\mathcal{G}$. Draw the edges from $x$ to $\Pa{S}\cap H$ and from $S$ to $y$
    \State $MC\gets$ minimum-weight vertex cut for $x-y$ in $\mathcal{H}$, with weights $\omega(v)=\mathbf{C}(v)$ for $v\notin S$ \& $\omega(s)=\infty$ for $s\in S$
    \State $A\gets MC\cup\PaC{S}$
    \State {\bfseries return} $A$
\end{algorithmic}
\end{algorithm}

Finally, we proceed to our third heuristic algorithm, which is based on a greedy approach.
First, note that if we intervene on every variable in the hedge hull of $S$ except $S$, $Q[S]$ becomes identifiable.
That is, defining $H=Hhull(S,\mathcal{G}_{[V\setminus\PaC{S}]})$, one trivial set in $\mathbf{ID}_\mathcal{G}(S,V\setminus S)$ is $\{(H\setminus S)\cup\PaC{S}\}$.
Similarly, if we intervene on a set of variables $A$, then $A\cup Hhull(S,\mathcal{G}_{[V\setminus A]}\setminus S)$ is a trivial solution.
In our greedy approach, we minimize the cost of this trivial solution at each iteration.
We proceed as follows.
We maintain an intervention set $A$, which is initialized as $\PaC{S}$.
At each iteration, we find the vertex $x\in Hhull(S,G_{[V\setminus A]})\setminus S$ that minimizes the objective function 
\[f(x)=\mathbf{C}(x)+\mathbf{C}(Hhull(S,\mathcal{G}_{[V\setminus(A\cup\{x\})]})),\]
and add this vertex to $A$.
Note that the function $f(x)$ is exactly the cost of the trivial solution in graph $\mathcal{G}\setminus(A\cup\{x\})$.
We add one vertex in each iteration until we reach a point where $Q[S]$ becomes identifiable.
The following result indicates the correctness of Algorithm \ref{alg: heursitic 3} along with its computational complexity.

\lemgreedyheuristic
\begin{algorithm}
\caption{Heuristic greedy algorithm.}
\label{alg: heursitic 3}
\begin{algorithmic}[1]
    \State $H\gets Hhull(S,\mathcal{G}_{[V\setminus\PaC{S}]})$
    \State initialize $A\gets\PaC{S}$
    \While{$H\neq S$}
        \State $c_{min}\gets \mathbf{C}(H)$
        \State $i\gets null$
        \For{$v\in H$}
            \State $H'\gets Hhull(S,\mathcal{G}_{[H\setminus \{v\}]})$ 
            \If{$\mathbf{C}(H')+\mathbf{C}(v)\leq c_{min}$}
                \State $c_{min}\gets \mathbf{C}(H')+\mathbf{C}(v)$
                \State $i\gets v$
            \EndIf
        \EndFor
        \State $H\gets Hhull(S,\mathcal{G}_{[H\setminus \{i\}]})$
        \State $A\gets A\cup \{i\}$
    \EndWhile
    \State {\bfseries return} $A$
\end{algorithmic}
\end{algorithm}

\paragraph{General subset identification using heuristic algorithms.}
The heuristic algorithms proposed in this work are devised under the assumption that $\mathcal{G}_{[S]}$ is a c-component.
However, as claimed in the main text, all of the three heuristic algorithms return a valid intervention set to identify $Q[S]$ in $\mathcal{G}$, even if $\mathcal{G}_{[S]}$ is not a c-component.
This follows from the result that $Q[S]$ is identifiable in $\mathcal{G}$, if and only if $Q[S_1],...,Q[S_k]$ are identifiable in $\mathcal{G}$, where $S_1,...,S_k$ are the maximal c-components of $\mathcal{G}_{[S]}$ \citep{tian2002testable}.
The following result formalizes this claim.

\lemheuristicgeneral
Note that we Lemma \ref{lem: heuristic general} does not require that $\mathcal{G}_{[S]}$ be a c-component, unlike Lemmas \ref{lem: alg heuristic 1} and \ref{lem: alg heuristic 2}.
As a result, all of these algorithms can also be utilized as a subroutine in line (7) of Algorithm \ref{alg: general heuristic}, the general algorithm proposed in this work.
\paragraph{Post-process.}
In many cases, when the output of the proposed heuristic algorithms is not optimal, it is a super-set of the optimal intervention.
As a result, we propose greedily deleting such extra variables from the intervention set $A$ while $Q[S]$ remains identifiable.
That is, assuming $A$ is the output of one of the Algorithms \ref{alg: heursitic 1},\ref{alg: heursitic 2},\ref{alg: heursitic 3}, we start with the vertex $a\in A$ with the highest cost, and while there exists $a\in A\setminus\PaC{S}$ such that $\{A\setminus\{a\}\}\in\mathbf{ID}_\mathcal{G}(S,V\setminus S)$, we remove $a$ from $A$.
Testing whether a set is in $\mathbf{ID}_\mathcal{G}(S,V\setminus S)$ requires time $\mathcal{O}(\vert V\vert^3)$ in the worst case.
As a result, the proposed post-process does not alter the worst-case complexity of the algorithms.
\paragraph{Discussion.}
The proposed algorithms have no theoretical guarantee of how well they can approximate the solution to the min-cost intervention problem.
However, their performances as well as their runtimes are dependent on the structure of the graph $\mathcal{G}_{[H]}$, where $H=Hhull(S,\mathcal{G}_{[V\setminus\PaC{S}]})$.
For instance, if the edge-induced subgraph of $\mathcal{G}_{[H]}$ on its bidirected edges is much more dense than the edge-induced subgraph of $\mathcal{G}_{[H]}$ on its directed edges, Algorithm \ref{alg: heursitic 1} will need to solve a more complex min-weight vertex cover problem compared to Algorithm \ref{alg: heursitic 2}.
It will also add potentially many extra vertices that are not needed in the intervention set.
Since $\mathcal{G}_{[H]}$ is constructed as a pre-process of all three algorithms, we propose choosing the heuristic algorithm after constructing $\mathcal{G}_{[H]}$ as follows.
Algorithm \ref{alg: heursitic 2} is preferred over the other two, as it solves a min vertex cut in a directed graph rather than an undirected graph. 
However, if the graph $\mathcal{G}_{[H]}$ is dense on its directed edges, we choose Algorithm \ref{alg: heursitic 1}.
In certain cases, as shown by our empirical evaluation, the greedy approach achieves lower regret despite the higher time complexity.

\section{Hitting Set \& Algorithm \ref{alg: ultimate}}\label{apdx: hit set}
\subsection{Greedy approach for minimum hitting set}
In this section, we present the greedy weighted minimum hitting set algorithm mentioned in the main text \citep{johnson1974approximation}.
This greedy approach is depicted in Algorithm \ref{alg: greedy hitting set}.
Let $V$, $\mathbf{F}$, and $\omega(\cdot)$ be the universe of objects, the collection of sets for which we want to find a hitting set, and the weight function respectively.
For an object $v\in V$, we denote by $N(v)$ the number of sets $F\in\mathbf{F}$ such that $v\in F$, that is, the number of sets $v$ hits.
We begin with an empty hitting set $A$.
At each iteration, we choose the variable $v\in V$ that maximizes $\frac{N(v)}{\omega(v)}$, and add it to $A$.
We then remove all the sets $F$ that include $v$ from $\mathbf{F}$.
The algorithm runs until $\mathbf{F}$ becomes empty.
The resulting set $A$ is a hitting set for $\mathbf{F}$.
It has been shown that this greedy algorithm achieves a logarithmic-factor approximation of the optimal hitting set in the worst case \citep{johnson1974approximation,chvatal1979greedy}.
Note that using certain data structures, we can avoid recalculating $N(v)$ at each iteration in line (4).

\begin{algorithm}
\caption{Greedy weighted minimum hitting set algorithm.}
\label{alg: greedy hitting set}
\begin{algorithmic}[1]
    \State {\bfseries input:} universe $V$, collection of sets $\mathbf{F}$, weights $\omega(v)$ for $v\in V$, {\bfseries output:} a hitting set for $\mathbf{F}$
    \While{$\mathbf{F}\neq\emptyset$}
        \ForAll{$v\in V$}
            \State $N(v)\gets \vert\{F\in\mathbf{F}\vert v\in F\}\vert$
        \EndFor
        \State $v\gets\arg\min_{v\in V}\frac{N(v)}{\omega(v)}$
        \State $A\gets A\cup\{v\}$
        \State $\mathbf{F}\gets\mathbf{F}\setminus\{F\in\mathbf{F}\vert v\in F\}$
    \EndWhile
    \State {\bfseries return} $A$
\end{algorithmic}
\end{algorithm}
\subsection{On Algorithm \ref{alg: ultimate}}
In this section, we provide a slight modification of Algorithm \ref{alg: ultimate}.
One caveat to Algorithm \ref{alg: ultimate} is that it might call numerous times as a subroutine, a solution to the minimum hitting set problem (line (13)).
Although we propose using the greedy approach mentioned above as the subroutine, we also provide a modification, depicted as Algorithm \ref{alg: ultimate 2}, which reduces the number of calls to this subroutine as follows.
At the end of each iteration (inner loop, that is, lines (7-13)), instead of solving the minimum hitting set problem, we simply add the vertex $a$ found in the last step to a set of interventions $A$.
We postpone the call to minimum hitting set to when $A$ grows large enough so that $\{A\}\in\mathbf{ID}_\mathcal{G}(S,V\setminus S)$.
Through this modification, we discover more hedges and add them to $\mathbf{F}$ before calling for the solution of the minimum hitting set problem.
Therefore, this modification reduces the number of calls to the subroutine of solving the min hitting set in certain cases.

\begin{algorithm}[t]
\caption{Modified algorithm to reduce the calls to minimum hitting set.}
\label{alg: ultimate 2}
\begin{algorithmic}[1]
    \State $\mathbf{F}\gets\emptyset, A\gets\emptyset$
    \State $H\gets Hhull(S,\mathcal{G}_{[V\setminus\PaC{S}]})$
    \If{$H=S$}
        \State \textbf{return} $\PaC{S}$
    \EndIf
    \While{True}
        \While{True}
            \While{True}
                \State $a\gets\arg\min_{a\in H\setminus S}\mathbf{C}(a)$
                \If{$Hhull(S,\mathcal{G}_{[H\setminus\{a\}]})=S$}
                    \State $\mathbf{F}\gets \mathbf{F}\cup\{H\}$
                    \State \textbf{break}
                \Else
                    \State $H\gets Hhull(S,\mathcal{G}_{[H\setminus\{a\}]})$
                \EndIf
            \EndWhile
            \State $A\gets A\cup\{a\}$
            \If{$\{A\cup\PaC{S}\}\in\mathbf{ID}_\mathcal{G}(S,V\setminus S)$}
                \State {\bfseries break}
            \EndIf
            \State $H\gets Hhull(S,\mathcal{G}_{[V\setminus (A\cup\PaC{S})]})$
        \EndWhile
        \State $A\gets$ min hitting set for $\{F\setminus S\vert F\in\mathbf{F}\}$
        \If{$\{A\cup\PaC{S}\}\in\mathbf{ID}_\mathcal{G}(S,V\setminus S)$}
            \State {\bfseries return} $A\cup\PaC{S}$
        \EndIf
    \EndWhile
\end{algorithmic}
\end{algorithm}

\section{Further Empirical Evaluation}\label{apdx: empirical}
In this section, we provide further details of the experimental setup of the paper.
We also provide complementary evaluations of our proposed algorithms.

\paragraph{Setup.} We have evaluated our algorithms in two different settings.
In Appendix \ref{apdx:benchmark}, we evaluate our algorithms on a set of well-known graphs, which are the benchmark causal graphs in the causality literature.
These graphs are obtained under the assumption of no latent variables.
However, often the observed variables of a system are confounded by a hidden variable.
We added a common confounder for each pair of variables in these graphs with probability $q$.
We then ran our algorithms to find the min-cost intervention for identifying $Q[S]$, where $S$ is the last vertex in the causal order.
We assumed that the cost of intervening on each variable is uniformly sampled from $\{1,2,3,4\}$.

In the second setting considered throughout our evaluations, we generated random graphs based on Erdos-Renyi generative model.
The directed and bidirected edges of the graph in this model are sampled mutually independently, with probabilities $p$ and $q$ respectively.
We then assigned a random cost of intervening to each variable, sampled from the uniform distribution over $\{1,2,3,4\}$.
Set $S$ in these set of evaluations is randomly chosen among the last 5\% vertices of the graph, such that $\mathcal{G}_{[S]}$ is a c-component.
Appendix \ref{apdx:random} provides empirical results of our algorithms on the randomly generated graphs.
Finally, an evaluation of the hedge enumeration task of Algorithm \ref{alg: ultimate} is given in Figure \ref{fig:numhedge}.

\subsection{Benchmark Structures} \label{apdx:benchmark}

In this section, we evaluate our algorithms on graphs corresponding to real-world problems, namely the Barley \citep{kristensen1997decision}, Water \citep{jensen1989expert} and Mehra \citep{vitolo2018modeling} structures \footnote{See https://www.bnlearn.com/bnrepository/ for details.}.
These structures are formed as causal DAGs under the assumption of no hidden confounder.
However, often hidden variables confound observed variables.
In our experiments, we randomly added a latent confounder for every pair of variables with probability $q\in\{0.05, 0.15,0.25,0.35\}$, and evaluated the performance of our algorithms.
The intervention costs are assigned uniformly at random from $\{1,2,3,4\}$, and the set $S$ is chosen to be the last vertex in the causal ordering.
The results are depicted in Figure \ref{fig:realworld}.

\begin{figure}
    \centering
    \begin{subfigure}[b]{0.99\textwidth}
        \centering
        \includegraphics[width=0.43\textwidth]{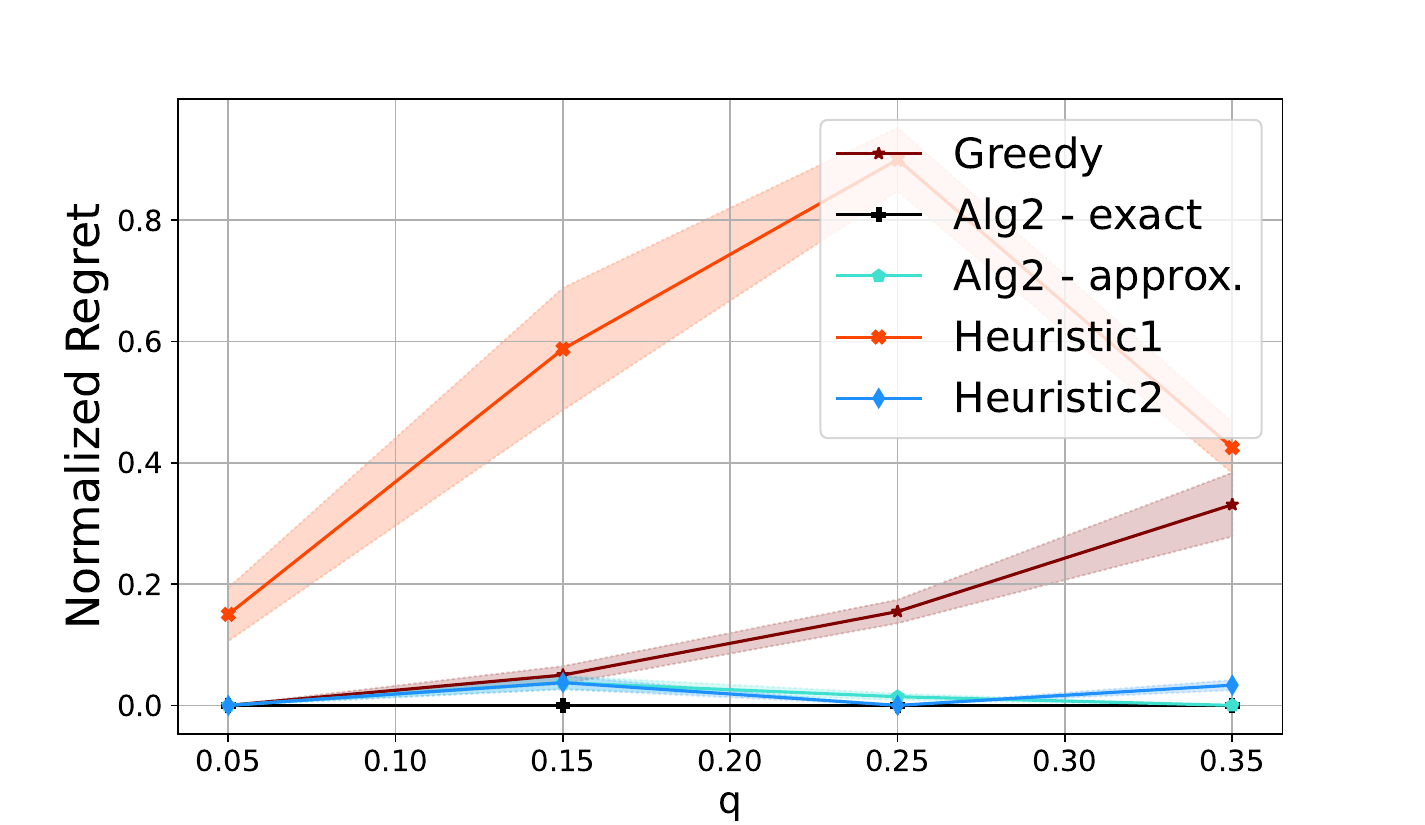}
        \includegraphics[width=0.43\textwidth]{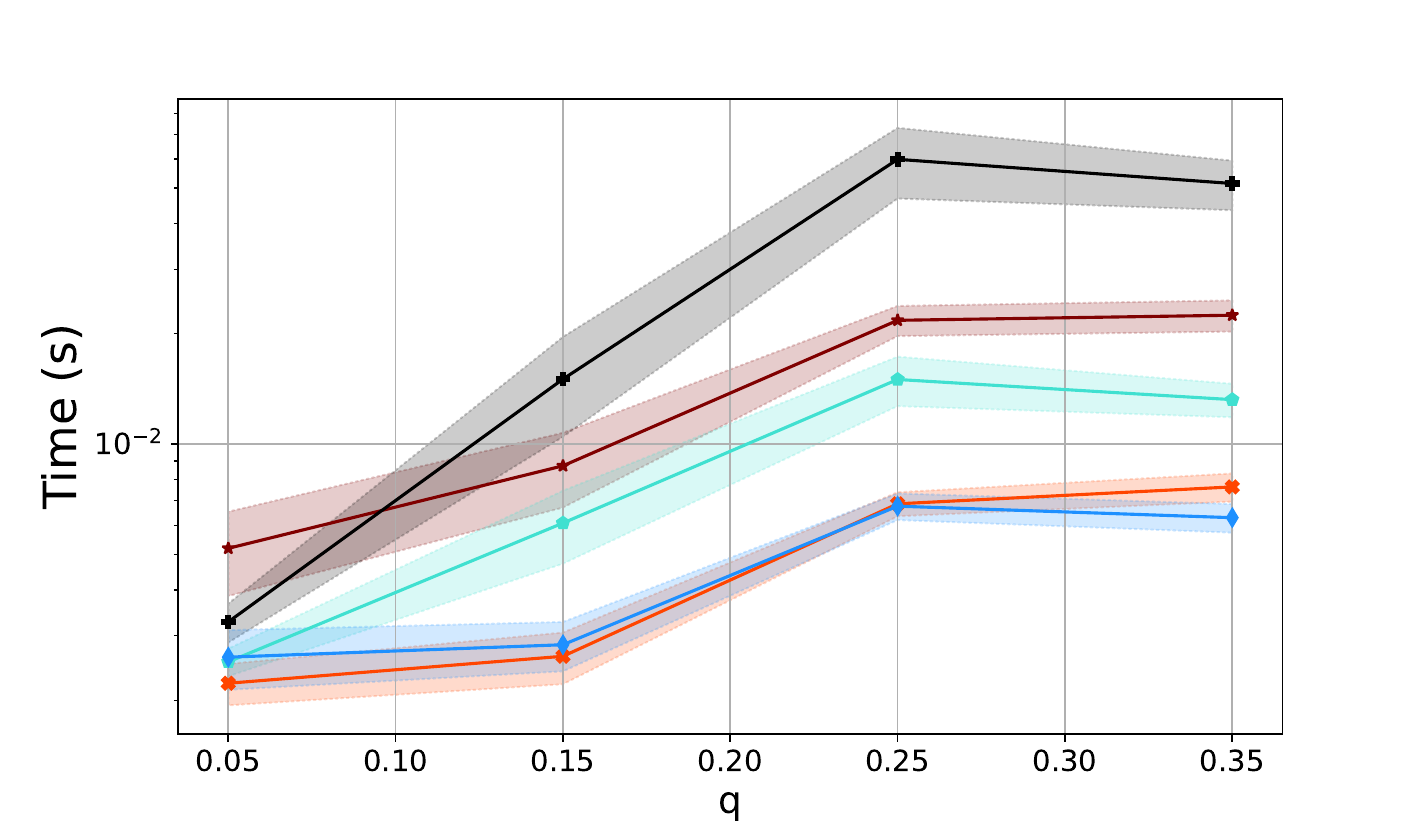}
        \caption{Barley structure}
    \end{subfigure}
    \begin{subfigure}[b]{0.99\textwidth}
        \centering
        \includegraphics[width=0.43\textwidth]{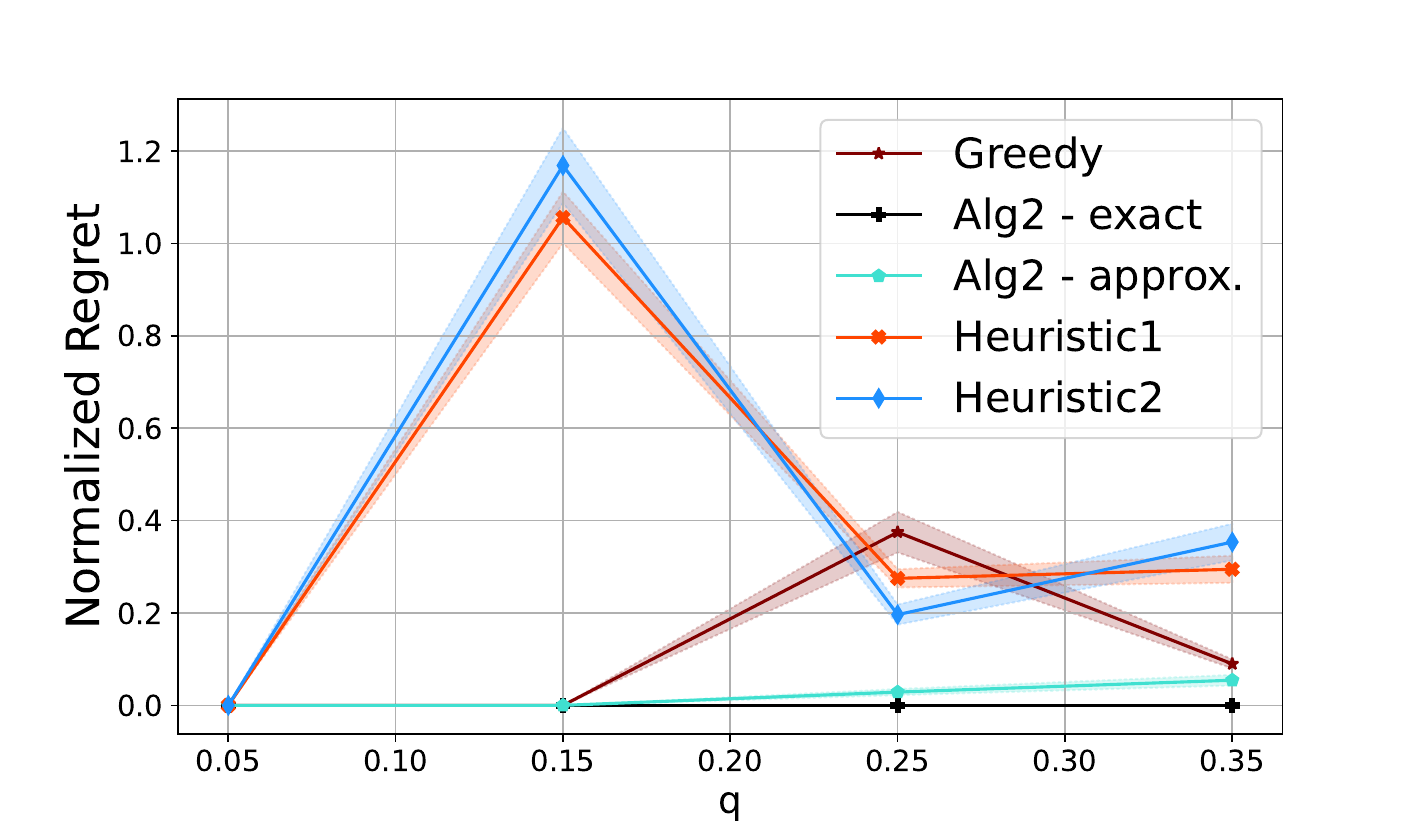}
        \includegraphics[width=0.43\textwidth]{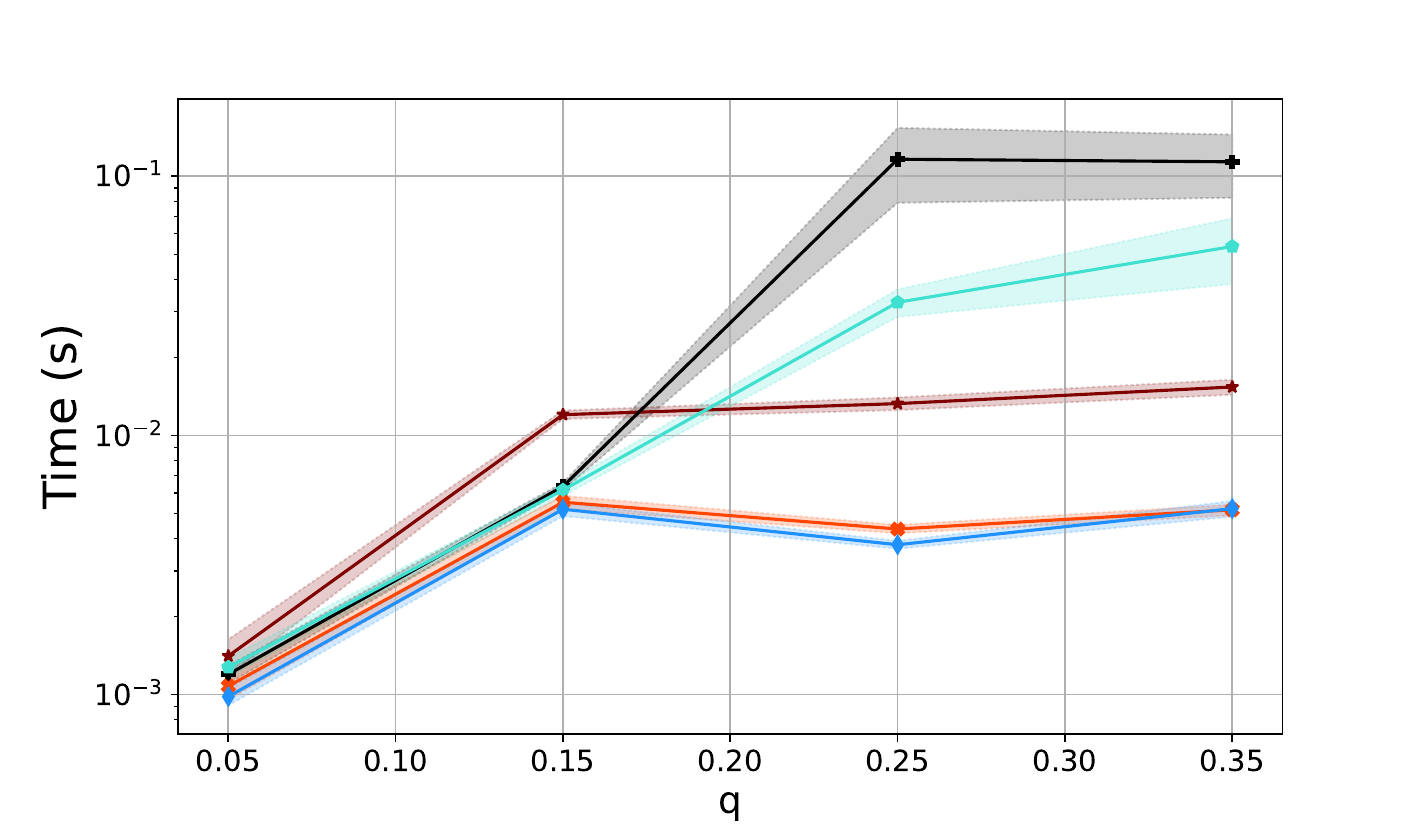}
        \caption{Water structure}
    \end{subfigure}
    \begin{subfigure}[b]{0.99\textwidth}
        \centering
        \includegraphics[width=0.43\textwidth]{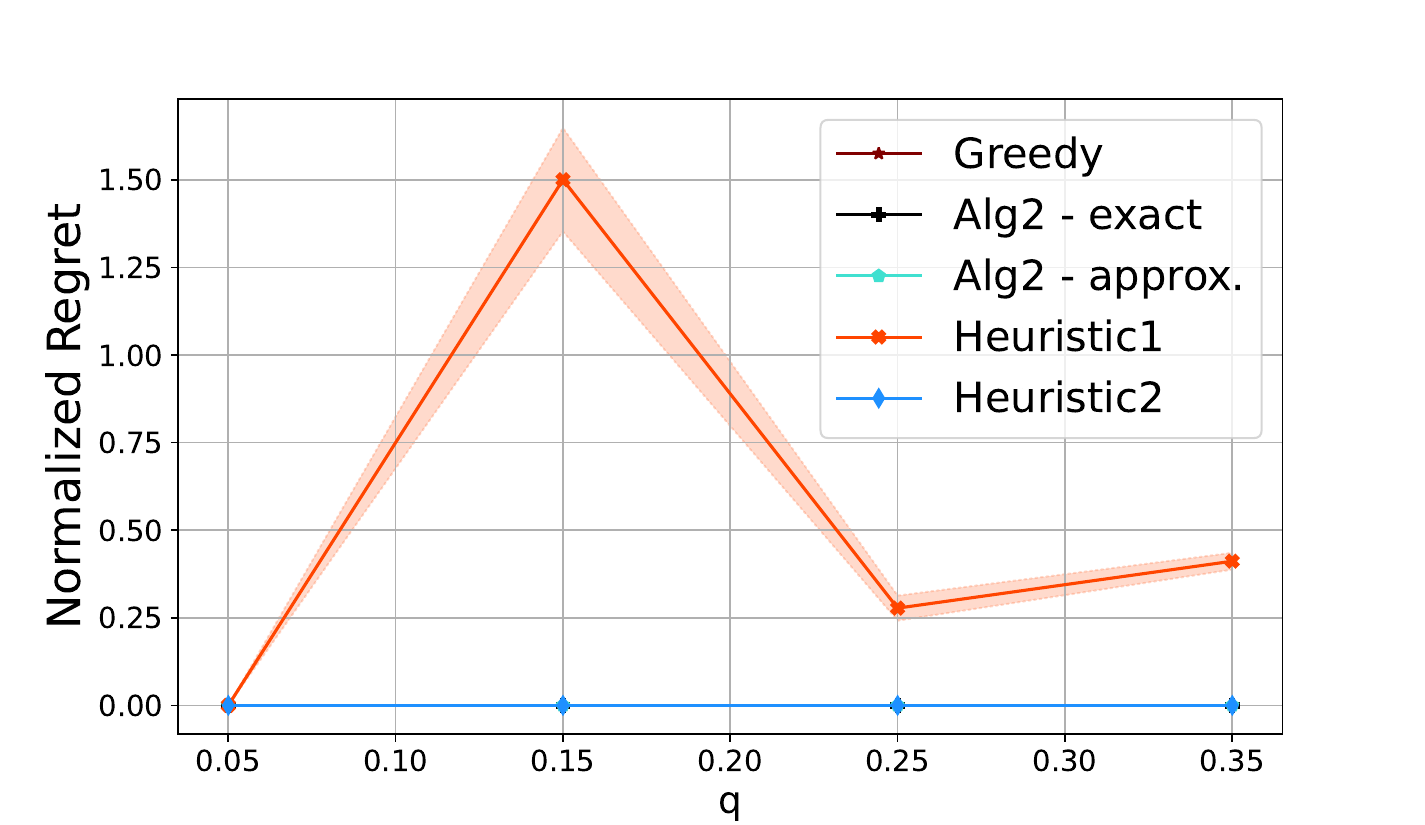}
        \includegraphics[width=0.43\textwidth]{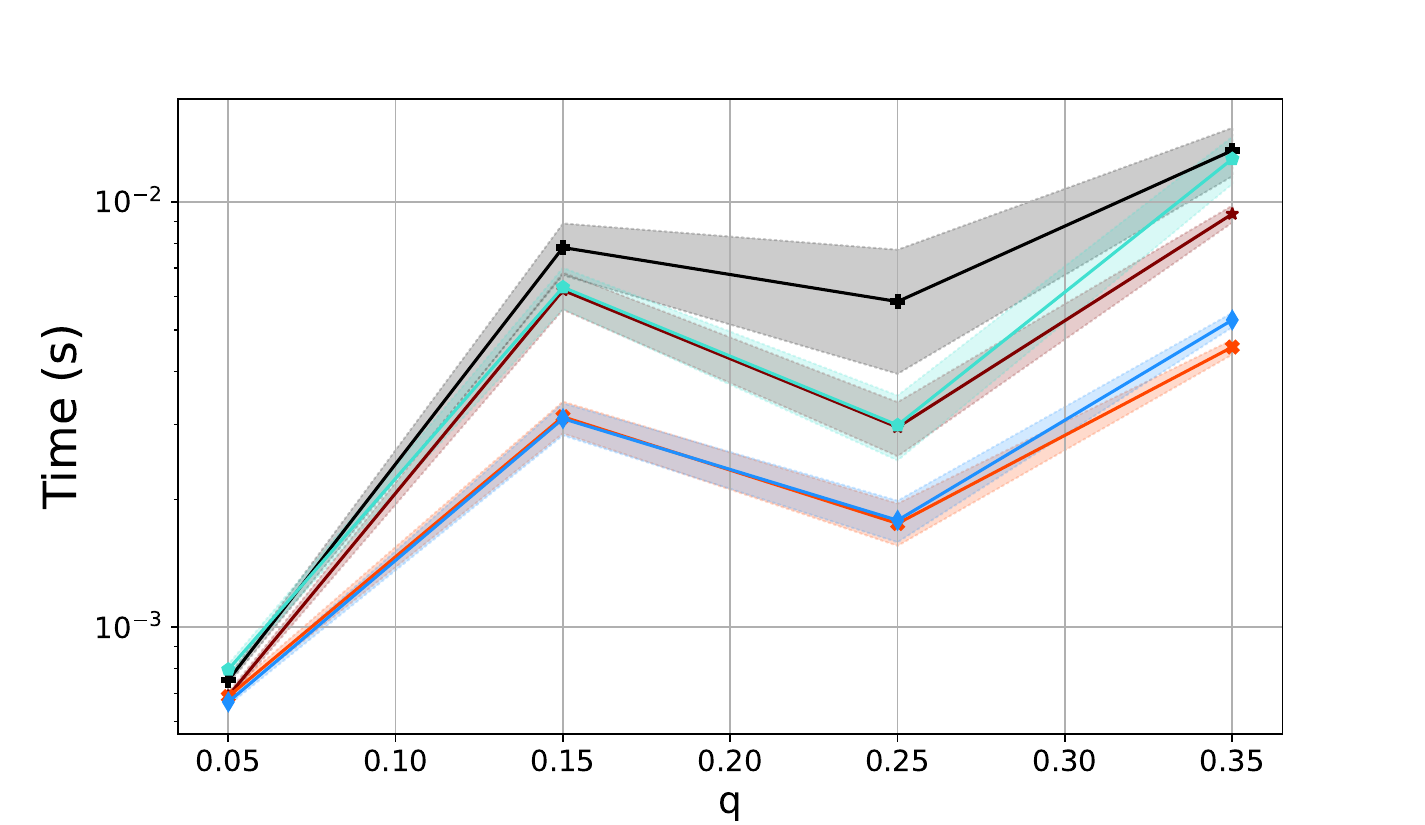}
        \caption{Mehra structure}
    \end{subfigure}
    
      \caption{The performance of the proposed algorithms on three real-world structures.}
    \label{fig:realworld}
\end{figure}

\subsection{Randomly Generated Graphs}\label{apdx:random}
Figure \ref{fig:fixedpq} illustrates the runtime and the normalized regret (as defined in Section \ref{sec:experiment}) of our algorithms on randomly generated graphs, with different values of $p$ and $q$ over random graphs of size $n=10$ to $n=200$.
Figure \ref{fig:fixednp} shows the effect of the density of the bidirected edges on the performance of the algorithms.
Random graphs of size $n=30$ are generated with different values of $p$.
Figure \ref{fig:fixednq} shows the effect of the density of the directed edges on the performance of the algorithms.
Random graphs of size $n=30$ are generated with different values of the parameter $q$.
An important observation in all of these figures is that the normalized regret is not necessarily a monotone function of the graph size.
This measure depends on the structure of the graph, size and location of the desired set $S$, and the random cost assignments.
Another observation is that the runtime of the algorithms is not a monotone funciton of the graph density.
This is due to the fact that the denser the graph becomes, the larger the set $\PaC{S}$ grows.
As a result, the set $H$ defined in Equation \ref{eq:h-def} becomes smaller and after a certain threshold, the problem becomes even simpler for denser graphs.


\begin{figure}
    \centering

    \begin{subfigure}[b]{0.99\textwidth}
        \centering
        \includegraphics[width=0.43\textwidth]{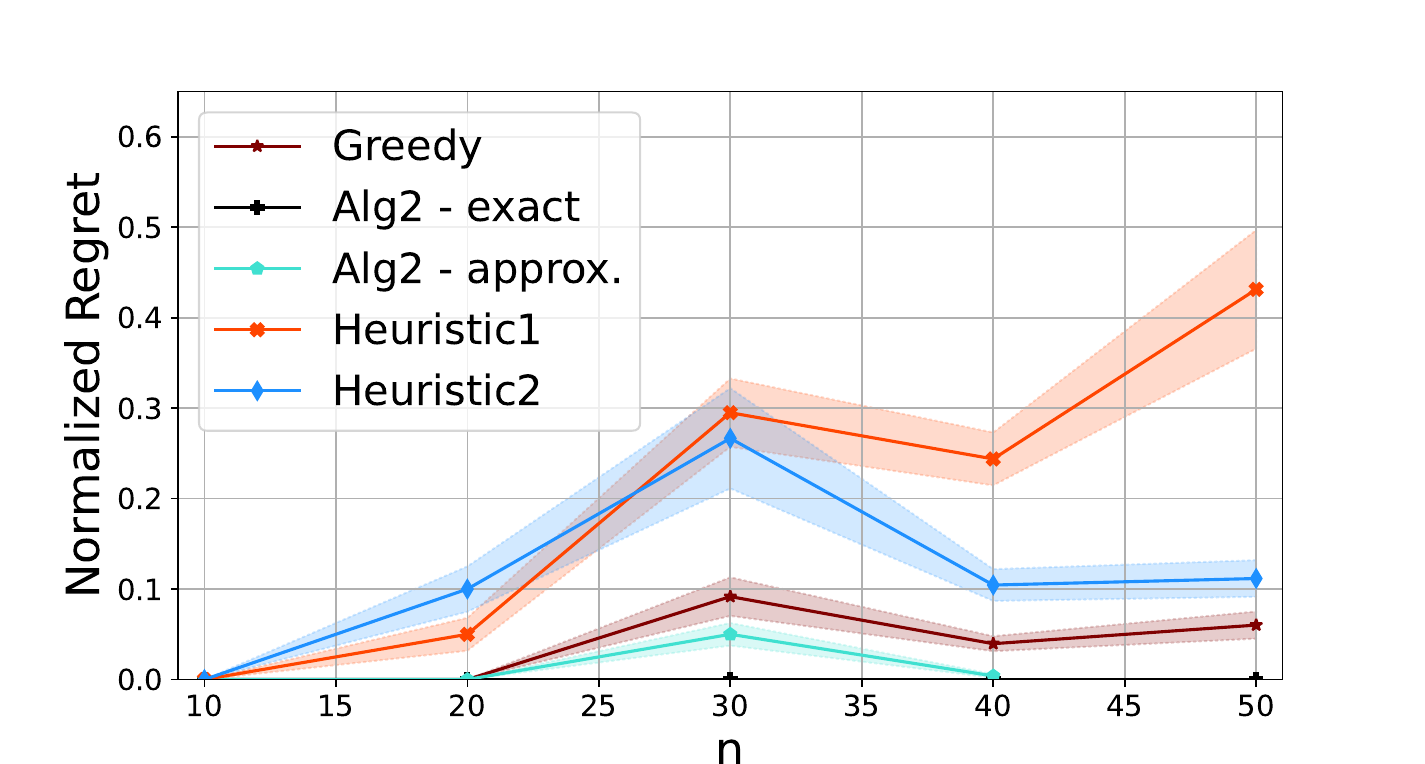}
        \includegraphics[width=0.43\textwidth]{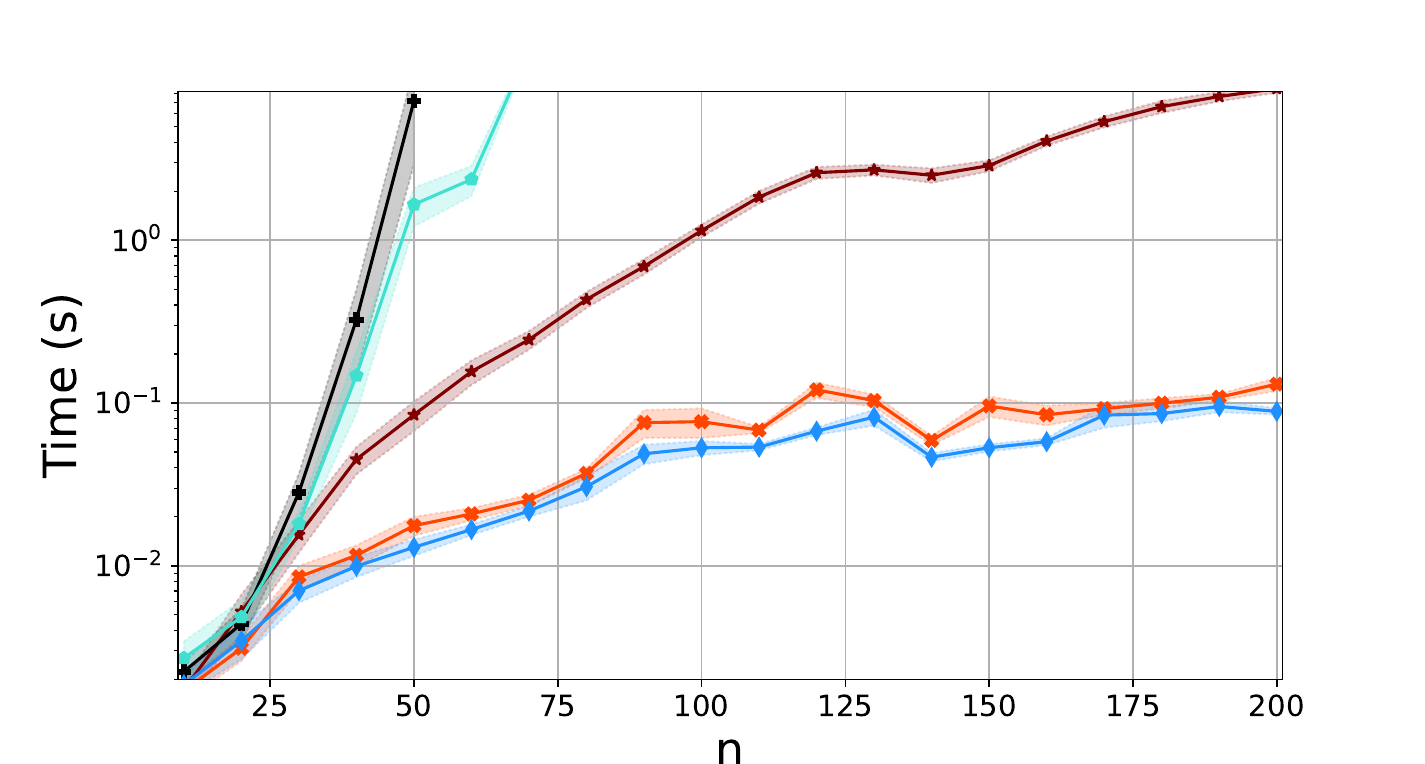}
        \caption{$p=0.15,q=0.15$}
    \end{subfigure}
    \begin{subfigure}[b]{0.99\textwidth}
        \centering
        \includegraphics[width=0.43\textwidth]{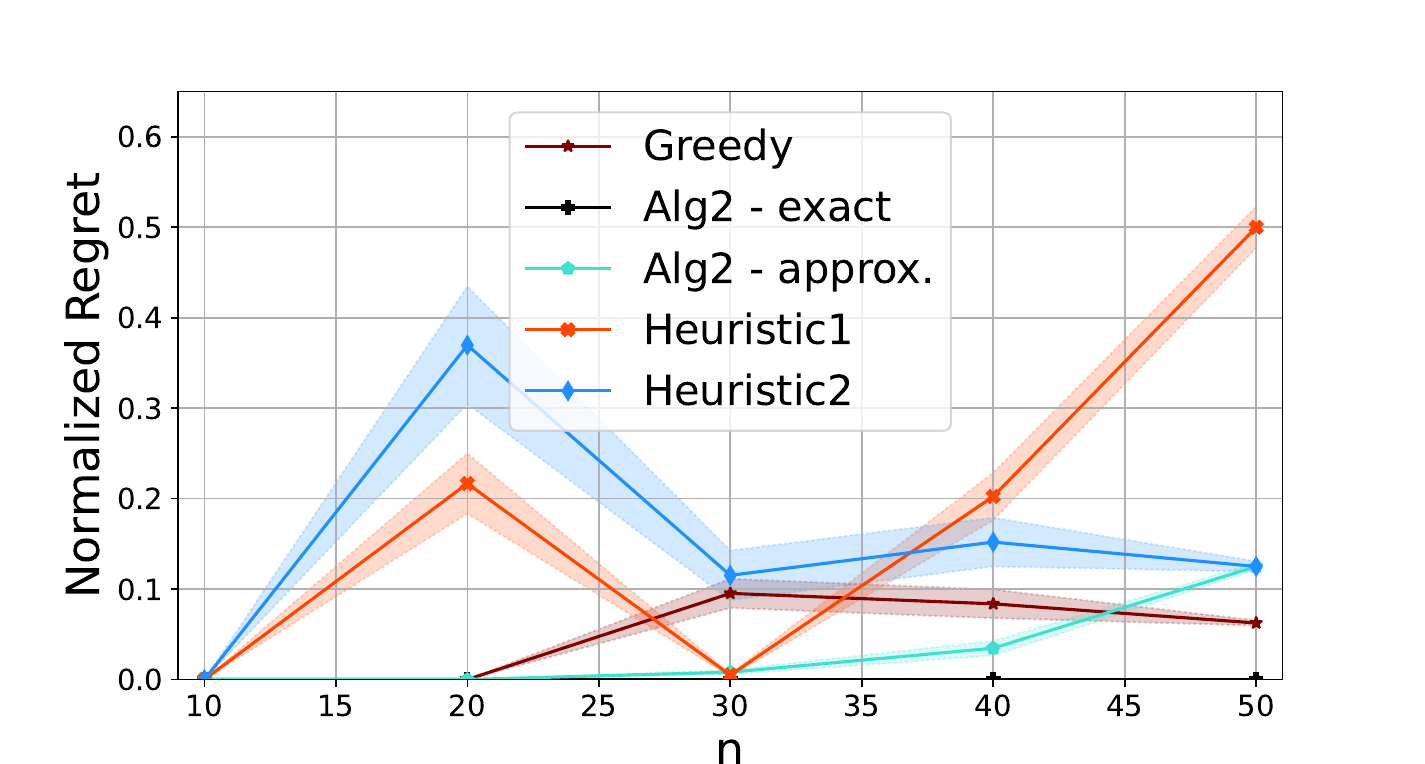}
        \includegraphics[width=0.43\textwidth]{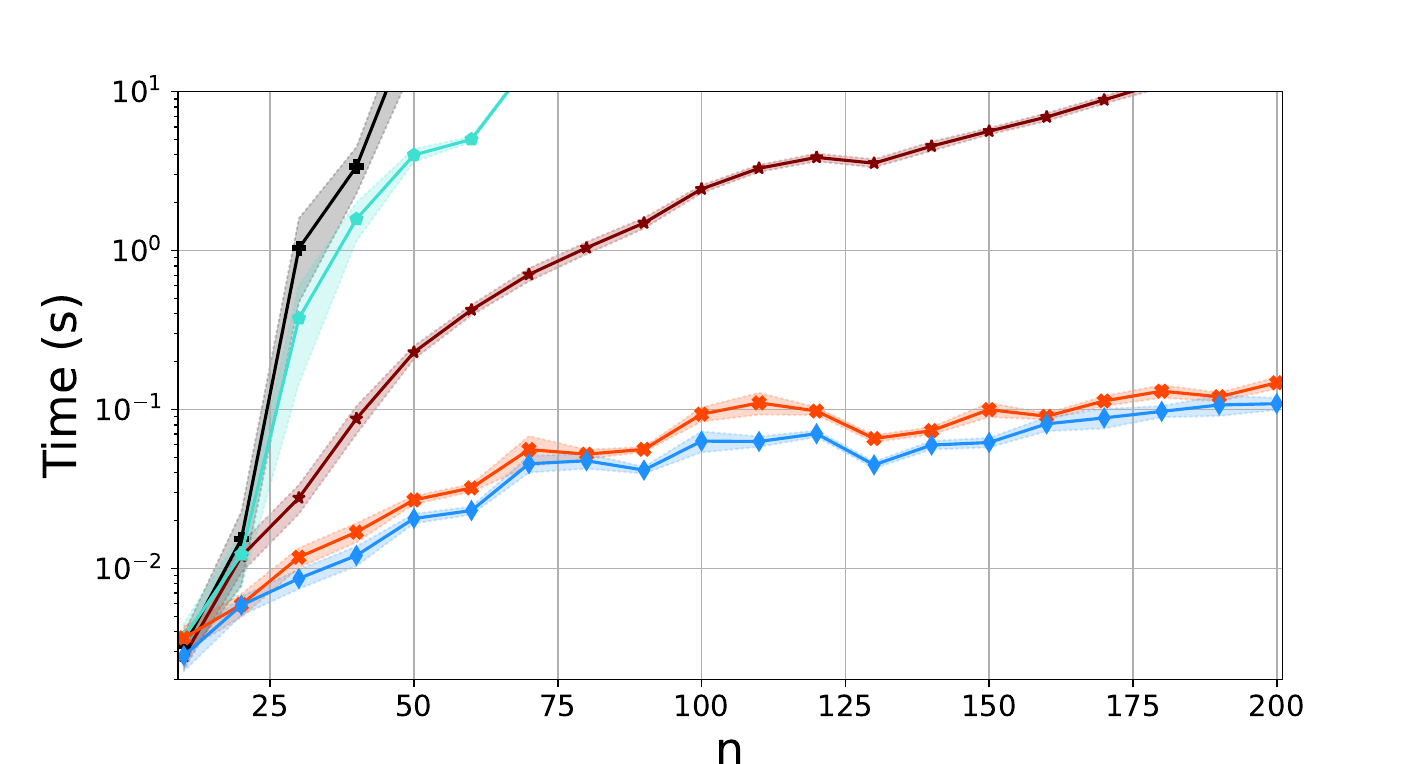}
        \caption{$p=0.25,q=0.15$}
    \end{subfigure}
    \begin{subfigure}[b]{0.99\textwidth}
        \centering
        \includegraphics[width=0.43\textwidth]{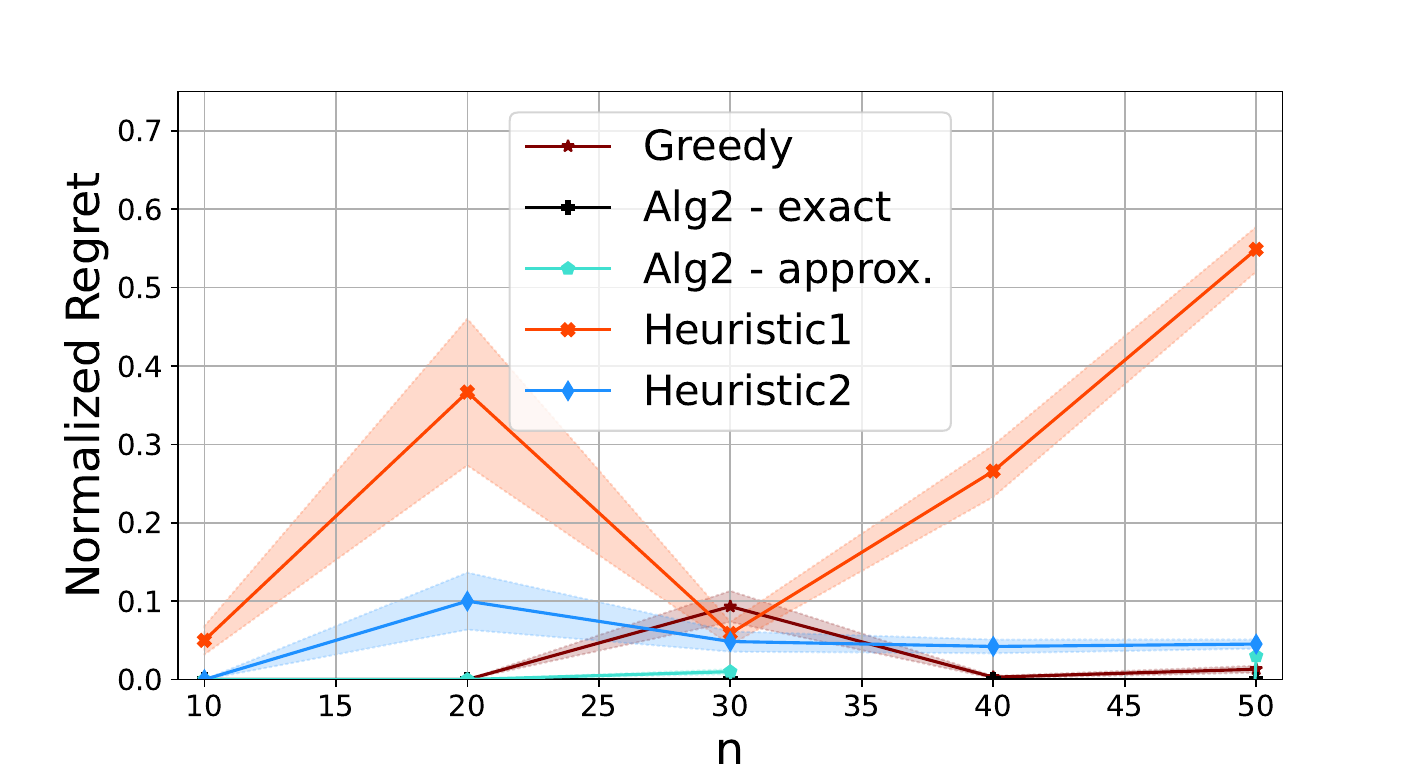}
        \includegraphics[width=0.43\textwidth]{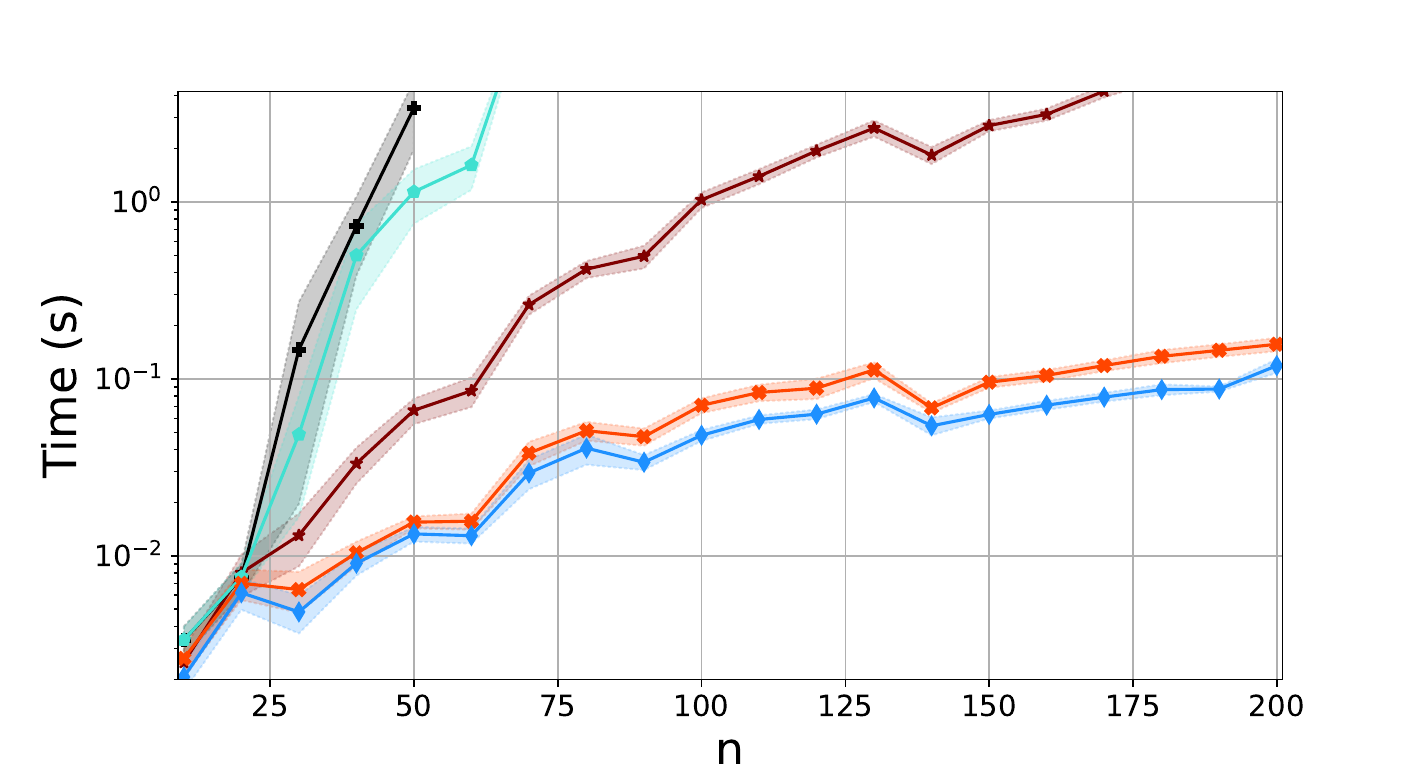}
        \caption{$p=0.15,q=0.25$}
    \end{subfigure}
    \begin{subfigure}[b]{0.99\textwidth}
        \centering
        \includegraphics[width=0.43\textwidth]{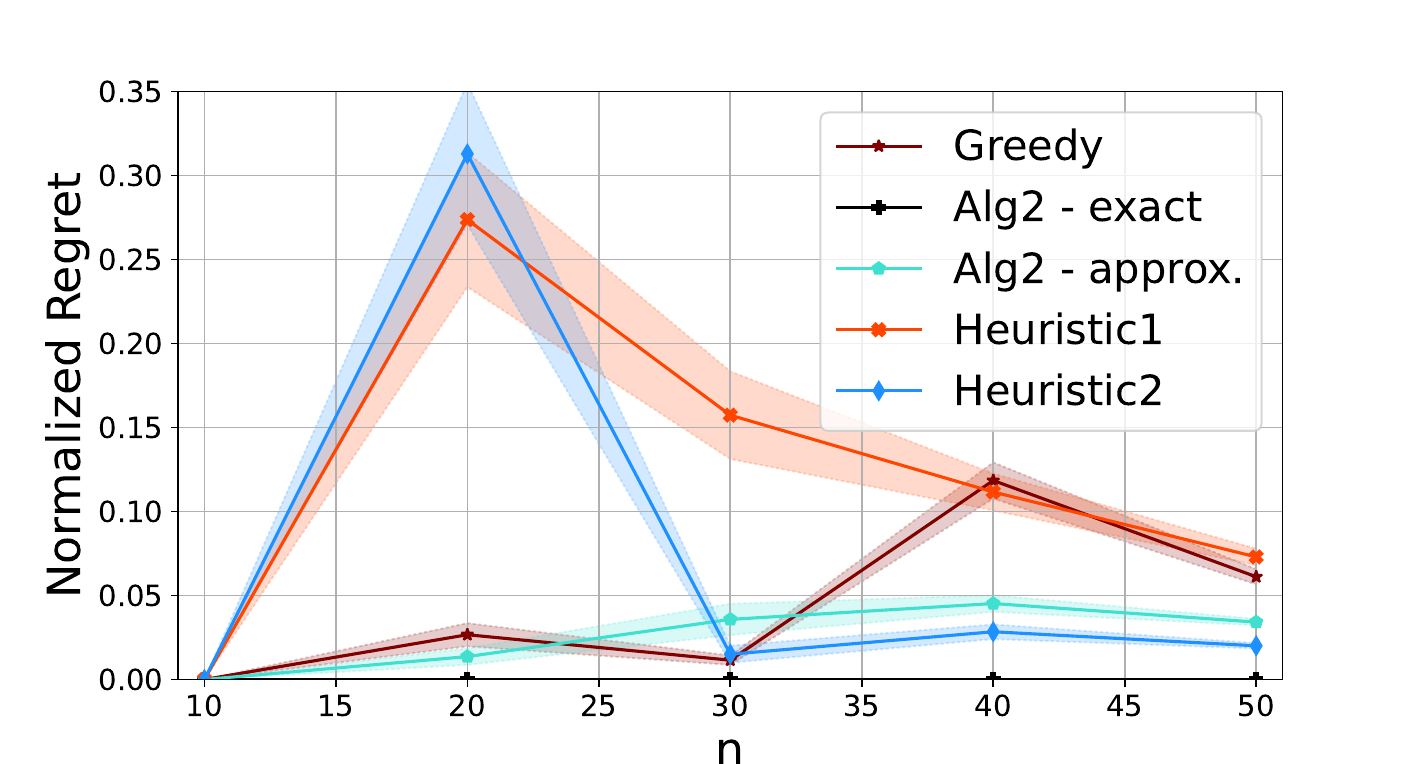}
        \includegraphics[width=0.43\textwidth]{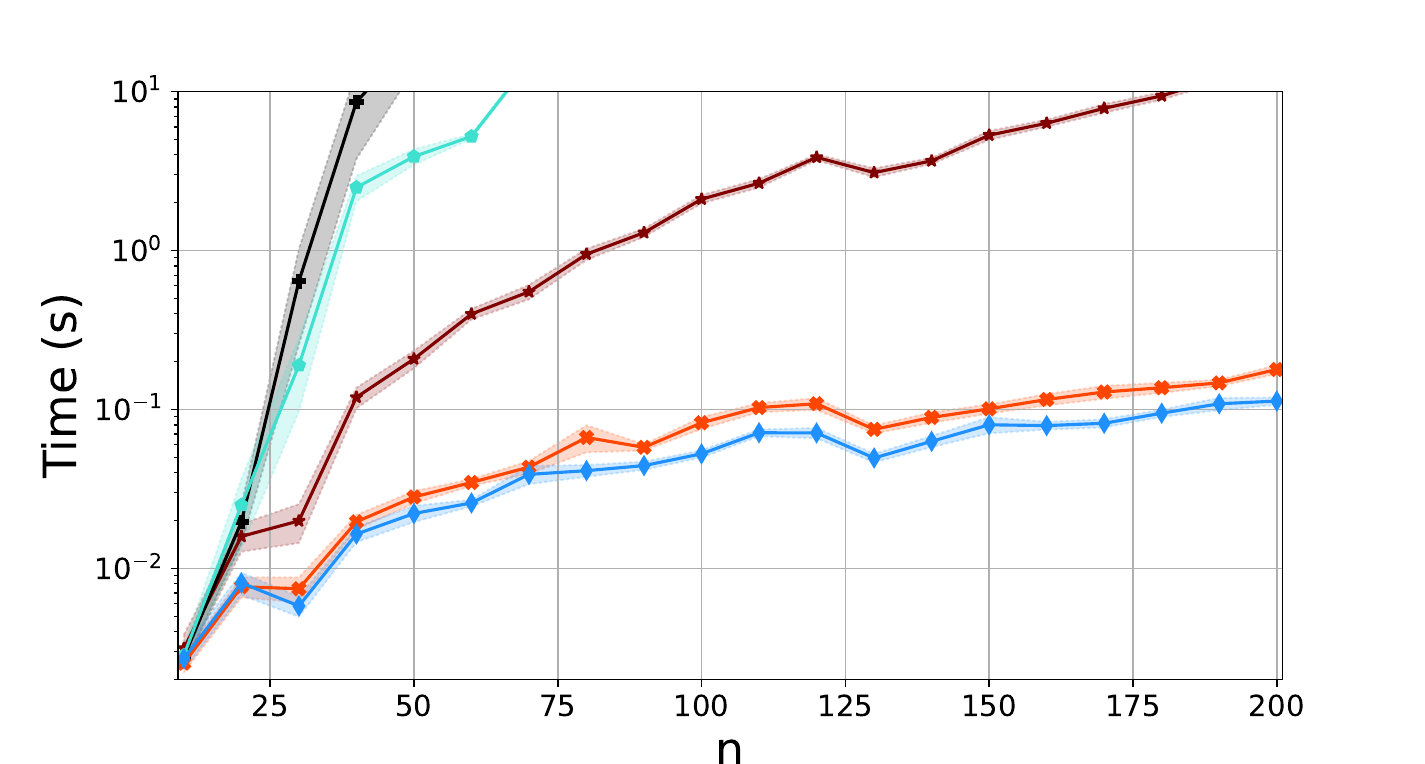}
        \caption{$p=0.25,q=0.25$}
    \end{subfigure}
    \caption{Evaluation of the proposed algorithms on random graphs with various parameters.}
    \label{fig:fixedpq}
\end{figure}

\begin{figure}
    \centering
    \begin{subfigure}[b]{0.99\textwidth}
        \centering
        \includegraphics[width=0.43\textwidth]{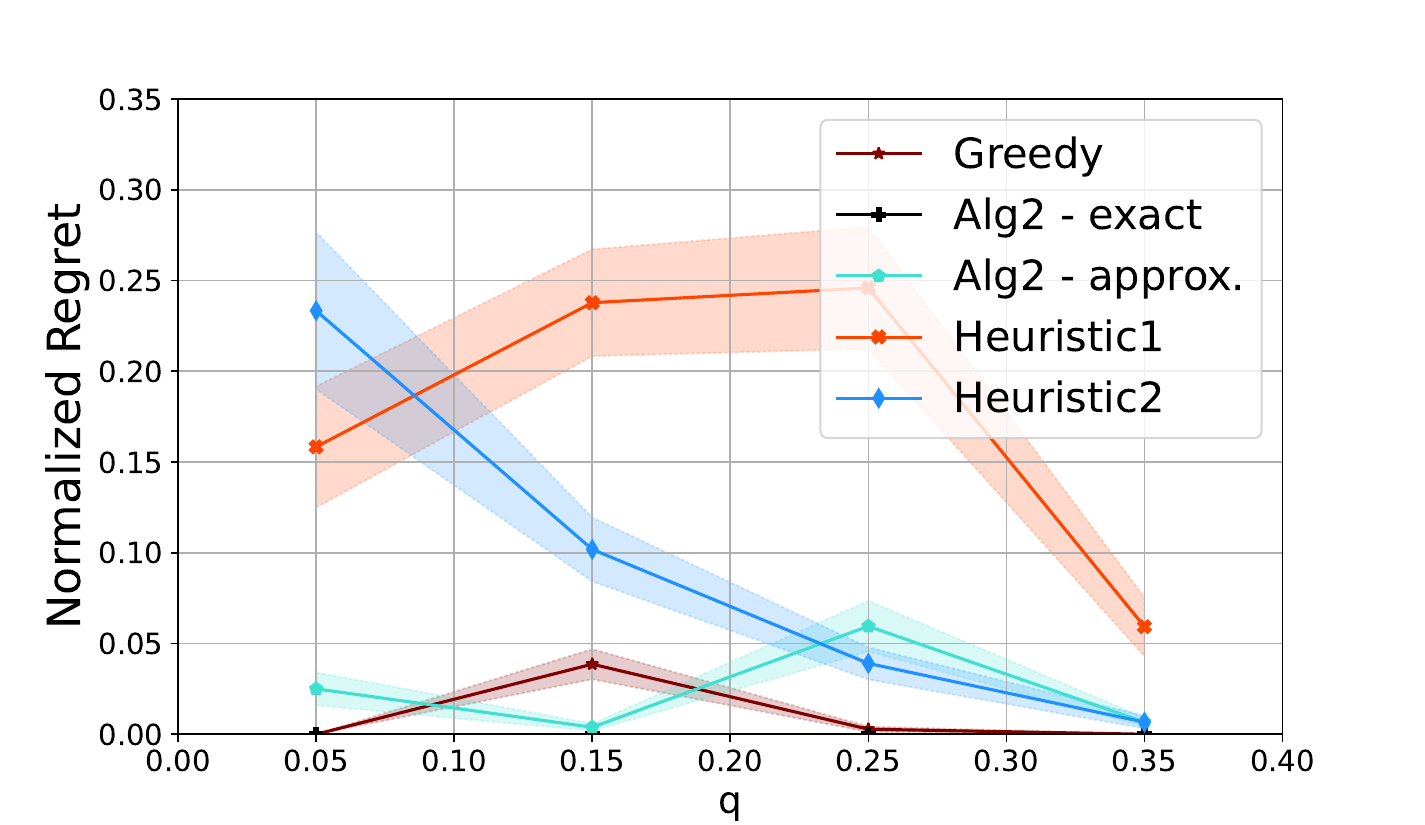}
        \includegraphics[width=0.43\textwidth]{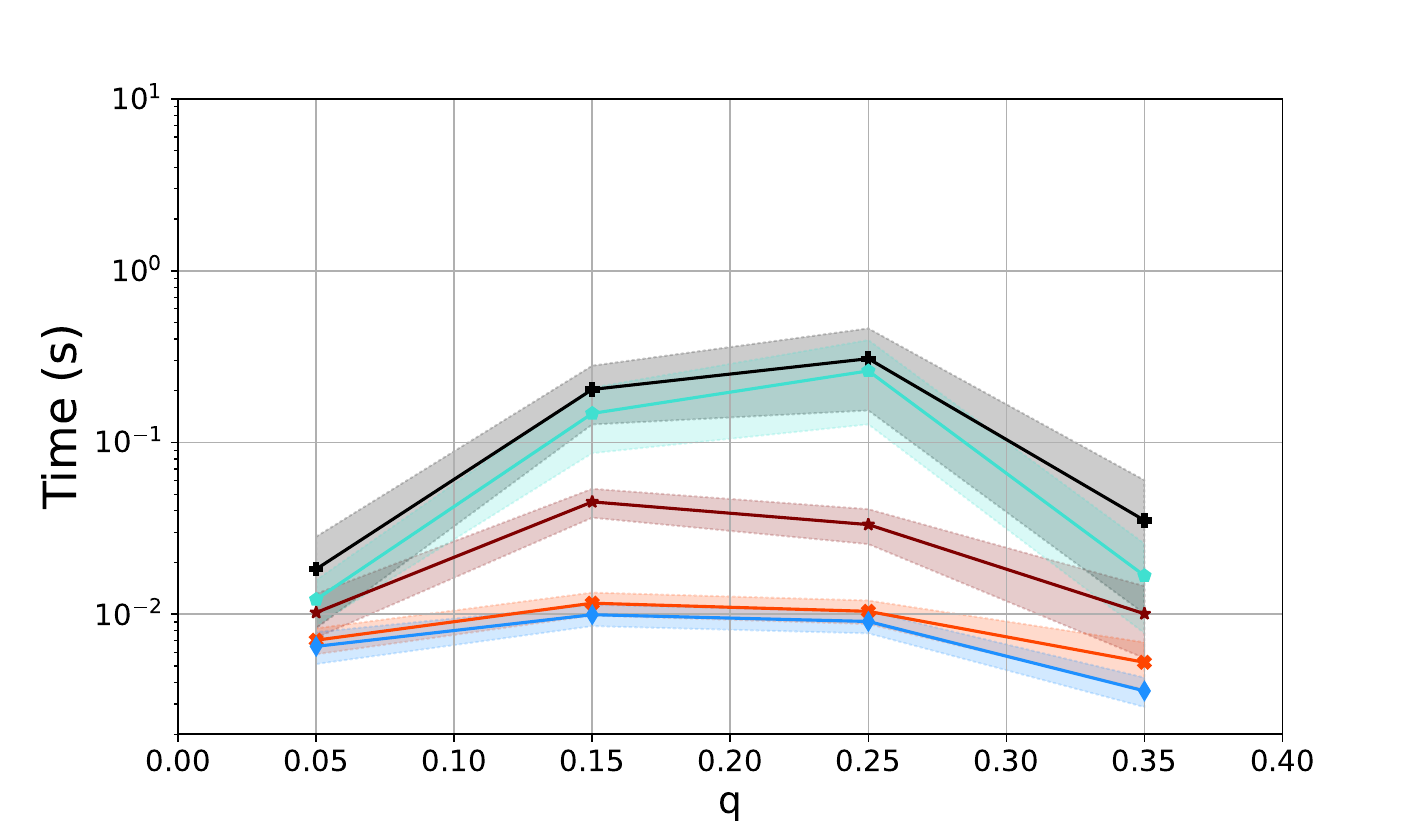}
        \caption{$p=0.15$}
    \end{subfigure}
    \begin{subfigure}[b]{0.99\textwidth}
        \centering
        \includegraphics[width=0.43\textwidth]{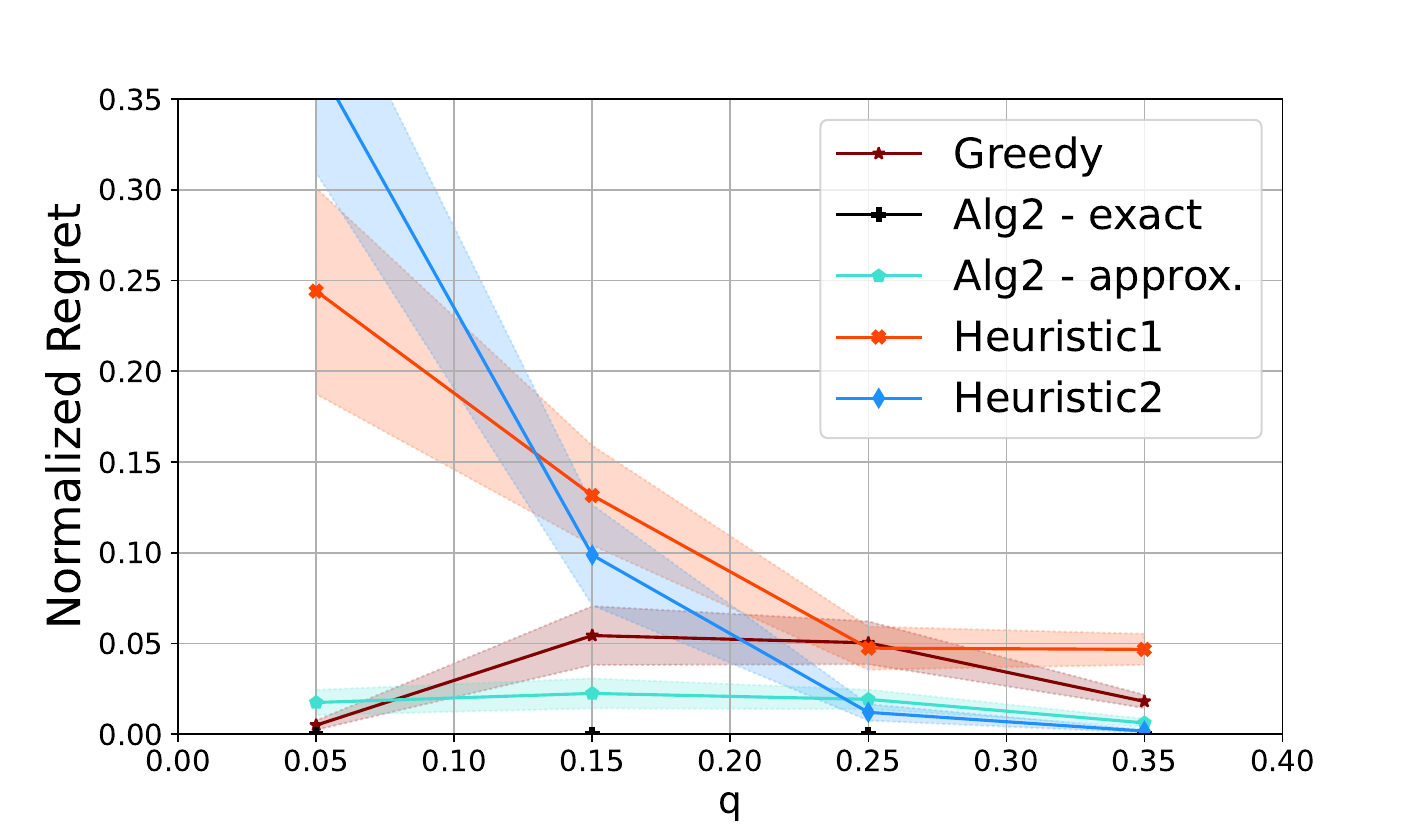}
        \includegraphics[width=0.43\textwidth]{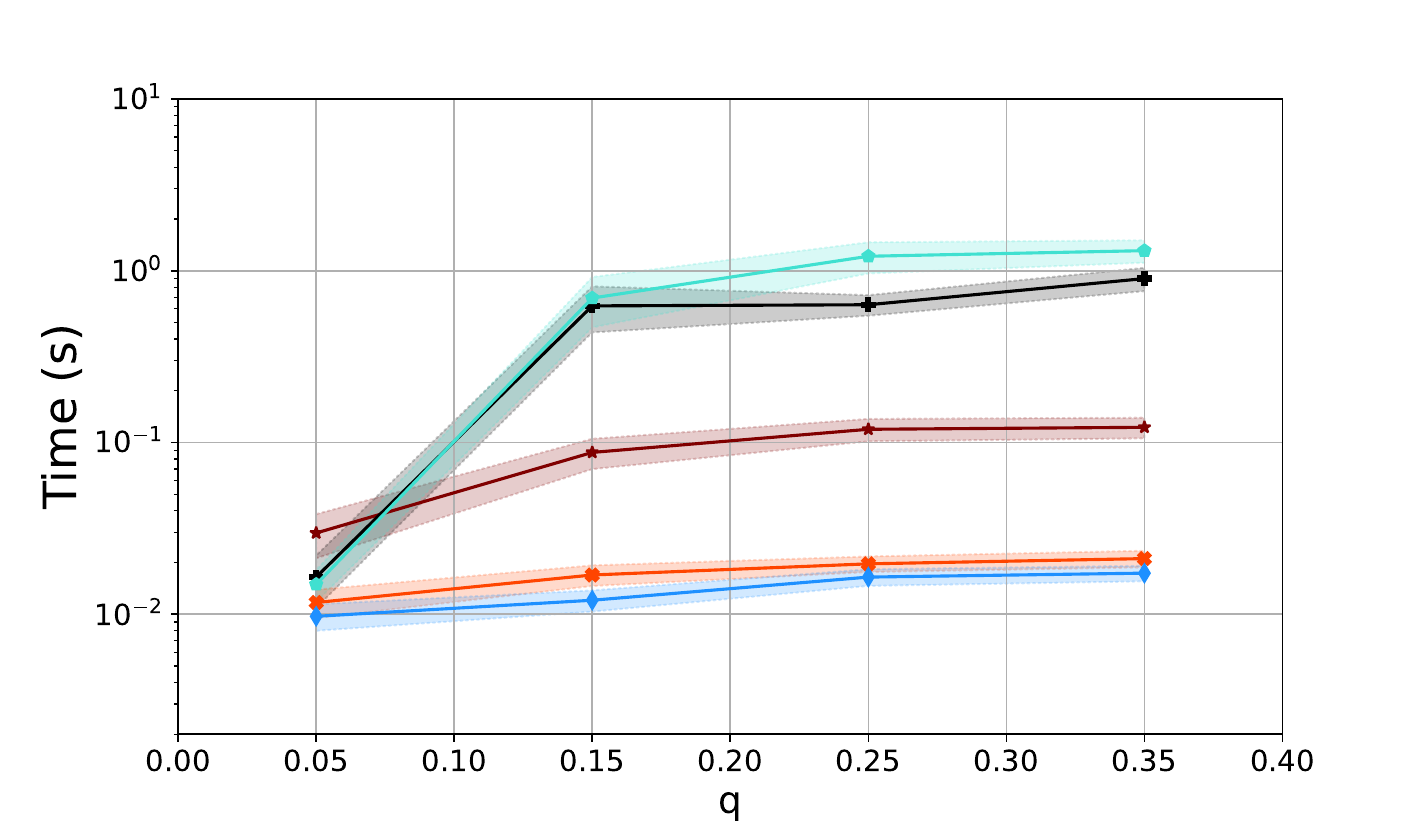}
        \caption{$p=0.25$}
    \end{subfigure}
    \begin{subfigure}[b]{0.99\textwidth}
        \centering
        \includegraphics[width=0.43\textwidth]{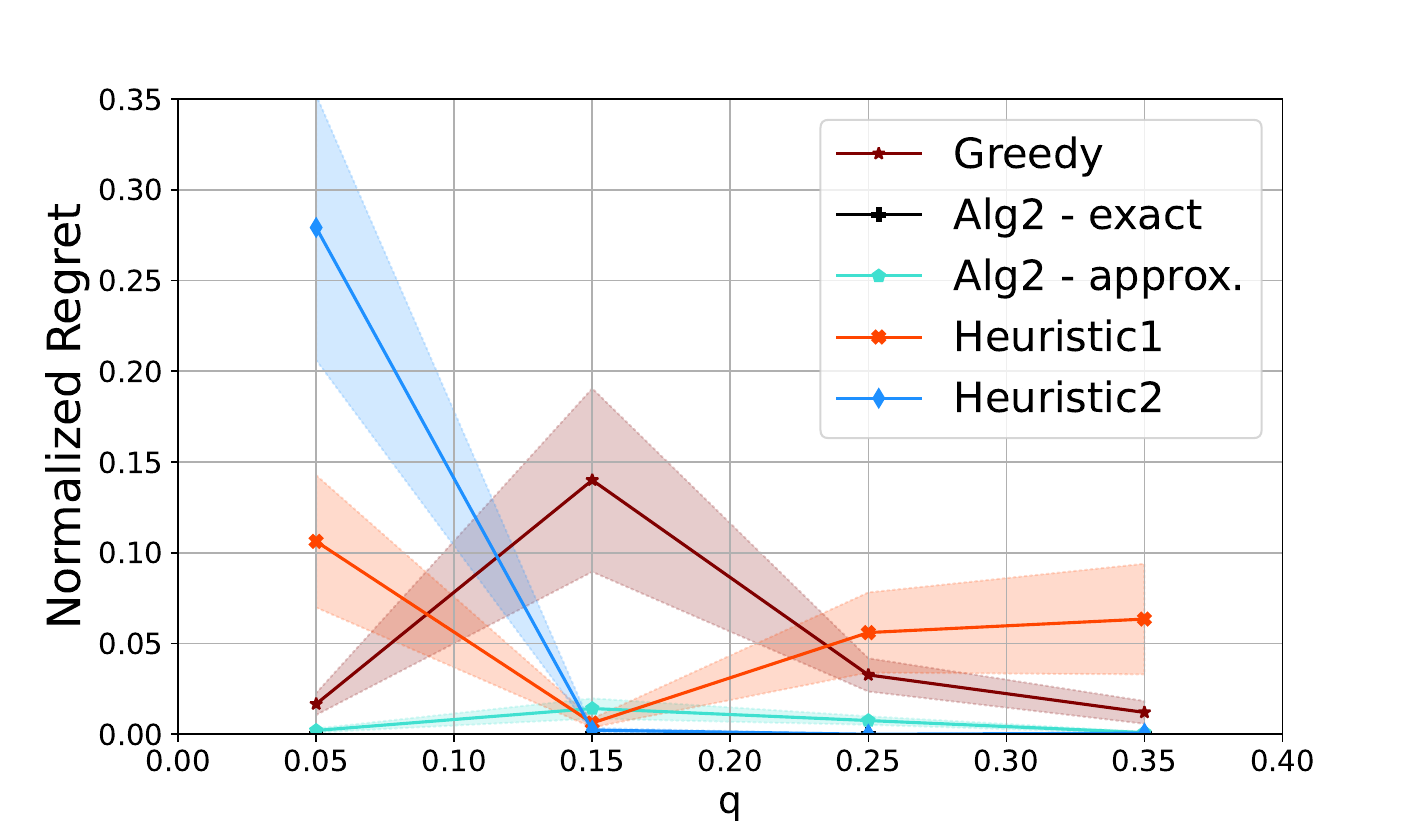}
        \includegraphics[width=0.43\textwidth]{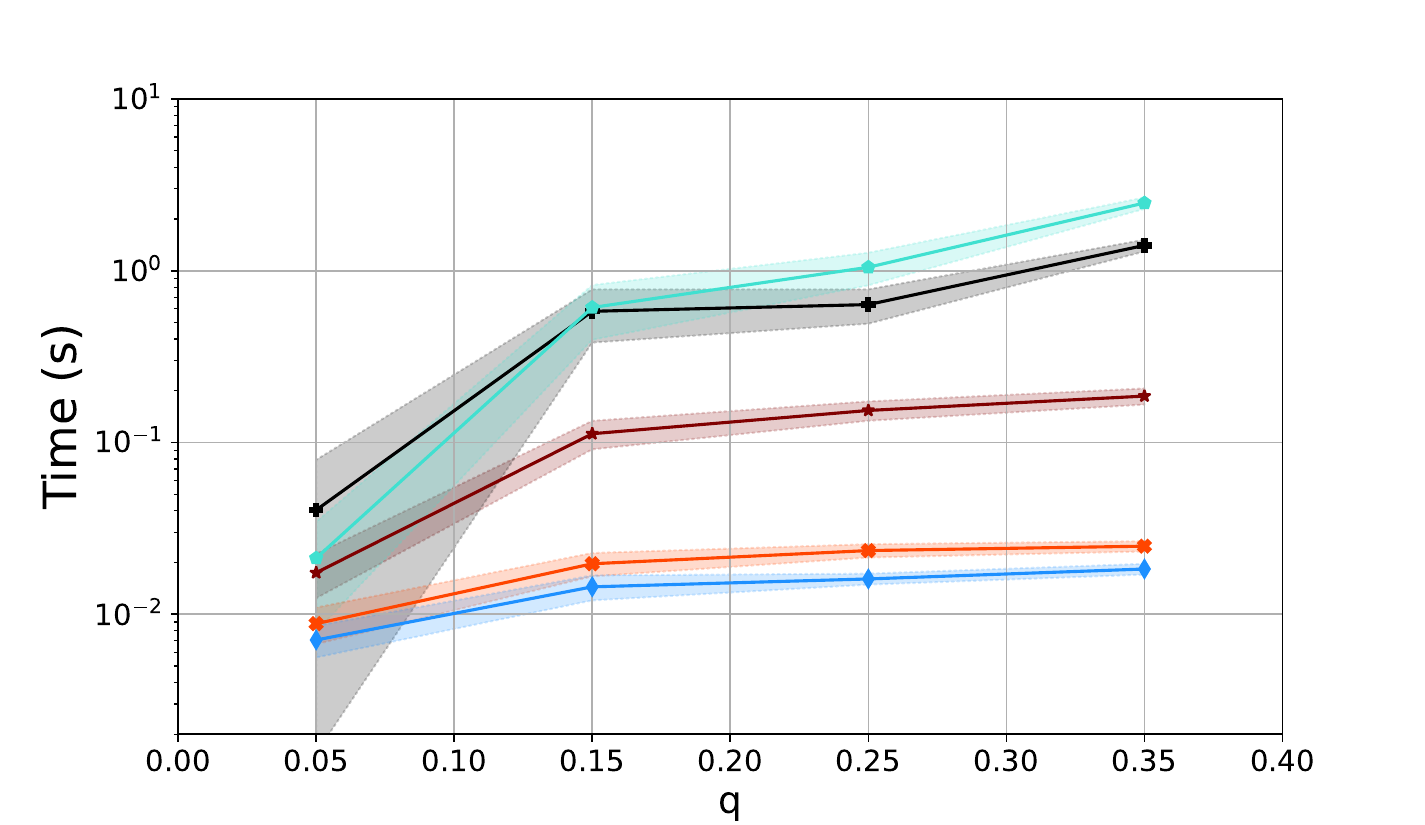}
        \caption{$p=0.35$}
    \end{subfigure}
    
      \caption{The effect of the density of bidirected edges. Random graphs of size $n=30$ are generated with different densities of directed edges.}
    \label{fig:fixednp}
\end{figure}

\begin{figure}
    \centering
    \begin{subfigure}[b]{0.99\textwidth}
        \centering
        \includegraphics[width=0.43\textwidth]{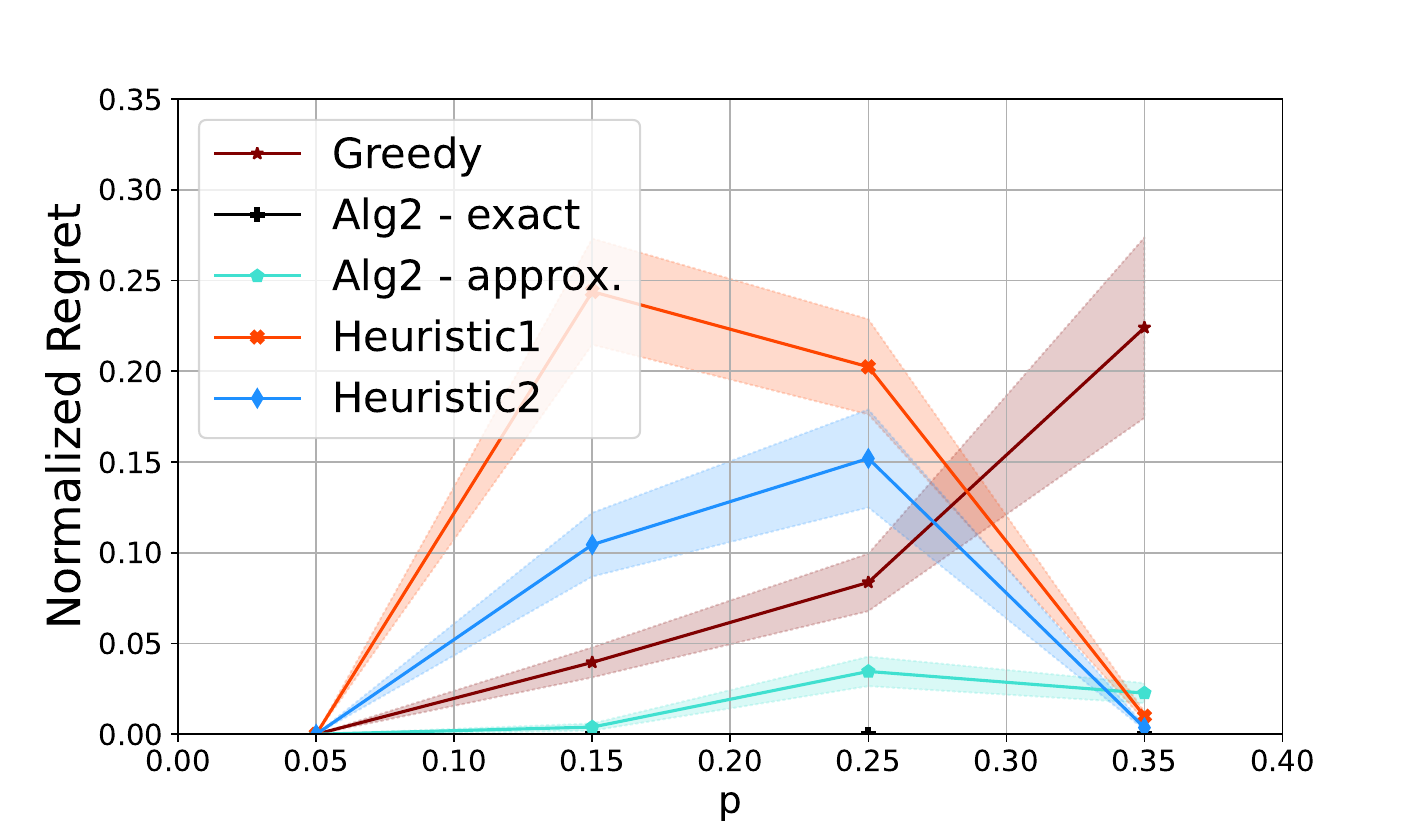}
        \includegraphics[width=0.43\textwidth]{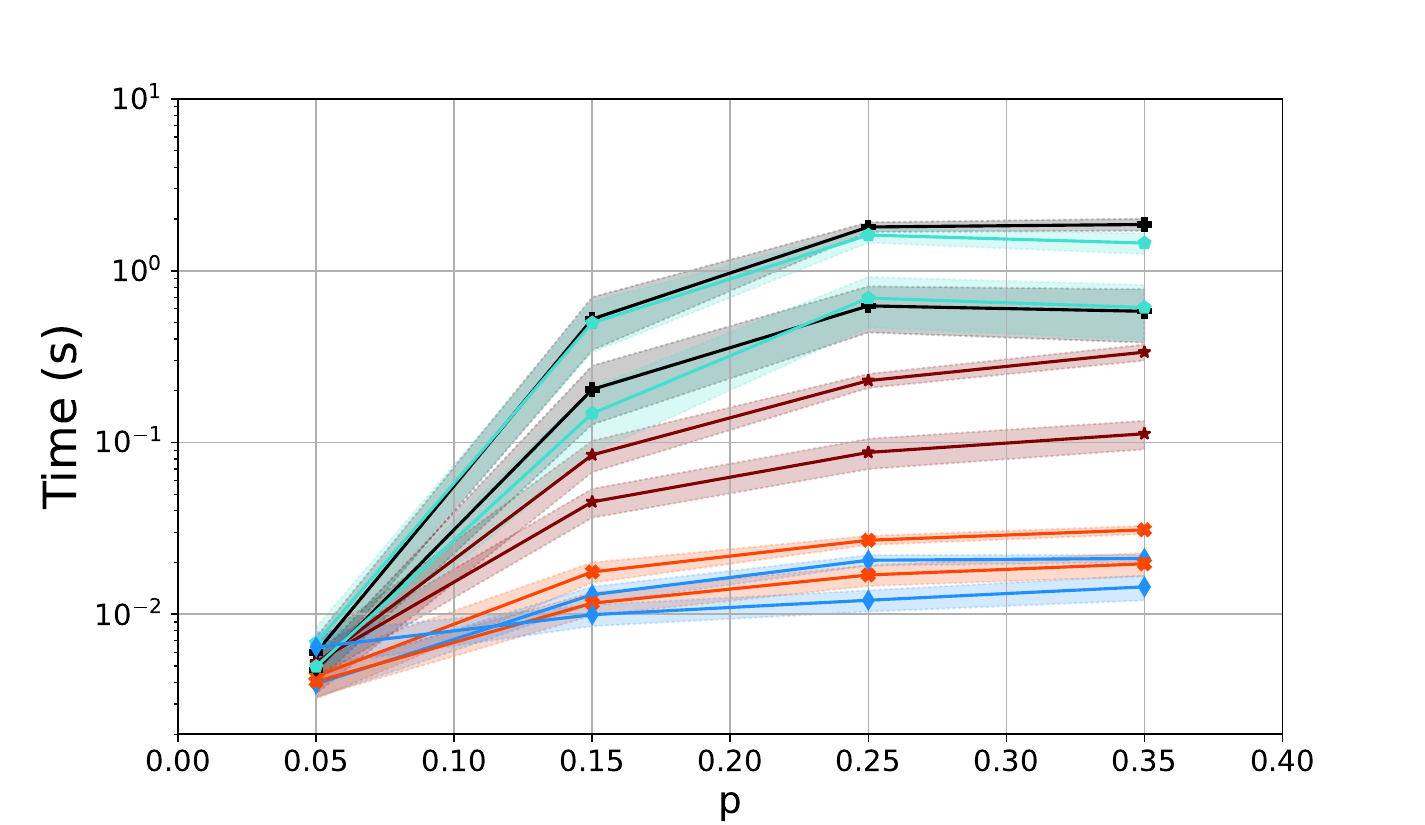}
        \caption{$q=0.15$}
    \end{subfigure}
    \begin{subfigure}[b]{0.99\textwidth}
        \centering
        \includegraphics[width=0.43\textwidth]{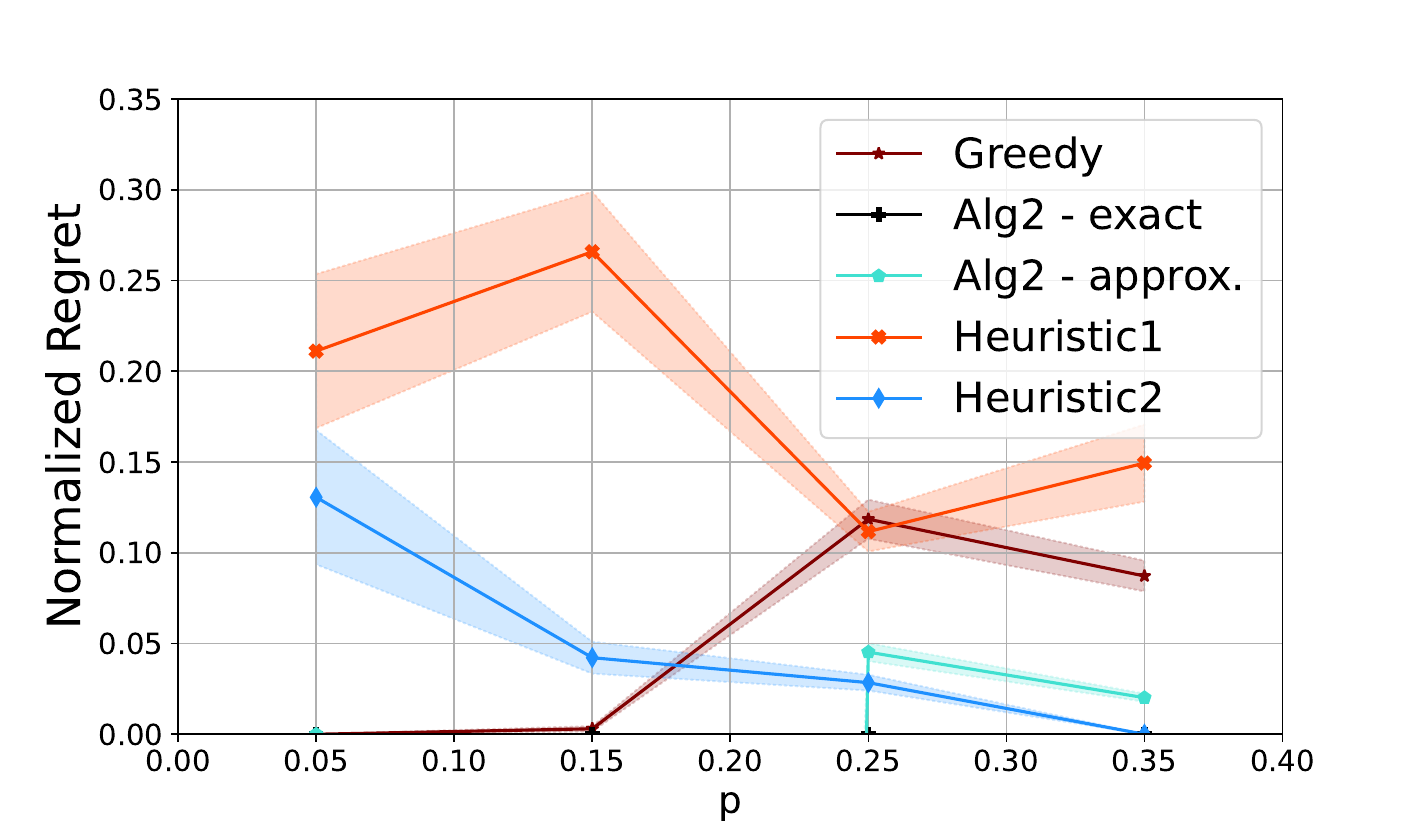}
        \includegraphics[width=0.43\textwidth]{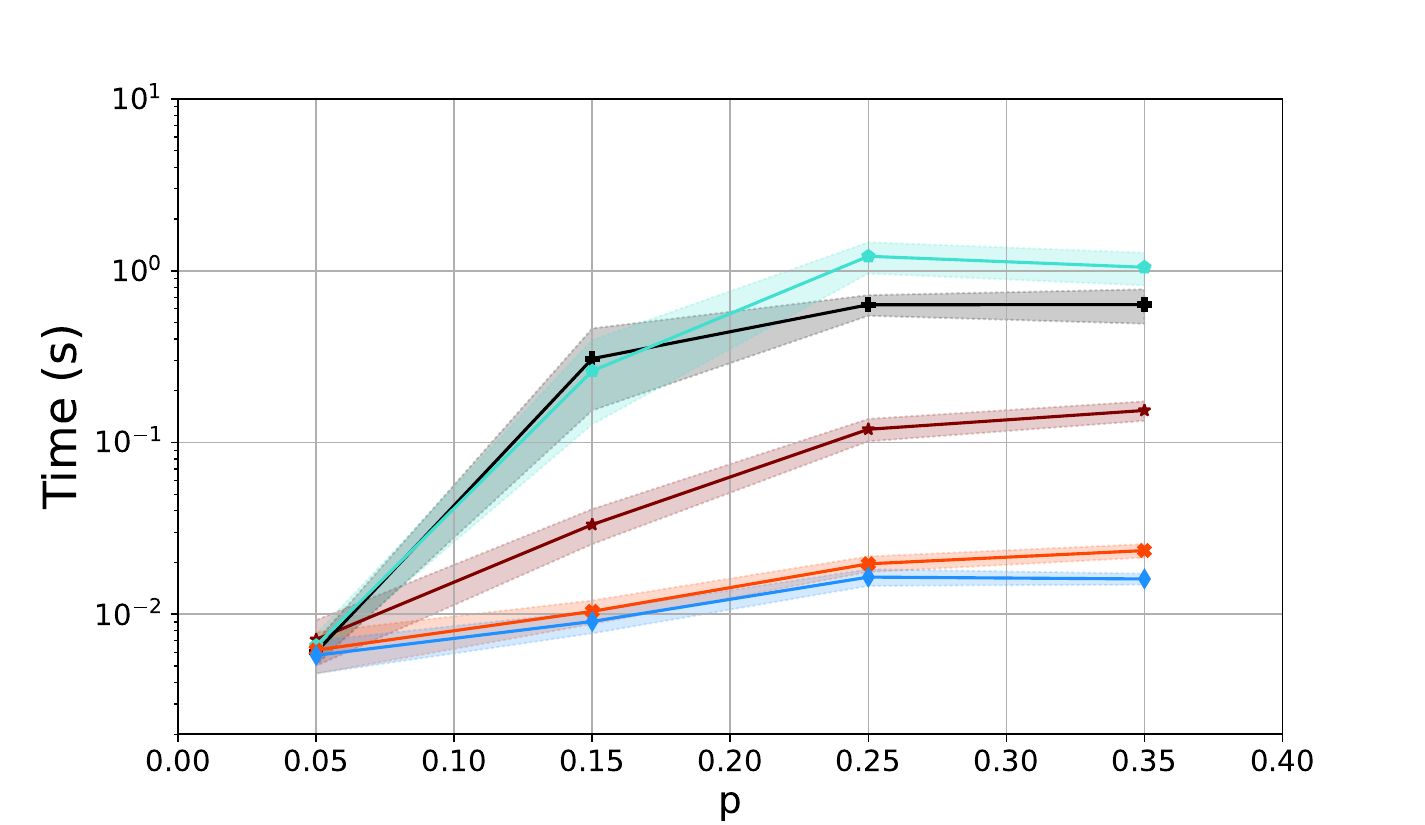}
        \caption{$q=0.25$}
    \end{subfigure}
    \begin{subfigure}[b]{0.99\textwidth}
        \centering
        \includegraphics[width=0.43\textwidth]{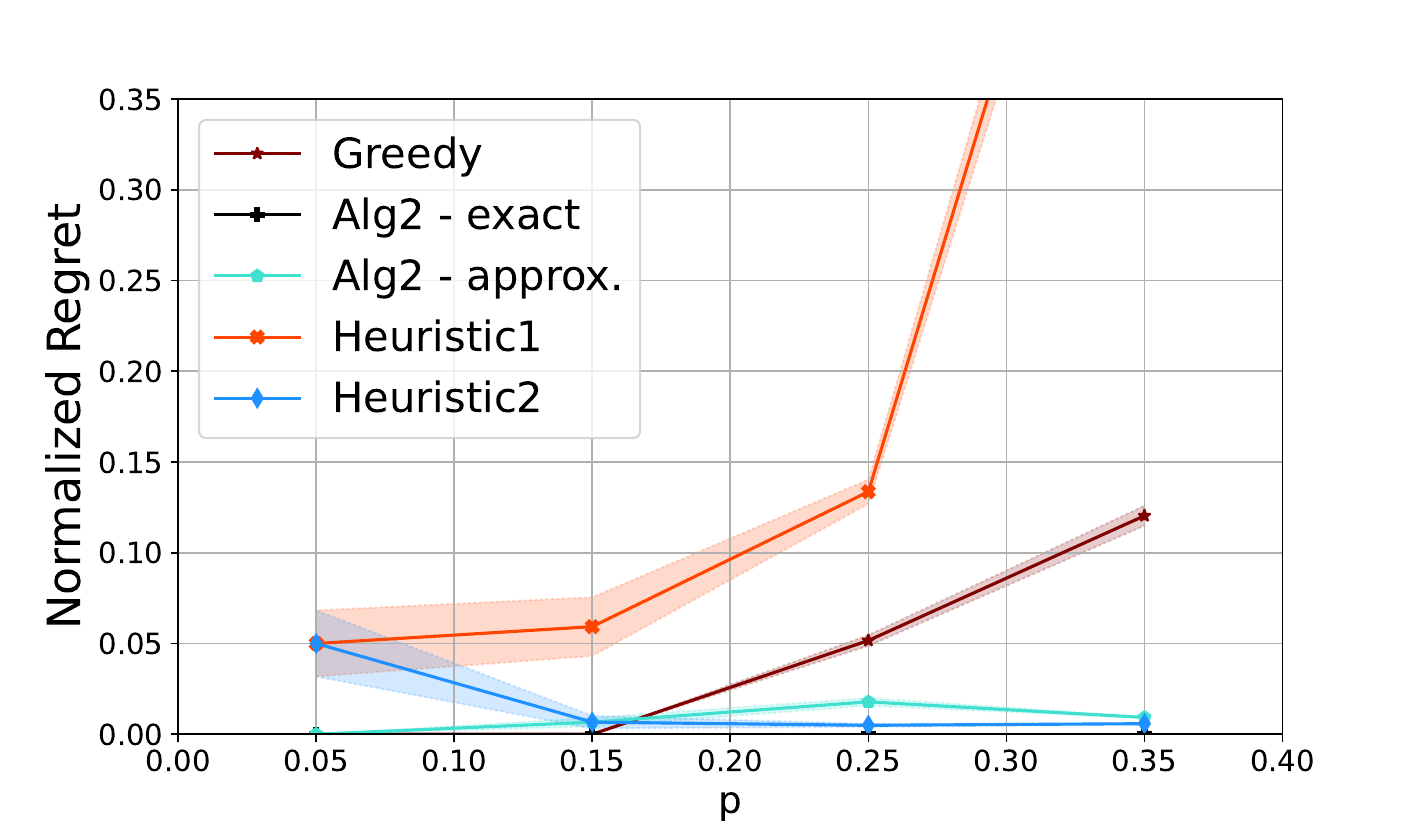}
        \includegraphics[width=0.43\textwidth]{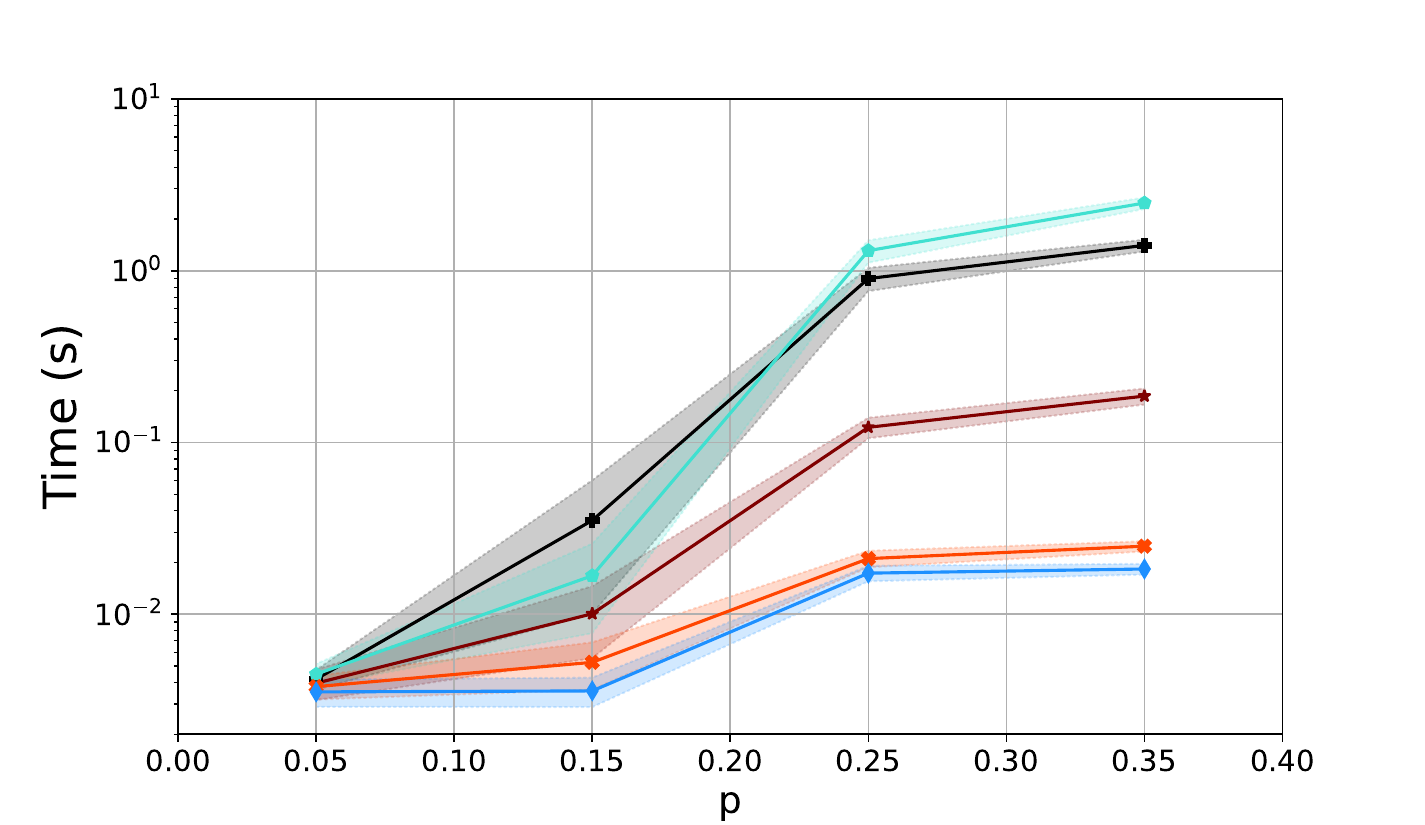}
        \caption{$q=0.35$}
    \end{subfigure}
    
      \caption{The effect of the density of directed edges. Random graphs of size $n=30$ are generated with different densities of bidirected edges.}
    \label{fig:fixednq}
\end{figure}


\end{document}